\documentclass[11pt]{article}
\usepackage{amsmath}
\usepackage{graphicx} 
\usepackage{fullpage}
\usepackage[utf8]{inputenc}
\usepackage{amsmath,amssymb}
\usepackage{bbm}

\DeclareMathOperator*{\argmax}{arg\,max}
\DeclareMathOperator*{\argmin}{arg\,min}
 
\usepackage{hyperref}
\usepackage{cleveref}

\usepackage{amsthm}
\newtheorem{theorem}{Theorem}[section]
\newtheorem{lemma}[theorem]{Lemma}

\newtheorem{definition}[theorem]{Definition}
\newtheorem{corollary}[theorem]{Corollary}

\newtheorem{remark}[theorem]{Remark}

\newtheorem{assumption}{Assumption}

\usepackage{comment}
\usepackage{mathtools}

\usepackage[shortlabels]{enumitem}

\newcommand{\E}{\mathbb{E}}

\usepackage[shortlabels]{enumitem}

\newcommand{\ignore}[1]{}

\usepackage[shortlabels]{enumitem}

\usepackage[shortlabels]{enumitem}

\usepackage{caption}
\usepackage{xcolor}

\let\Pr\undefined

\DeclareMathOperator*{\Pr}{\mathbb{P}}

\newcommand{\cA}{\mathcal{A}}
\newcommand{\cB}{\mathcal{B}}
\newcommand{\cC}{\mathcal{C}}
\newcommand{\cD}{\mathcal{D}}

\newcommand{\cM}{\mathcal{M}}

\newcommand{\cX}{\mathcal{X}}
\newcommand{\cY}{\mathcal{Y}}
\newcommand{\cZ}{\mathcal{Z}}

\newcommand{\eps}{{\varepsilon}}

\newcommand{\nat}[1]{\textcolor{blue}{[Natalie: #1]}}

\usepackage{nicematrix}

\usepackage{natbib}

\usepackage{algorithm}
\usepackage[noend]{algorithmic}
\usepackage{amsmath,amsthm}
\usepackage{amsfonts}
\usepackage{bbm}
\usepackage{xcolor}
\newcommand{\ar}[1]{\textcolor{blue}{[Aaron: #1]}}
\newcommand{\vg}[1]{\textcolor{orange}{[Varun: #1]}}

\newcommand{\calC}{\mathcal{C}}
\newcommand{\cV}{\mathcal{V}}

\newcommand{\SQE}{\mathrm{SQE}}
\newcommand{\SQErr}{\mathrm{SQErr}}
\newcommand{\CalDist}{\mathrm{CalDist}}
\newcommand{\ECE}{\mathrm{ECE}}

\newcommand{\Tk}[1]{T^{\ge #1}}

\newcommand{\ymk}[2]{\hat{y}^{#1,#2}_m}
\newcommand{\yhk}[2]{\hat{y}^{#1,#2}_h}

\newcommand{\bmk}[2]{b^{#1,#2}_m}
\newcommand{\bhk}[2]{b^{#1,#2}_h}

\newcommand{\pmk}[2]{p^{#1,#2}_m}
\newcommand{\phk}[2]{p^{#1,#2}_h}
\newcommand{\barpmk}[2]{\bar{p}^{#1,#2}_m}
\newcommand{\barphk}[2]{\bar{p}^{#1,#2}_h}

\allowdisplaybreaks
\makeatletter
\newenvironment{protocol}[1][htb]{%
    \renewcommand{\ALG@name}{Protocol}
   \begin{algorithm}[#1]%
  }{\end{algorithm}}
\makeatother

\title{Tractable Agreement Protocols}
\author{Natalie Collina \and Surbhi Goel \and Varun Gupta \and Aaron Roth}

\begin{document}

\maketitle
\begin{abstract}
We give an efficient reduction through which any machine learning algorithm can be converted into an interactive protocol that can interact with another party (such as a human) to reach agreement on predictions and improve accuracy. The requirements on each party are calibration conditions which are computationally and statistically tractable relaxations of Bayesian rationality --- that are sensible even in prior free settings --- and hence are a substantial generalization of Aumann's classic ``agreement theorem''  \cite{aumann1976}. In the interactive protocol, the machine learning model first produces a prediction. Then, the human responds to the model's prediction by either conveying agreement, or else providing feedback of some sort. The model then updates its state and provides a new prediction, and the human in turn may update their beliefs. The process continues until the model and the human reach agreement. 

The first setting we study  generalizes past work on Aumann's Agreement Theorem, in which the parties aim to agree on a one-dimensional expectation. At each round, each party simply communicates an estimate of their current prediction for the expectation. In this setting we recover the quantitative convergence theorem of \cite{aaronson2004complexity} (but under our much weaker assumptions). We then move on to the case in which the parties maintain beliefs about a distribution over $d$ outcomes and consider two feedback mechanisms. The first simply corresponds to a vector-valued estimate of the agents' current prediction. The second takes a decision theoretic perspective: if the human needs to take some downstream action from a finite set, and has an arbitrary utility function of their action and the outcome, then we show that the parties can communicate and reach agreement about the correct downstream action to take by simply communicating at each round the action that they believe to be utility maximizing. The number of rounds until agreement remains independent of $d$ in this case. We can also generalize our protocols to more than 2 parties, with computational complexity that degrades only linearly with the number of parties. Our protocols are based on simple, efficiently maintainable conditions  and result in predictions that are more accurate than any single party's alone. 

\end{abstract}
\setcounter{page}{0}
\thispagestyle{empty}
 \clearpage

\tableofcontents
\setcounter{page}{0}
\thispagestyle{empty}
 \clearpage

\section{Introduction}
Consider a machine learning model designed to help doctors make clinical decisions. This predictive model is trained on a much larger dataset of patients than the doctor's experience can draw on. However, it is also necessarily trained on different, and perhaps less rich data. For example, the doctor's observations often include qualitative information, such as their patients' reaction to touch, which are hard to encode as input to the model.
As a result, even if the model is more accurate on average than the doctor, there will still be situations in which the doctor ought not to follow the model's recommendation and proceed to act in accordance with their own predictions. 
In such a situation, rather than forcing the doctor to \emph{either} use the model \emph{or} choose to ignore it, we would prefer an interface through which the doctor and model can interact to update their beliefs and reach agreement on a prediction that is guaranteed to be more accurate than either of their initial predictions. When designing such an interface, we would like to be able to prove convergence and utility guarantees under minimal, computationally and statistically tractable assumptions on the interaction between the model and the doctor, for at least two reasons:
\begin{enumerate}
    \item We need to actually implement the model's interaction with the protocol. So, we want to be able to start with an arbitrary black-box model (e.g. a trained neural network) and convert it efficiently (in terms of both computation and data requirements) into an interactive system
    \item We want to make minimal assumptions on how the human will interact with the model. These assumptions should be significantly weaker than ``perfect rationality'' --- meaning that they should be satisfied by informed Bayesian reasoners --- but our results should not hinge on making a computationally implausible assumption. Rather we should make assumptions that can be satisfied in a computationally and statistically tractable way. The weaker our assumptions are, the more likely they are to be satisfied. 
\end{enumerate}

In this paper, we show how to efficiently convert an arbitrary predictive model into an interactive protocol that can be used to interact with another party in a way that quickly leads to agreement while improving accuracy under tractable assumptions. We give results across a variety of feedback models and extend our results to multiple parties. Our results rely on the theory of calibration, which has a long intellectual history \citep{dawid1982well,dawid1985calibration,foster1998asymptotic,FV99,hebert2018multicalibration,blasiok2023unifying}, and naturally live in a sequential adversarial setting that involves repeated interaction across many predictions without distributional assumptions.  Moreover, if the instances \emph{are} drawn from a prior distribution and the two parties are informed Bayesians, then we are able to give an ``online-to-one-shot'' reduction that translates our theorems to high probability guarantees for interactions on a single instance drawn from this prior. We show that all of our calibration requirements are satisfied by Bayesians updating on a correct prior, implying that our approach generalizes past works in the well-studied one-shot ``Aumann Agreement Theorem'' setting \citep{aumann1976,geanakoplos1982we,aaronson2004complexity,kong2023false,frongillo2023agreement}. In particular, we generalize the kind of instance-independent convergence bounds proven by \cite{aaronson2004complexity} in the 1-dimensional real valued setting to $d$ dimensions, both when the feedback is vector valued, and when  it is only in the form of a ``best response action''. In the latter case we get convergence at a rate that is not only independent of the complexity of the underlying prior, but also independent of the dimension $d$.

\subsection{Our Model and Results}

 Over a series of \emph{days} $t$, examples arrive. Each example has a true label $y^t\in \mathbb{R}^d$ which is initially unobserved and that at least two parties want to predict or otherwise act on (we focus on the two party case for this informal description). We call one of the parties the \emph{human}, and the other the \emph{model}. Before making their initial predictions, the model sees features $x_m^t$ relevant to the example, and the human sees a potentially different set of features $x_h^t$ (potentially over a different feature space). Based on these features, the model and the human engage in a \emph{conversation}  over a series of \emph{rounds} $k=1, \ldots, L$. The human and model alternate speaking, with the model speaking in odd rounds and the human speaking in even rounds. In each odd round $k$, the model produces a prediction $\pmk{t}{k}$ (as a function of all prior history of both conversation rounds that day as well as previous days), which is observed by the human, who in turn produces a prediction $\phk{t}{k+1}$ in round $k+1$, which may also be a function of all previously observed history. The conversation continues until at some round $k$, the pair of predictions $(\pmk{t}{k-1}, \phk{t}{k})$ (if $k$ is even) or $(\phk{t}{k-1}, \pmk{t}{k})$ (if $k$ is odd) satisfy an agreement condition, at which point both the human and the model observe the label, and time proceeds to the next day. We give  conditions---all of which are computationally and statistically tractable relaxations of full Bayesian rationality---under which the conversation is guaranteed to quickly lead to agreement on predictions that are more accurate than the initial predictions. 
\paragraph{The Canonical Setting.} In the simplest setting that we study, the labels $y^t \in [0,1]$ are one dimensional, and the predictions $\pmk{t}{k}, \phk{t}{k} \in [0,1]$ are also one-dimensional numeric values, and intended to convey a numeric estimate of $y^t$ or its expectation. We measure the accuracy of predictions using squared error. For example, the squared error of the human's initial (round 2) predictions over days is $\sum_{t=1}^T(\phk{t}{2}-y^t)^2$. In this case, we say that the human and the machine are in $\epsilon$-agreement at some round $k$ if $|\pmk{t}{k}-\phk{t}{k-1}| \leq \epsilon$ (odd $k$) or $|\phk{t}{k}-\pmk{t}{k-1}| \leq \epsilon$ (even $k$). We define a calibration condition that we call ``conversation calibration''.
Informally speaking, conversation calibration for the model requires that for each round $k$, if we consider the subsequence of days on which conversation extended to round $k$, denoted by $\Tk{k}$, the predictions made at round $k$ 
$\{\pmk{t}{k}\}_{t \in \Tk{k}}$
are calibrated to the outcome subsequence on those days 
$\{y^t\}_{t \in \Tk{k}}$
not just marginally, but \emph{conditionally} on the value of the prediction made by the human in the previous round  ($k-1$).
\begin{definition}[Informal, see Definition \ref{def:conversation-calibration}]
   We say that the model satisfies conversation calibration  with respect to the human if for all odd rounds $k$ and $v, v' \in [0,1]$:
\[ \sum_{t\in \Tk{k}} \mathbbm{1}[\pmk{t}{k}=v]\cdot\mathbbm{1}[\phk{t}{k-1}=v'](\pmk{t}{k} - y^t)=0\] Conversation calibration for the human is a symmetric condition for even rounds $k$. 
\end{definition}
Importantly, conversation calibration does \emph{not} require that the predictions be unbiased conditional on the whole conversation so far (which a correctly specified Bayesian would satisfy), but only conditional on the current prediction of the model, and the most recent prediction of its human interlocutor. This makes the condition computationally and statistically tractable to enforce using standard algorithms for online calibration (e.g. \cite{foster1998asymptotic,gupta2022online}). In fact, in our use-case it turns out to be sufficient to measure calibration error using ``distance to calibration'' \cite{blasiok2023unifying}, a recently defined relaxation of traditional calibration measures. This is useful because there are extremely simple algorithms that can make predictions with ``distance to calibration'' diminishing at much better rates than are possible for standard calibration metrics \cite{qiao2024distance,arunachaleswaran2024}. When we construct conversation algorithms from static models, we make use of the simple efficient algorithm of \cite{arunachaleswaran2024}, which can be used to bound distance to conversation calibration, as conversation calibration requires conditioning only on disjoint events. 

\begin{theorem}[Informal, see Theorem \ref{thm:reduction}]
There is a computationally efficient reduction that takes as input an arbitrary model $M$ mapping features to predictions, and outputs an algorithm that can engage in a conversation protocol. The algorithm uses the predictions of model $M$ at the first round, and is guaranteed to satisfy (approximate) conversation calibration against any agent that it converses with. 
\end{theorem}

We show that if both parties are conversation calibrated, then on a large fraction of days, the human and the model agree very quickly.
\begin{theorem}[Informal, see Theorem \ref{thm:canonical}]
If the human and model are conversation-calibrated, then for any $\epsilon,\delta \in (0,1]$ and large enough $T$, on a $1-\delta$ fraction of days, they reach $\epsilon$-agreement after at most $K = \frac{1}{\epsilon^{2}\delta}$ rounds of conversation. Furthermore, for this $1-\delta$ fraction of days, if they reach agreement in round $i$, their final predictions have a lower squared error than the base predictions of either the human or the model, by a term that scales as $\frac{i}{\delta \epsilon^{2}}$ ) (so longer conversations directly lead to correspondingly more accurate predictions). 
\end{theorem}

Paired with our algorithmic reduction, this allows us to efficiently implement conversation protocols that lead to fast agreement and are guaranteed to be accuracy improving, starting with any model $M$ (about which we make no assumptions), and any interlocutor that also satisfies conversation calibration.

Our result recovers the parameters proven by \cite{aaronson2004complexity} for the special case of agreement by fully rational Bayesian forecasters in a setting with a known prior. Moreover, our result generalizes beyond the setting of \cite{aaronson2004complexity} in a number of ways. For example, it straightforwardly generalizes to the setting in which the outcome $y^t \in [0,1]^d$ is $d$-dimensional, by requiring that the forecasts satisfy conversation calibration marginally in each coordinate. This comes at a cost of $d$ in our convergence bounds: 

\begin{theorem}[Informal, see Theorem \ref{thm:dimensions}]
When $\cY = [0,1]^d$ and the human and model satisfy conversation-calibration marginally in each coordinate (Definition \ref{def:conversation-calibration-dimensional}), then for any $\epsilon,\delta \in (0,1]$ and large enough $T$, on a $1-\delta$ fraction of days, they reach $\epsilon$-agreement after at most $K = \frac{d}{\epsilon^{2}\delta}$ rounds of conversation. Furthermore, for this $1-\delta$ fraction of days, if they reach agreement in round $i$, their final predictions have a lower squared error than the base predictions of either the human or the model, by a term that scales as $\frac{i}{\delta \epsilon^{2}}$.
\end{theorem}

We similarly show an efficient reduction that can convert an arbitrary model $M$ into an algorithm capable of engaging in a conversation protocol, that is guaranteed to satisfy conversation calibration marginally in each coordinate when interacting with any interlocutor (see Theorem \ref{thm:reduction-dimensinoal}). 

\paragraph{Agreeing on Actions}
Although our analysis of the ``canonical setting'' extends to $d$ dimensional agreement, it requires that both parties provide $d$-dimensional numeric predictions at each round. The model in our reduction is able to do this, but it would be better not to require the human interlocutor to provide numeric feedback, especially in high dimensions when $y^t \in [0,1]^d$. To avoid this, we adopt a downstream decision-making perspective. We imagine that the human has an action set $\cA$ (e.g. treatments and diagnostic tests that a doctor could order), as well as a utility function $U:\cA\times [0,1]^d \rightarrow [0,1]$ that maps an action $a \in \cA$ and a label $y \in [0,1]^d$ to a utility $U(a,y)$ that the human would like to maximize. We assume that the utility function is linear in its second argument. This captures (among other things) the scenario in which there are $d$ discrete outcomes for which the human has arbitrary utilities, the predictions are \emph{probability distributions} over these $d$ outcomes, and the human is an expectation maximizer.

Given the human's forecast for the label in an even round $\phk{t}{k}$, there is an action that they believe to be utility maximizing -- their ``best response'' action $\bhk{t}{k} = \argmax_{a \in \cA}U(a,\phk{t}{k})$. Similarly, for even rounds $k$, there is an action $\bmk{t}{k} = \argmax_{a \in \cA} U(a,\pmk{t}{k})$ that is utility maximizing given the model's forecast. Rather than measuring how accurate the predictions are with respect to squared error, we instead measure the realized utility for the sequence of predicted actions. For example, the utility of the human's initial (round 2) action predictions over days is $\sum_{t=1}^T U(b_h^{t,2},y^t)$. We say that the human and the machine are in $\epsilon$-agreement about a pair of predicted actions $(\bhk{t}{k-1},\bmk{t}{k})$ (if $k$ is odd) or $(\bmk{t}{k-1},\bhk{t}{k})$ (if $k$ is even) if both parties agree that the other's recommended action is an $\epsilon$-approximate best response (under their own predictions). In other words, if $k$ is odd then:
$$U(\bmk{t}{k},\phk{t}{k-1}) \geq U(\bhk{t}{k-1},\phk{t}{k-1})-\epsilon \qquad \textrm{and} \qquad U(\bhk{t}{k-1},\pmk{t}{k}) \geq U(\bmk{t}{k},\pmk{t}{k})-\epsilon.$$

We similarly define a calibration condition that we call ``conversation decision calibration''. We use the version of decision calibration given by \cite{noarov2023high} (a multi-dimensional generalization of ``Decision Outcome Indistinguishability'' \cite{gopalan2023loss}), a slight strengthening of the version originally studied by \cite{zhao2021calibrating}, which requires that the vector valued predictions $\phk{t}{k},\pmk{t}{k}$ be unbiased conditionally on the best response action they induce. Conversation decision calibration additionally conditions on the action most recently suggested by one's interlocutor.
\begin{definition}[Informal, see Definition \ref{def:conversation-decision}]
   We say that the model satisfies conversation decision calibration with respect to the human if for all odd rounds $k$ and  $a, a' \in \cA$:
$$\left\|\sum_{t\in \Tk{k}}\mathbbm{1}[\bmk{t}{k}=a]\mathbbm{1}[\bhk{t}{k-1}=a'](\pmk{t}{k}-y^t)\right\| = 0$$ Conversation calibration for the human is a symmetric condition for even rounds. 
\end{definition}
Like our definition of conversation calibration, conversation decision calibration only involves a small number $(|\cA|^2)$ of conditioning events, and so is computationally and statistically tractable to enforce --- in this case by using the online algorithm for making $d$ dimensional forecasts unbiased subject to an arbitrary polynomial collection of conditioning events given by \cite{noarov2023high}. Note also that the conditioning events are defined \emph{only} by the best response actions $\bmk{t}{k}, \bhk{t}{k}$ of both parties, and so obtaining conversation decision calibration does not require that either party directly communicate their numeric valued forecasts $\phk{t}{k}, \pmk{t}{k}$ (although the model might as well) --- only communication of the best response actions is required. 

We show that for any set of actions and any utility function, if both parties are conversation decision calibrated, then on a large fraction of the days, the human and model agree very quickly.
\begin{theorem}[Informal, see Theorem \ref{thm:agreement-action}]
If the human and model are conversation-decision-calibrated, then for any $\epsilon,\delta \in (0,1]$ and large enough $T$, on $1-\delta$ fraction of days, they reach $\epsilon$-agreement after at most $K = \frac{1}{\epsilon\delta}$ rounds of conversation. Furthermore, for this $1-\delta$ fraction of days, if they reach agreement in round $i$, the final predictions have higher utility than the base predictions of either the human or the model, by a term that scales with $\frac{i}{\delta \epsilon}$. 
\end{theorem}
Note that here the number of rounds until convergence has no dependence on $d$ (although how large $T$ needs to be will have a polynomial dependence on $d$).   And in contrast to even the 1-dimensional canonical setting, we get an improved dependence on $\epsilon$. Intuitively, this is because we are able to use the (linear) utility of the human for the suggested actions as a potential function across rounds, rather than the squared error of our numeric predictions.

We use an algorithm from~\cite{noarov2023high} to  construct an efficient reduction from an arbitrary model $M$ to an algorithm capable of engaging in a decision-conversation protocol, which is guaranteed to satisfy conversation decision calibration when interacting with any interlocutor (Theorem~\ref{thm:reduction-action})


\paragraph{One-Shot Guarantees for Bayesians.}
Bayesian posterior beliefs, when computed from a known (and correct) prior are known to be well calibrated \citep{dawid1982well}. We show that this extends to our notions of conversation calibration: when instances are drawn i.i.d. from a fixed and known prior, then a Bayesian, reporting at every round their posterior expectation for $y$, will satisfy  our notions of decision calibration, no matter how their interlocutor is making predictions. The most immediate implication of this is that the calibration assumptions that our convergence results rely on are all strict relaxations of Bayesian rationality. However it also allows us to lift all of our convergence guarantees to the ``one-shot'' setting when two Bayesians with a shared prior are conversing with one another. Rather than speaking of a sequence of conversations over days and making guarantees on the maximum length of  conversations on all but a $1-\delta$ fraction of days, we can make exactly the same guarantees on the length of a single conversation between two Bayesians, with probability $1-\delta$ over the draw of the instance from their commonly shared prior distribution. For example, when applied to our 1-dimensional convergence result in the canonical setting, we recover the Theorem of \cite{aaronson2004complexity}: Two Bayesians will reach $\epsilon$-agreement after $k=O(1/\epsilon^2 \delta)$ many rounds with probability $1-\delta$ over the draw of the instance from the prior. Our other results  generalize \cite{aaronson2004complexity} and lead to new theorems about the rate of convergence in the Bayesian agreement setting -- agreement on $d$-dimensional expectations and agreement using action feedback.

\begin{theorem}[Informal, see Corollary \ref{corr:Bayesian}]
Fix any prior $\cD$ over triples $(x_h,x_m,y)$. For an instance drawn from $\cD$, with probability $1-\delta$ over the draw of the instance:
\begin{enumerate}
    \item In the $d$-dimensional ``full-feedback'' setting  two Bayesian parties agree after exchanging at most $K \leq \frac{3d}{\epsilon^{2}\delta}$ messages.
    \item In the action-feedback setting,  two Bayesian parties agree after exchanging at most $K \leq \frac{3}{2\eps \delta}+1$ messages --- i.e. they obtain a dimension independent convergence rate. 
\end{enumerate}
\end{theorem}

\paragraph{Extension to $n$ Parties.}
We can extend all of our results from $2$ parties to $n$ parties, with only a polynomial overhead in $n$ in terms of computational and statistical complexity of the $n$ parties, and  no dependence on $n$ in terms of how many times each party must speak before agreement (in total the number of ``rounds'' of conversation increases by a factor of $n$ simply because it now takes $n$ rounds in between each agent speaking two times consecutively).   To do this, we only require that every pair of parties satisfy \emph{pairwise} conversation calibration conditions with each other --- that is, every party must satisfy $n-1$ conversation calibration conditions with respect to the $n-1$ other parties in the interaction. In fact, something weaker suffices: there must only be a single distinguished party who satisfies these $n-1$ conversation calibration conditions; the remaining parties only need to be conversation calibrated with the distinguished party. This remains a tractable calibration condition to satisfy using the algorithm of \cite{noarov2023high} despite the fact that these conversation calibration conditions now require that forecasts are unbiased conditional on intersecting events. Crucially we do not require that the parties are calibrated conditional on the \emph{joint} forecasts of all $n-1$ other parties, which would have introduced an exponential dependence on $n$. We state the theorem here for the canonical case; similar theorems hold for the $d$-dimensional and action feedback case.
\begin{theorem}[Informal, see Theorem \ref{thm:nparties}]
Suppose there are $n$ agents. If agent $1$ is conversation-calibrated (marginally) with respect to each agent $2,\ldots,n$, and the agents $2,\ldots,n$ are each conversation calibrated with respect to agent $1$, then on a $1-\delta$ fraction of days, they reach $\epsilon$-agreement after at most: $K \leq \frac{2n}{\epsilon^2\delta}$ rounds of conversation.
\end{theorem}

\subsection{Related Work}
\paragraph{Agreement.} Aumann's classic ``agreement theorem'' \cite{aumann1976} states that two Bayesians with a common and correct prior, who have \emph{common knowledge} of each other's posterior expectation of any predicate must have the same posterior expectation of that predicate. This sparked a very large literature on agreement amongst Bayesians --- we touch upon only the most related work here. ``Common Knowledge'' is the limit of an infinite exchange of information, but Geanakoplos and Polemarchakis \cite{geanakoplos1982we} showed that whenever the underlying state space is finite, then agreement occurs after a finite number rounds (depending on the cardinality of the state space) in which the information exchanged in each round is the posterior expectation of each party. Aaronson \cite{aaronson2004complexity} showed  that for 1-dimensional expectations, $\epsilon$-approximate agreement can be obtained (with probability $1-\delta$ over the draw from the prior distribution) after the parties exchange only $O(1/\epsilon^2\delta)$ messages. Notably this bound is independent of the representation size or complexity of the underlying prior distribution. Two papers \cite{kong2023false,frongillo2023agreement} study conditions under which Aumannian agreement implies information aggregation --- i.e. when ``agreement'' is reached at the same posterior belief that would have resulted had the two parties shared all of their information, rather than interacting within an agreement protocol. There is also a large literature that studies multi-party agreement amongst Bayesians connected via a communication network --- see e.g. \cite{geanakoplos1982we,gale2003bayesian,aaronson2004complexity,mossel2014asymptotic,deshpande2022agreement}. Part of this literature (e.g. \cite{gale2003bayesian,mossel2014asymptotic}) studies settings in which the beliefs of the parties are not directly observed, but rather what action they take is, under the presumption that they take a utility maximizing action. In general this literature is interested in exact asymptotic agreement. These papers also all assume that there is a commonly known prior and that all parties are able to compute correct posterior expectations for the predicate of interest, despite the fact that this might be computationally intractable. Aaronson \cite{aaronson2004complexity} gives a computational reduction from the problem of participating in an agreement protocol to the problem of computing (and sampling from) correct posterior distributions, conditional on any vector of features that might be observed by either party. This might itself be a computationally hard task, and the reduction requires a number of calls to this posterior-computation oracle that is super-exponential in $1/\epsilon$, but independent of the cardinality of the state space.  Our primary point of departure from this literature is that we ask for algorithms that are truly computationally tractable (i.e. worst-case polynomial time in all parameters) and make no distributional assumptions, although when there is a commonly known prior and agents are Bayesians, we recover theorems in the classical Aumannian setting. This leads to new quantitative agreement theorems in the setting and style of \cite{aaronson2004complexity} (i.e. bounds that depend only on approximation parameters and are independent of the complexity of the instance) --- in $d$ dimensional settings. In particular,  that agents providing only action feedback in $d$-dimensional belief spaces will arrive at $\epsilon$-agreement after only $O(1/\epsilon \delta)$ many rounds of conversation with probability $1-\delta$ (independent of the dimension $d$). 

\paragraph{Calibration.} Our techniques are rooted in the ability to maintain \emph{calibrated} forecasts in online adversarial settings, which was first shown by Foster and Vohra \cite{foster1998asymptotic}. Calibration itself dates back to Dawid \cite{dawid1982well,dawid1985calibration}, who also showed that Bayesians with correctly specified priors are calibrated \cite{dawid1982well}. Conditional calibration guarantees also have a long history \cite{dawid1985calibration,sandroni2003calibration,foster2006calibration} with a recent seminal formalization as  \emph{multicalibration} \cite{hebert2018multicalibration} which can be obtained with good rates in both the batch and online adversarial settings \cite{gupta2022online}. The traditional calibration measure of ``expected calibration error'' has a number of shortcomings; the most relevant for us is that it cannot be obtained with $O(\sqrt{T})$ rates in online adversarial settings \cite{qiao2021stronger,dagan2024improved}. This has led to a recent exploration of alternative calibration measures, notably ``distance to calibration'' \cite{blasiok2023unifying}. Distance to calibration \emph{can} be obtained at $O(\sqrt{T})$ rates in online adversarial settings with extremely simple, deterministic algorithms \cite{qiao2024distance,arunachaleswaran2024}, and turns out to be sufficient for our application. In particular our reduction in the canonical case uses the algorithm of \cite{arunachaleswaran2024}. In our ``action feedback'' setting we use a variant of ``decision calibration'' \cite{zhao2021calibrating}, which can similarly be guaranteed in online adversarial settings with good rates, using the algorithm of \cite{noarov2023high}. This is related to a line of recent work exploring notions of calibration tailored to downstream decision-making \cite{kleinberg2023u,noarov2023high,RS24,HW24}.

Several papers \cite{camara2020mechanisms,collina2023efficient} have replaced traditional assumptions of Bayesian rationality (and common prior assumptions) with calibration assumptions in \emph{principal agent} problems arising e.g. in contract theory and Bayesian Persuasion. In particular, \cite{collina2023efficient} shows how to do this with tractable decision calibration conditions. Beyond this, the most thematically related use of calibration is its use as an ensembling method. \cite{garg2019tracking} shows how to produce a model that is ``cross calibrated'' to two models, and is more accurate than each while improving various fairness measures of decisions downstream of the model. \cite{roth2023reconciling} shows how to use cross-calibration to resolve ``predictive multiplicity'', and derive a single more accurate model from any pair of models that are equally accurate and yet frequently disagree. This kind of ensembling was recently extended to agreement for downstream actions \cite{du2024reconciling}  and for ensembling models for high dimensional downstream optimization problems \cite{globusharris2024model}. Alur, Raghavan, and Shah  employ a similar model ensembling approach motivated by human-AI collaboration \cite{alur2024human}. Informally, they learn a model in a batch setting that is cross-calibrated to the fixed judgments of a human. What distinguishes our work from the line of work using calibration for ensembling (aside from the fact that we work in the online adversarial setting) is that prior work in this area treats the models to be ensembled as static. That is, the models to be ensembled are defined by fixed mappings from features to predictions, and do not update their beliefs as a function of interaction with other models. As a result, these methods cannot be applied to Bayesian-like entities which requires the kind of interactive conversation protocol we adopt in this work.

\paragraph{Multi-agent Debates with LLMs.} To improve the accuracy of the responses of large language model (LLM) generations, recent work \cite{du2024improving,liang2023encouraging,chan2024chateval} has proposed the multi-agent debate approach in which two (or more) LLMs ``debate'' their individual responses and reasoning processes in multiple rounds until they converge to a final answer, and then a ``judge'' (often another LLM or human) validates the final answer. Here, debate loosely refers to the two LLMs getting to see each other's responses after each round and update their subsequent responses. This shares notable similarities to our agreement protocol where the LLMs map to the agents, the messages map to the generations of each LLM in each round, and the judge maps to the outcome label at each day. Moreover, they share the general motivation to improve the accuracy of the agents using interaction and the assumption that agents behave in good faith. In contrast to our work which deals with numerical predictions and makes formal calibration assumptions on the agents, multi-agent debates operate in natural language under less formal assumptions. We believe our framework of agreement protocols could potentially be adapted to analyze multi-agent LLM debate dynamics and explain why LLMs reach consensus and improve overall accuracy. Additionally, our techniques to enforce the calibration conditions could be useful to improve the efficiency and performance of LLM debates. We note that prior work \cite{deshpande2022agreement} has also discussed viewing communicating LLMs through the lens of agreement. However, their work assumes that LLMs are purely Bayesian agents with a shared prior, and they focus only on statistical efficiency of reaching agreement on a network, not the rate of convergence to agreement.

\section{Preliminaries}

In most of the paper we study a setting with two agents, whom we call the \emph{human} and the \emph{model}  (in \Cref{sec:multiagent} we generalize to the setting to $n \geq 2$ agents). Both the human and model are able to make predictions about a label not just in isolation (given features), but as a function of an interaction that they have had with another agent. The agents interact to make predictions over a sequence of days $t= 1, \ldots, T$. We let $\cX_h$ and $\cX_m$ denote feature spaces for the human and model, respectively. We let $\cY$ represent the outcome (label) space, which we always take to be real or vector valued, so that we can sensibly speak of expectations over it. For notational simplicity, we will assume that $\cY$ is convex, so that expectations over $\cY$ are themselves elements of $\cY$, although this is not necessary. 

On each day $t$, the human and model aim to reach agreement, with respect to some agreement condition, on their predictions of that day's outcome $y^t$ based on the features they each see: $x^t_h$ and $x^t_m$, respectively. They do so by conversing over a series of \emph{rounds}  $k = 1 \, \ldots, L$. 
The human and model will alternate speaking, and we suppose that the model acts in odd numbered rounds; the human acts in even numbered rounds. 
In an odd round $k$, the model sends a message $\pmk{t}{k}$, and then in the next round $k+1$, the human responds with a message $\phk{t}{k+1}$. We use the subscript $h$ and $m$ for readability, so there is a clear distinction between the human and model messages when possible. However, since whether the human or model is speaking is simply a function of the parity of the round $k$, we can also write $p^{t,k}$ as shorthand for $p_h^{t, k}$ or $p_m^{t,k}$ when the round $k$ is even or odd, respectively. We write $\Omega_h$ for the message space of the human and $\Omega_m$ for the message space of the model.
At each round $k$ when they are speaking, an agent has an underlying prediction of the (expectation of the) label, denoted $\ymk{t}{k}$ and $\yhk{t}{k}$ respectively. This underlying prediction can be a function of everything the agent has observed so far --- the features relevant to the instance, the messages sent by the other party, and past outcomes on previous days. The message each agent sends at each round will be a function of this underlying prediction. For example, the messages sent might be the underlying predictions themselves (as in the full feedback setting we study in Sections  \ref{sec:full-info} and \ref{sec:d-dimensions}) ---  but the messages might also be some ``coarsening" of the prediction, as in the action feedback setting we study in Section \ref{sec:action}.
The day terminates once an agreement condition is met, at which point that day's label $y^t$ is revealed to both parties.

\subsection{Agreement Protocols}

We study a variety of settings, each of which is instantiated by the label space, the message space of each of the parties, and an agreement condition.  
We begin by defining a generic agreement condition, which we can instantiate for each particular setting.
Informally, the agreement condition takes in the messages and underlying predictions of  each agent, and decides if they are sufficiently close to terminate the conversation for that day.

\begin{definition}[Agreement Condition]
    An agreement condition is a function that determines when the human and model's predictions are ``$\eps$-close'', for any $\eps > 0$ as a function of their most recently sent messages and predictions: $\textsc{Agree}_{\eps}: (\Omega_h, \cY) \times (\Omega_m, \cY) \to \{0, 1\}$. An agreement condition should be the conjunction of two conditions (one for the model and one for the human), each of which can be evaluated with only knowledge of \emph{their own} predictions and their counter-party's message. In other words we should be able to write $\textsc{Agree}_{\eps}(p_h,\hat y_h,p_m,\hat y_m) = \textsc{Agree}^h_{\eps}(p_m,p_h,\hat y_h)\cdot\textsc{Agree}^m_{\eps}(p_m,p_h,\hat y_m)$ for some pair of functions $\textsc{Agree}^h_{\eps},\textsc{Agree}^m_{\eps}$.
\end{definition}

\begin{remark}
In practice, whether each agent is in agreement with the other is determined by the agent---this is why we want agreement conditions to be the conjunction of a pair of conditions each of which can be evaluated by each agent in isolation. The formalism of an ``agreement condition'' is only to let us easily describe and instantiate our various settings. 
\end{remark}

As an example, we can consider the simplest agreement condition we use. This is the agreement condition for the full feedback, one-dimensional prediction (``canonical") setting: when $\Omega_h = \Omega_m = \cY = [0, 1]$. 

\begin{definition}[Agreement Condition in the Canonical Setting] \label{def:agree-canonical}
    The agreement condition in the canonical setting is the function $\textsc{Agree-Canonical}_{\eps}: \Omega_h \times \cY \times \Omega_m \times \cY \to \{0, 1\}$ defined as:
    \begin{align*}
        \textsc{Agree}_{\eps}(p, y_h, q, y_m) =
        \begin{cases}
            1, & \text{ if } |p-q| < \eps \\
            0, & \text{ otherwise.}
        \end{cases}
    \end{align*}
\end{definition}

We formalize the interaction between the two agents in Protocol \ref{alg:general-agreement} --- a generic ``agreement protocol''---which can be instantiated with the particulars of each setting we study. 

\begin{protocol}[ht]
\begin{algorithmic}
    \STATE{ {\bf Input} $(\Omega_h, \Omega_m, \cY, \textsc{Agree}_{\epsilon}$) }
    \FOR{each day $t = 1, \ldots$}
        \STATE Receive $x^t = (x^t_h,x^t_m)$. The model sees $x^t_m$ and the human sees $x^t_h$.
        \FOR{each round $k = 1, 2, \ldots,L$}
            \IF { $k$ is odd}
                \STATE The Model predicts $\ymk{t}{k} \in \cY$, and sends the Human $\pmk{t}{k} \in \Omega_m$ 
                \IF{ $ \textsc{Agree}_{\eps}( \phk{t}{k-1}, \yhk{t}{k-1}, \pmk{t}{k}, \ymk{t}{k} ) $} 
                \STATE Return $\pmk{t}{k}$ and break out of loop
                \ENDIF
            \ENDIF
            \IF{ $k$ is even}
                \STATE The Human  predicts $\yhk{t}{k} $, and sends the model $\phk{t}{k} \in \Omega_h$
                \IF{ $ \textsc{Agree}_{\eps}( \phk{t}{k}, \yhk{t}{k}, \pmk{t}{k-1}, \ymk{t}{k-1} ) $} 
                    \STATE Return $\pmk{t}{k-1}$ and break out of loop
                \ENDIF
            \ENDIF
        \ENDFOR
        \STATE{The Human and Model observe $y^t \in \cY$}
    \ENDFOR

\end{algorithmic}
\caption{\textsc{General $\eps-$Agreement Protocol}}  \label{alg:general-agreement}
\end{protocol}

When interacting within an agreement protocol, we say that on day $t$ the two agents agree after $k$ rounds of conversation if the agreement condition is met at round $k$ of day $t$.

\subsubsection{Instantiating Different Feedback Models}
\label{sec:prelims-agreement-models}

We can now formally specify the various settings we study. These will vary in the label space, the message space for each participant, the mapping between predictions and messages, and the agreement condition. 

\paragraph{Full Feedback} 

The first setting we study is the full feedback, one-dimensional prediction, or ``canonical", setting.
Here, the human and model are both communicating their precise point predictions for the (expectation of the) unknown label to each other. Agreement will refer to when the human and model's predictions are sufficiently close numerically.

\begin{definition}[Canonical Setting] \label{def:setting-canonical}
    The canonical setting refers to Protocol \ref{alg:general-agreement} instantiated with $\Omega_m = \Omega_h = \cY = [0,1]$, messages  $\pmk{t}{k} = \ymk{t}{k}$ and $\phk{t}{k} = \yhk{t}{k}$, and the agreement condition $\textsc{Agree}_{\eps} = \textsc{Agree-Canonical}_{\eps}$ (Definition \ref{def:agree-canonical}).
\end{definition}
This naturally extends to the $d$-dimensional setting, in which we measure agreement using the $\ell_\infty$ norm: 

\begin{definition}[Agreement Condition in the $d$-dimensional Setting] \label{def:agree-dimensions}
    The agreement condition in the $d-$dimensional setting is the function $\textsc{Agree-dDim}_{\eps}: \Omega_h \times \cY \times \Omega_m \times \cY \to \{0, 1\}$ defined as:
    \begin{align*}
        \textsc{Agree-dDim}_{\eps}(p, y_h,  q, y_m) =
        \begin{cases}
            1, & \text{ if } \|p - q\|_{\infty} < \eps \\
            0, & \text{ otherwise.}
        \end{cases}
    \end{align*}
\end{definition}

\begin{definition}[$d$-dimensional Full Feedback Setting] \label{def:setting-dimensions}
    The $d$-dimensional full feedback setting refers to Protocol \ref{alg:general-agreement} instantiated with $\Omega_h = \Omega_m = \cY = [0,1]^d$, messages $\pmk{t}{k} = \ymk{t}{k}$ and $\phk{t}{k} = \yhk{t}{k}$, and the agreement condition $\textsc{Agree}_{\eps} = \textsc{Agree-dDim}_{\eps}$ (Definition \ref{def:agree-dimensions}).
\end{definition}

\paragraph{Action Feedback}

In this setting, we study a human and a model who aim to agree on an action to take when their predictions are used to inform downstream decision-making. We model the human as having a known  action set $\cA$ and utility function $U: \cA \times \cY\rightarrow [0,1]$. The human and model are both maintaining predictions of the underlying state -- which is here a $d$-dimensional vector --- $\cY = [0,1]^d$ --- and are using their predictions to choose an action that is utility maximizing given the forecast. 
In this setting, the human and model do not exchange their estimates of the state directly, but instead simply suggest actions to one another (utility maximizing actions under their forecasts): $\Omega_h = \Omega_m = \cA$. 
Here, our notion of $\epsilon$-agreement will be that both parties agree that the action suggested by the other party obtains utility that is within $\epsilon$ of the best-response action, as measured under their own forecasts. 

\begin{definition}[Agreement Condition in the Action Feedback Setting] \label{def:agree-action}
    The agreement condition in the action feedback setting is the function $\textsc{Agree-Action}_{\eps}: \Omega_h \times \cY \times \Omega_m \times \cY \to \{0, 1\}$ defined as:
     \begin{align*}
        \textsc{Agree-Action}_{\eps}(p, y_h, q, y_m) =
        \begin{cases}
            1, & \text{ if } U(p, y_m) \geq U(q, y_m)-\epsilon  \text{ and } U(q, y_h)) \geq U(p, y_h) -\epsilon \\
            0, & \text{ otherwise.}
        \end{cases}
    \end{align*}
\end{definition}

\begin{definition}[Action Feedback Setting] \label{def:setting-action}
    The action feedback setting refers to Protocol \ref{alg:general-agreement} instantiated with $\Omega_h = \Omega_m = \cA$, $\cY = [0,1]^d$, messages  $\pmk{t}{k} = \argmax_{a \in \cA} U(a,\ymk{t}{k})$ and $\phk{t}{k} = \argmax_{a \in \cA} U(a,\yhk{t}{k})$, and the agreement condition $\textsc{Agree}_{\eps} = \textsc{Agree-Action}_{\eps}$ (Definition \ref{def:agree-action}).
\end{definition}

Here we state the necessary assumptions for the utility functions our theorems will apply to, following the formalism of~\cite{noarov2023high}: 
\begin{assumption}[Utility $U(\cdot, \cdot)$]
\label{def:utility}
    The utility function $U: \mathcal{A} \times \mathcal{Y} \rightarrow [0,1]$ maps an action $a$ and a vector valued outcome $y$ to a real number $U(a,y)$.  We assume that for every action $a \in \mathcal{A}$:

    \begin{itemize}
        \item $U(a,\cdot)$ is linear in its second argument: for all $\alpha_{1},\alpha_{2}\in \mathbb{R}$, $y_{1},y_{2} \in \mathbb{R}^{d}$, $$U(a,\alpha_{1}y_{1} + \alpha_{2}y_{2}) = \alpha_{1}U(a,y_{1}) + \alpha_{2}U(a,y_{2})$$
        \item $U(a,\cdot)$ is L-lipschitz in its second argument in the L1-norm: for all $y_{1}$, $y_{2} \in \mathbb{R}^{d}$, $$|U(a,y_{1}) - U(a,y_{2})| \leq L \|y_{1} - y_{2} \|_{1}. $$
    \end{itemize}
\end{assumption}
\begin{remark}
    One natural special case is when $y$ represents a probability distribution over $d$ discrete outcomes $c_{1},\ldots,c_{d}$, such that there is an arbitrary mapping $M(a,c)$ from action/outcome pairs to utilities $[0,1]$. In this case, $U(a,y)$ represents the expected utility of the action $a$ over the outcome distribution, which is linear in $y$ by the linearity of expectation.  The utility function is $L$-Lipschitz in the $L^1$-norm, where $L = \max_{a,c_{1},c_{2}}(M(a,c_{1}) - M(a,c_{2}))$. So this class of utility functions naturally captures any risk neutral decision maker with $d$ payoff relevant states, but is more general.
\end{remark}

\subsection{Algorithms for Interaction}
An agreement protocol as we have defined it is used by two agents who are able to update their predictions not only as a function of the features they have observed, but as a function of an interaction with another agent. We will want to convert static models (which map features to predictions) into such interactive algorithms. In order to define such algorithms,  it will be useful to establish a notation that refers to different pieces of information that both parties will have available to them at different times in Protocol \ref{alg:general-agreement}, which they can use in their predictions.

We refer to the history of interaction \emph{within} any given day $t$ as a ``conversation." This is, informally, the sequence of messages exchanged by the human and the model specifically about the currently unknown label $y^t$. Recall that the model and human speak in alternating (odd and even numbered, respectively) rounds.

\begin{definition}[Conversation $C$] 
    A conversation between the human and model on day $t$ over rounds 1 to $\ell$ is denoted by $C^{t,1: \ell} \in \{ \Omega_m \cup \Omega_h \}^{\ell},$ 
    is a sequence of $\ell$ messages: 
    \[
    C^{t,1: \ell} := \begin{cases}
        (\pmk{t}{1}, \phk{t}{2}, \pmk{t}{3}, \phk{t}{4}, \ldots, \pmk{t}{\ell}) &\text{if }\ell \text{ is odd},\\
        (\pmk{t}{1}, \phk{t}{2}, \pmk{t}{3}, \phk{t}{4}, \ldots, \phk{t}{\ell}) &\text{otherwise.}
    \end{cases}
    \]
    We refer to the full conversation at day $t$ as $C^t$. We define $\calC^{\ell}$ to be the space of all possible conversations of length $\ell$ and  $\calC = \bigcup_{\ell > 0} \calC^{\ell}$ represent all possible conversations.
\end{definition}

\begin{definition}[Conversation Length]
    We define $C^t$ to be the conversation at day $t$ and $\ell^t$ to be the length of $C^t: \ell^t = |C^t|$.
\end{definition}

It will often be useful to consider subsequences of our objects---messages, predictions, and labels---within a certain round. We provide notation for this below.

\begin{definition}[Round Subsequence]\label{def:roundsub}
For a fixed round $k$, we define $\Tk{k}$ to be the subsequence of days on which conversation reaches round $k$, that is, $\Tk{k}:= \{t \in \{1, \ldots, T\} ~|~ \ell^t \ge k\}$.
\end{definition}

\begin{definition}[Message Subsequence $p_{m}^{S,k}$]
For some set $S \subseteq \{1, \ldots, T\}$, we define $p_{m}^{S,k}$ as $\{ p^{t,k}_{m} : t \in S \cap \Tk{k} \}$,  the subsequence of model predictions at round $k$ corresponding to the subsequence of days $t$ which reach round $k$ and which are in the set $S$. We will similarly use the notation $p^{S, k}_h$, $\hat{y}^{S, k}_m$, and $\hat{y}^{S, k}_h$, to refer, respectively, to the human messages, model predictions, and human predictions over subsequences constrained in this way.
\end{definition}


We refer to the history of interaction \emph{across} multiple days as a ``message transcript." It is an object that records the interactions between the agents and is visible to both, and which they can use to make their predictions (unlike the ``prediction transcript" which we will define immediately following).

\begin{definition}[Message Transcript $\mu^{1:T}$] \label{def:message-transcript}
    A message transcript $\mu^{1:T} \in \{ \calC \times \cY \}^T $ is a sequence of conversation, outcome pairs over $T$ days: 
    \begin{align*}
        \mu^{1:T} = \left[ (C^1, y^1), \ldots , (C^T, y^T) \right].
    \end{align*}
    We define $\cM^{T}$ to be the space of all possible message transcripts over $T$ days and $\cM = \bigcup_{T > 0}\cM^{T}$ to be the space of all possible message transcripts.
   
    We define $\mu^{t,:}$ to be the restriction of the message transcript to the elements relevant to day $t$ --- this is simply the record of the conversation at day $t$ paired with the outcome at day $t$:
    \begin{align*}
        \mu^{t,:} = \begin{cases}
            \left((\pmk{t}{1},\phk{t}{2}, \ldots, \pmk{t}{\ell^t}), y^t \right) & \text{if } \ell^t \text{ is odd,}\vspace{1ex}\\
            \left((\pmk{t}{1},\phk{t}{2}, \ldots, \phk{t}{\ell^t}), y^t \right) & \text{otherwise.}
        \end{cases}
    \end{align*}
    Similarly, we define $\mu^{:,k}$ to be the restriction of the message transcript to the elements relevant to only round $k$ of conversation across the subsequence of days that reach round $k$ ($\Tk{k}$): 
    \begin{align*}
        \mu^{:,k} = \begin{cases}
            \left[(\pmk{t}{k}, y^t)~\middle|~ t \in \Tk{k}\right] & \text{if } k \text{ is odd,}\vspace{1ex}\\
            \left[(\phk{t}{k}, y^t)~\middle|~ t \in \Tk{k}\right] & \text{otherwise.}
        \end{cases}
    \end{align*}
\end{definition}

It will also be useful to be able to refer to the sequence of predictions made by the human and model across particular days or rounds. Note that depending on the setting we are working in, this ``prediction transcript'' will not generally be visible to both players (each player always observes their own predictions, but only the messages sent by the other):

\begin{definition}[Prediction Transcript $\pi^{1:T}$] \label{def:prediction-transcript}
    A prediction transcript $\pi^{1:T} \in \left\{ \bigcup_{\ell > 0} (\cY)^{\ell} \times \cY \right\}^T $ is a sequence of tuples of predictions over rounds made by the model and human (alternating across rounds), and the outcome, over $T$ days:
    \begin{align*}
        \pi^{1:T} = \left[ 
        \left(\ymk{1}{1}, \yhk{1}{2}, \ymk{1}{3}, \ldots \ymk{1}{\ell_1}, y^1\right), 
        ,  \ldots,  
        \left(\ymk{T}{1}, \yhk{T}{2}, \ymk{T}{3}, \ldots \ymk{T}{\ell^T}, y^T\right)        \right]
    \end{align*}
    Similar to Definition \ref{def:message-transcript}, we define $\pi^{t,:}$ to be the restriction to elements relevant to day $t$ and $\pi^{:, k}$ to be the restriction to only round $k$ of conversation across days as follows:
    \begin{align*}
        \pi^{t,:} = \begin{cases}
            \left((\ymk{t}{1}, \yhk{t}{2}, \ymk{t}{3}, \ldots, \ymk{t}{\ell_t}), y^t) \right) & \text{if } \ell^t \text{ is odd,}\vspace{1ex}\\
            \left((\ymk{t}{1}, \yhk{t}{2}, \ymk{t}{3}, \ldots, \yhk{t}{\ell_t}), y^t) \right) & \text{otherwise.}
        \end{cases} \qquad
        \pi^{:,k} = \begin{cases}
            \left[(\ymk{t}{k}, y^t)~\middle|~ t \in \Tk{k}\right] & \text{if } k \text{ is odd,}\vspace{1ex}\\
            \left[(\yhk{t}{k}, y^t)~\middle|~ t \in \Tk{k}\right] & \text{otherwise.}
        \end{cases}
    \end{align*}

    Finally, we will also use the restriction $\pi^{1:T}_h$ and $\pi^{1:T}_m$ the prediction transcript restricted to the human and model predictions, respectively, through day $T$.
\end{definition}

Finally, we introduce notation for the length $T$ sequence of final ``agreed upon'' predictions that result from a sequence of conversations up through round $k$. This will be a useful object to express our main results relating the utility of the predictions that the human or model begin with to the utility of the  predictions that result from the interaction at various points in the protocol. At each day $t$, if the conversation has not halted at (even) round $k$, then $\bar{p}_h^{t, k} =  p_h^{t, k}$ --- i.e. the prediction that we consider at that round is the one most recently made by the human. On the other hand, if the conversation ended at agreement on day $t$ \emph{before} round $k$, then we take $\bar{p}_h^{t, k}$ to be the final prediction that was made on the round at which the conversation ended in agreement. Formally, for any even round $k$:

\begin{align*}
    \bar{p}_h^{t, k} = \begin{cases}
    p_h^{t, k} & \text{ if } t \in \Tk{k} \\
    p^{t, k'} & \text{ otherwise, where } k' = \max \{j : t \in \Tk{j}\} .
\end{cases}
\end{align*}
We refer to the full sequence of $\bar{p}_h^{t, k}$ over $T$ as $\bar{p}_h^{1:T, k}$.
Symmetrically, we define the analogous sequence $\barpmk{1:T}{k}$ for the model, where
\begin{align*}
    \barpmk{t}{k} = \begin{cases}
    p_m^{t, k} & \text{ if } t \in \Tk{k} \\
    p^{t, k'} & \text{ otherwise, where } k' = \max \{j : t \in \Tk{j}\}.
\end{cases}
\end{align*}

With these definitions in hand, we can now give a formal specification of the types of algorithms we will be using in our results. 

\begin{definition}[Model Algorithm $M$]
    The Model's algorithm $M: \cM \times \Pi \times \mathcal{C} \times \mathcal{X}_{m} \to \Delta \Omega_m$ is a mapping from a $t$-length message transcript, a prediction transcript of the model's predictions through round $t$ $\pi^{1:t}_m$, an $\ell$-length conversation, and a feature vector $x^{t+1}_{m}$ to a distribution over messages $\pmk{t+1}{\ell+1}$ for day $t+1$ in round $\ell+1$. 
\end{definition}

\begin{definition}[Human Algorithm $H$]
    The Human's algorithm $H: \cM \times \Pi \times \mathcal{C} \times  \mathcal{X}_{h} \to \Delta \Omega_h$ is a mapping from a $t$-length message transcript, a prediction transcript of the humans's predictions through round $t$ $\pi^{1:t}_h$, an $\ell$-length conversation, and a feature vector $x^{t+1}_{h}$ to a distribution over messages $\phk{t+1}{\ell+1}$ for day $t+1$ in round $\ell+1$. 
\end{definition}

\subsection{Calibration} \label{sec:prelims-calibration}

The main focus of our work is studying computationally tractable conditions under which the two parties achieve fast agreement in the models described in Section \ref{sec:prelims-agreement-models}. The conditions that we study (and enforce) will be \emph{calibration} conditions of various sorts. In this section we give the basic calibration definitions that we will be working with.

The standard measure of calibration of some sequence of predictions $p^{1:T}$ to outcomes $y^{1:T}$ in a sequential prediction setting is \emph{expected calibration error}, defined as follows.

\begin{definition}[Expected Calibration Error] Given a sequence of predictions $p^{1:T}$ and outcomes $y^{1:T}$, their expected calibration error is,
\[
\ECE(p^{1:T}, y^{1:T}) = \sum_{p \in [0,1]} \left| \sum_{t=1}^T \mathbbm{1}[p^t = p] (p^t - y^t) \right|
\]

Here the outer sum is over the values $p$ that appear in the sequence $p^{1:T}$. 
\end{definition}

We will sometimes measure calibration error of a sequence instead using \emph{distance to calibration}, first defined by \cite{blasiok2023unifying} (we here use the definition given by \cite{qiao2024distance} in the sequential setting). Distance to calibration measures the  $\ell_1$ distance between a sequence of predictions and the closest sequence of  \textit{perfectly calibrated} predictions. 
\begin{definition}[Distance to Calibration] Given a sequence of predictions $p^{1:T}$ and outcomes $y^{1:T}$, the distance to calibration is,
\[
\CalDist(p^{1:T}, y^{1:T}) = \min_{q^{1:T} \in \cC(y^{1:T})} \left\|p^{1:T} - q^{1:T}\right\|_1
\]    
where $\cC(y^{1:T}) = \{ q^{1:T} : \ECE(q^{1:T}, y^{1:T}) = 0 \}$ is the set of predictions that are perfectly calibrated against outcomes $y^{1:T}$. 
\end{definition}

Calibration has a close relationship to squared error, which we will use as a potential function in some of our analyses. Below we define the squared error of a sequence of predictions relative to a sequence of outcomes:

\begin{definition}[Squared Error] Given a sequence of predictions $p^{1:T}$ and outcomes $y^{1:T}$, the squared error between them is,
  \begin{align*}
       \SQE (p^{1:T},y^{1:T}):= \sum_{t \in [T]}(p^{t} - y^{t})^{2}. 
  \end{align*} 
  We will overload this notation for the special case of constant sequences $p^{1} = \ldots = p^{T} = p$:
    \begin{align*}
       \SQE (p,y^{1:T}):= \sum_{t \in [T]}(p - y^{t})^{2}.
  \end{align*} 
\end{definition}

\subsubsection{Conversation Calibration}
We now define a new notion of calibration that we will make use of in the ``canonical'' setting, that we call \emph{conversation calibration}. Informally, an agent is \emph{conversation calibrated} if for every round of conversation $k$, the sequence of predictions (over days $t$) that they make at round $k$ of conversation is calibrated not just marginally, but \emph{conditionally} on the value of the prediction that the other agent made at round $k-1$. In fact, without making assumptions on the other agent, it will not be possible to give calibration guarantees that hold conditional on their predictions, because these may come from an arbitrarily large range. So instead we will condition on \emph{bucketings} of their predictions.

\begin{definition}[Bucketing of the Prediction Space] \label{def:bucketing}
For bucket coarseness parameter $n$, let $B_n(i)= \left[\frac{i-1}{n}, \frac{i}{n} \right)$ and $B_n(n) = \left[\frac{n-1}{n}, 1 \right]$ form a set $\cB_n$ of $n$ buckets of width $1/n$ that partition the unit interval.
\end{definition}

Next we define conversation calibration, which allows for calibration error as measured using distance to calibration. 

\begin{definition}[Conversation-Calibrated Predictions]
\label{def:conversation-calibration}
Fix an error function $f:\{1, \ldots, T\} \rightarrow \mathbb{R}$ and bucketing function $g: \{1, \ldots, T\} \rightarrow (0,1]$. Given a prediction transcript $\pi^{1:T}$ resulting from an interaction in the canonical setting (Definition \ref{def:setting-canonical}), a human is $(f, g)$-conversation-calibrated if for all even rounds $k$ and buckets $i \in \{1, \ldots, 1/g(T)\}$:
\begin{align*}
    \CalDist(\phk{T_m(k, i)}{k},y^{T_m(k, i)}) \leq f(|T_m(k, i)|),
\end{align*} 
where $T_m(k, i) = \left\{t \in \Tk{k} ~|~ \pmk{t}{k-1} \in B_i(1/g(T))\right\}$ is the subsequence of days where the predictions of the model at the previous round fall in bucket $i$ and the conversation reaches round $k$.

Symmetrically, a model is $(f, g)$-conversation-calibrated if for all odd rounds $k$ and buckets $i \in \{1, \ldots, 1/g(T)\}$:
\begin{align*}
    \CalDist(\pmk{T_h(k, i)}{k},y^{T_h(k, i)}) \leq f(|T_h(k, i)|),
\end{align*}
where $T_h(k, i) = \{t \in \Tk{k} ~|~ \phk{t}{k-1} \in B_i(1/g(T))\}$, that is, the subsequence of days where the predictions of the human in the previous round fall in bucket $i$ and the conversation reaches round $k$.
\end{definition}

When convenient we will assume that  $f(\cdot)$ is concave. This captures the case where $f(T) = T^{\alpha}$ for any $\alpha \in [0,1]$, which is the form that all calibration bounds we are aware of take.

\begin{assumption}
    $f(\cdot)$ is a concave function.
\end{assumption}

We also define a $d$-dimensional notion of conversation calibration. A naive (and intractable) generalization of conversation calibration would require that an agent's $d$-dimensional forecasts be (fully) calibrated conditional on the value of the $d$-dimensional forecasts made at the previous round by the other agent. But this would require making predictions that are unbiased subject to an exponential (in $d$) number of conditioning events. Instead our generalization requires that the forecasts made by each party satisfy a \emph{marginal} conversation calibration condition in each coordinate of their prediction. That is, each coordinate $i$ of an agent's prediction should be calibrated marginally, conditional on the value of the other agent's previous prediction \emph{in coordinate $i$}. This increases the number of conditioning events compared to the $1$ dimensional case only by a factor of $d$, and hence will be tractably obtainable.

\begin{definition}[Conversation-Calibrated Vector Predictions]
\label{def:conversation-calibration-dimensional}
Fix an error function $f:\{1, \ldots, T\} \rightarrow \mathbb{R}$ and bucketing function $g: \{1, \ldots, T\} \rightarrow (0,1]$. Given a prediction transcript $\pi^{1:T}$ resulting from an interaction in the full-feedback, $d$-dimensional setting (Definition \ref{def:setting-dimensions}), a human is $(f, g)$-conversation-calibrated if for all even rounds $k$, indices $j \in [d]$, and buckets $i \in \{1, \ldots, 1/g(T)\}$:
\begin{align*}
    \CalDist(\phk{T_m(k, i,j)}{k}[j] ,y^{T_m(k, i,j)}[j] ) \leq f(|T_m(k, i,j)|),
\end{align*} 
where $T_m(k, i,j) = \left\{t \in \Tk{k} ~|~ \pmk{t}{k-1}[j] \in B_i(1/g(T))\right\}$ is the subsequence of days where the $j$'th coordinate of the predictions of the model at the previous round fall in bucket $i$ and the conversation reaches round $k$.

Symmetrically, a model is $(f, g)$-conversation-calibrated if for all odd rounds $k$, indices $j \in [d]$, and buckets $i \in \{1, \ldots, 1/g(T)\}$:
\begin{align*}
    \CalDist(\pmk{T_h(k, i,j)}{k}[j] ,y^{T_h(k, i,j)}[j] ) \leq f(|T_h(k, i,j)|),
\end{align*}
where $T_h(k, i,j) = \{t \in \Tk{k} ~|~ \phk{t}{k-1}[j] \in B_i(1/g(T))\}$.
\end{definition}

\subsubsection{Decision Conversation Calibration}
Next we turn to the action feedback setting (Definition \ref{def:setting-action}). The outcome space $\cY$ is now vector valued, and instead of communicating vector valued predictions as messages, the agents communicate downstream \emph{actions}.  We define decision-conversation-calibration, which asks for ``decision calibration'' \cite{zhao2021calibrating,noarov2023high,gopalan2023loss} conditional on the previous message sent by the other agent. In other words, the predictions that each agent makes should be unbiased conditional on both 1) the best response action implied by the predictions themselves, and 2) the best response action communicated at the previous round. Here we use an expected-calibration-error style definition, since this is what we can achieve algorithmically using the algorithm of \cite{noarov2023high}. 

\begin{definition}[Decision-Conversation-Calibrated (DC-Calibrated) Predictions]
\label{def:conversation-decision}
Given a prediction transcript $\pi^{1:T}$ resulting from an interaction in the action feedback setting (Definition \ref{def:setting-action}), a human is $f(\cdot)$-decision-conversation-calibrated (or $f(\cdot)$-DC-calibrated) if for all even rounds $k$,
coordinates $i \in [d]$, and pairs of actions $a,a^\prime \in \cA$: 
\begin{align*}
    \left| \sum_{t=1}^T \mathbbm{1}[t \in T_h(k,a,a^\prime)]  (\yhk{t}{k}[i] - y^t[i]) \right| \leq  f(|T_h(k,a,a^\prime)|),
\end{align*}
where $T_h(k,a,a^\prime)= \{t \in \Tk{k} ~|~ \pmk{t}{k-1} = a \text{ and } \phk{t}{k} = a^\prime \}$ is the subsequence of days in which the model's recommendation on round  $k-1$ is $a$ and the human's recommendation on round $k$ is $a'$.

Symmetrically, a model is $f(\cdot)$-DC-calibrated if for all odd rounds $k$, coordinates $i \in [d]$, and pairs of actions $a,a^\prime \in \cA$: 
\begin{align*}
    \left| \sum_{t= 1}^T \mathbbm{1}[t \in T_m(k,a,a^\prime)]  (\ymk{t}{k}[i] - y^t[i]) \right| \leq  f(|T_m(k,a,a^\prime)|),
\end{align*}
where $T_m(k,a,a^\prime)= \{t \in \Tk{k} ~|~ \phk{t}{k-1} = a \text{ and } \pmk{t}{k} = a^\prime \}$ is the subsequence of days in which the human's recommendation on round  $k-1$ is $a$ and the model's recommendation on round $k$ is $a^\prime$.
\end{definition}

\section{Agreement in the Canonical Setting}
\label{sec:full-info}

In this section we study the simple ``canonical'' setting (Definition \ref{def:setting-canonical}) in which $\cY = \Omega_m = \Omega_h = [0,1]$, which most closely maps onto the relevant prior work stemming from Aumann's agreement theorem \cite{aumann1976,geanakoplos1982we,aaronson2004complexity,frongillo2023agreement}. We show that when interacting in the Agreement protocol (Protocol \ref{alg:general-agreement}), if both agents satisfy appropriately instantiated \emph{conversation calibration} conditions (Definition \ref{def:conversation-calibration}), then once the total number of days $T$ is sufficiently large, on a $1-\delta$ fraction of days, they $\epsilon$-agree after at most $K \leq 2/(\epsilon^2\delta)$ rounds of conversation without reducing accuracy. We  give an efficient reduction through which any static model can be converted into an algorithm satisfying these conversation calibration conditions after at most  $T \leq O\left(\frac{1}{\epsilon^{6}\delta^{3}}\right)$ days. We remark that this bound is possible because we are able to carry out our analysis using \emph{distance to calibration} bounds, which admit algorithms that obtain $O(\sqrt{T})$ rates in online adversarial settings \cite{qiao2024distance,arunachaleswaran2024} --- we would obtain worse rates if we used the same reduction using algorithms bounding expected calibration error \cite{qiao2021stronger}. 

As predictions are the same as messages in the canonical setting ($\pmk{k}{t} = \hat{y}_{m}^{k,t}$ and $\phk{k}{t} = \hat{y}_{h}^{k,t}$), in this section we will refer to both these terms as $\pmk{k}{t}$ (and $\phk{k}{t}$) for simplicity. The following theorem formalizes the statement that conversation calibration (at sufficiently diminishing rates) guarantees fast agreement on most rounds, and that the resulting conversations improve accuracy.

\begin{theorem} 
\label{thm:canonical} 
If the Human is $(f_h, g_h)$-conversation-calibrated and the Model is $(f_{m}, g_m)$-conversation-calibrated, then for any $\epsilon,\delta \in [0,1]$, on a $1-\delta$ fraction of days, they reach $\epsilon$-agreement after at most $K$ rounds of conversation for 
$$K \leq \frac{1}{\epsilon^{2}\delta  - \beta(T)}$$
 where $\beta(T)= 3\left(g_m(T) + g_{h}(T) + \frac{f_{m}(g_{m}(T) \cdot T)}{g_{m}(T) \cdot T} + \frac{f_{h}(g_{h}(T) \cdot T)}{g_{h}(T) \cdot T}\right)$, a term that will tend to $0$ for appropriately instantiated functions $g$ and $f$.

Furthermore, for any round $k$ such that $|\Tk{k}| \geq \delta T$, we have that
$$\frac{\SQErr(\barphk{1:T}{k},y^{1:T})}{T} \leq \min \left( \frac{\SQErr(\barpmk{1:T}{1},y^{1:T})}{T}, \frac{\SQErr(\barphk{1:T}{2}, y^{1:T})}{T} \right) - k(\epsilon^{2}\delta - \beta(T)).$$
In other words, each round of conversation is error improving compared to the initial predictions of the human (or the model), with the error improving at a rate that is linear in the number of rounds of conversation.
\end{theorem}

A corollary of this theorem is that after $T$ is taken to be sufficiently large, agreement occurs rapidly on almost every day, and each further round of conversation leads to an $\epsilon^{2}\delta$ decrease in squared error. 

\begin{corollary}
\label{cor:canonical}
When $\beta(T) \leq \frac{\delta\epsilon^{2}}{2}$ , on a $1 - \delta$ fraction of days, the number of rounds until agreement is at most $K \leq \frac{2}{\delta \epsilon^{2}}$.
\end{corollary}

 Finally, in Theorem \ref{thm:reduction} we give a reduction that allows us to convert an arbitrary model into an algorithm that satisfies $(\sqrt{T},T^{\frac{-1}{3}})$-conversation calibration, for which it suffices to take $T \geq O(\frac{1}{\epsilon^{6}\delta^{3}})$ to satisfy the conditions of Corollary \ref{cor:canonical}. 

We now turn to proving Theorem \ref{thm:canonical}. First we give some intuition for the theorem. Our analysis will focus on the sequence of predictions made at each \emph{round} $k$ of conversation, over all days for which the conversation reaches that round. Intuitively, there are two cases: 
 \begin{enumerate}
     \item In the first case, on most days, the prediction at round $k$ is within $\epsilon$ of the prediction made at round $k-1$. In this case, most conversations that make it to round $k$ end in agreement at round $k$. 
     \item In the second case, most predictions at round $k$ differ by more than $\epsilon$ from the predictions at round $k-1$. But the sequence of predictions made at round $k$ satisfies conversation calibration. This means that when we condition on the subsequence at which (for example), the prediction at round $k-1$ was $v'$ and the prediction at round $k$ was $v$ for some $|v-v'| \geq \epsilon$, on this subsequence, the label mean was actually $v$. As a result, the sequence of predictions at round $k$ must be substantially more accurate than the predictions at round $k-1$. 
 \end{enumerate}
 But neither case can occur very often: every time case (1) occurs, the fraction of conversations that makes it beyond round $k$ is reduced by a constant factor, which can occur at most $\log(1/\delta)$ many times before only a $\delta$ fraction of conversations remain. And each time case (2) occurs, the average squared error of the predictions at round $k$ (which is reached on at least a $\delta$ fraction of days) decreases by at least $\approx \epsilon^2$ --- but as the labels and predictions are both bounded in $[0,1]$, this cannot occur more than $1/(\epsilon^2 \delta)$ many times. 

 In fact, we smoothly handle both kinds of events without explicitly breaking the analysis down into two cases, and as a result do not have to pay for the $\log(1/\delta)$ term. The following lemma is the work-horse of our analysis. It states that at any round, if the human is (perfectly) conversation calibrated given some bucketing of the model's predictions, then the squared error of the human's predictions is lower than the squared error of the model's most recent predictions by an amount scaling with $\epsilon^2$ times the number of days that did not lead to agreement at that round --- minus an error term that depends on the coarseness of the bucketing function $g_h$ defining the human's conversation calibration guarantee. A symmetric guarantee holds for the model.

\begin{lemma}\label{lem:mh} 
   If the human is $(0, g_{h}(T))$-conversation-calibrated, then for any even $k$, 
   \begin{align*} 
    \SQE(\barphk{T}{k},y^{1:T}) \leq \SQE(\barpmk{T}{k-1},y^{1:T}) -
   (\epsilon - g_{h}(T))^{2}|\Tk{k+1}| + g_h(T)T
   \end{align*}
   And if the model is $(0, g_{h}(T))$-conversation-calibrated, for any odd $k$,
   \begin{align*} 
    \SQE(\barpmk{T}{k},y^{1:T}) \leq \SQE(\barphk{T}{k-1},y^{1:T}) -
   (\epsilon - g_{m}(T))^{2} |\Tk{k+1}| + g_m(T)T
   \end{align*}
\end{lemma}

\begin{proof}
Let $T_{k}^{i,p_{h}} = \{t: t \in \Tk{k} \text{ and } \phk{t}{k} = p_{h} \text{ and } \pmk{t}{k-1} \in B_{i}(\frac{1}{g(T)})\}$ be the subsequence of days such that the human predicts $p_{h}$ in round $k$ and the model predicts in bucket $B_{i}(\frac{1}{g(T)})$ in round $k-1$. Let $m_{k}^{i,p_h} = \frac{\sum_{t \in T_{k}^{i,p_{h}}}y^{t}}{|T_{k}^{i,p_{h}}|} $ be the true mean on this subsequence. The difference in squared errors over this subsequence can be written as: 

\begin{align*}
 & \sum_{t \in T_{k}^{i,p_{h}}}(\pmk{t}{k-1} - y^{t})^{2} - \sum_{t \in T_{k}^{i,p_{h}}}(\phk{t}{k} - y^{t})^{2} 
 \\ & = \left[\sum_{t \in T_{k}^{i,p_{h}}}(\pmk{t}{k-1} - y^{t})^{2} - \sum_{t \in T_{k}^{i,p_{h}}}(m_{k}^{i,p_h} - y^{t})^{2} \right] - \left[\sum_{t \in T_{k}^{i,p_{h}}}(\phk{t}{k} - y^{t})^{2}  - \sum_{t \in T_{k}^{i,p_{h}}}(m_{k}^{i,p_h} - y^{t})^{2} \right]  \tag{Adding and subtracting $\sum_{t \in T_{k}^{i,p_{h}}}(m_{k}^{i,p_h} - y^{t})^{2}$}
 \\ & \geq \left[\sum_{t \in T_{k}^{i,p_{h}}}(i \cdot g_{h}(T) - y^{t})^{2} -|T_{k}^{i,p_{h}}| \cdot g_h(T)  - \sum_{t \in T_{k}^{i,p_{h}}}(m_{k}^{i,p_h} - y^{t})^{2} \right] - \\ &  \left[\sum_{t \in T_{k}^{i,p_{h}}}(\phk{t}{k} - y^{t})^{2}  - \sum_{t \in T_{k}^{i,p_{h}}}(m_{k}^{i,p_h} - y^{t})^{2} \right] \tag{By Lemma~\ref{lem:v1}}
  \\ & =\left[\sum_{t \in T_{k}^{i,p_{h}}}(i \cdot g_{h}(T) - m_{k}^{i,p_{h}})^{2} - |T_{k}^{i,p_{h}}| \cdot g_h(T) \right] -  \left[\sum_{t \in T_{k}^{i,p_{h}}}(\phk{t}{k} - y^{t})^{2}  - \sum_{t \in T_{k}^{i,p_{h}}}(m_{k}^{i,p_h} - y^{t})^{2} \right]  \tag{By Lemma~\ref{lem:squares_diff}}
    \\ & = \left[\sum_{t \in T_{k}^{i,p_{h}}}(i \cdot g_{h}(T) - m_{k}^{i,p_{h}})^{2} - |T_{k}^{i,p_{h}}| \cdot g_h(T)   \right] - \left[\sum_{t \in T_{k}^{i,p_{h}}}(p_h - y^{t})^{2}  - \sum_{t \in T_{k}^{i,p_{h}}}(m_{k}^{i,p_h} - y^{t})^{2} \right] \tag{As by definition of $T_{k}^{i,p_{h}}$, $\phk{t}{k} = p_{h}$} 
      \\ & \geq \left[\sum_{t \in T_{k}^{i,p_{h}}}(i \cdot g_{h}(T) - m_{k}^{i,p_{h}})^{2} - |T_{k}^{i,p_{h}}| \cdot g_h(T) \right]-  \left[\sum_{t \in T_{k}^{i,p_{h}}}(p_h - m_{k}^{i,p_{h}})^{2} \right]  \tag{By Lemma~\ref{lem:squares_diff}}
           \\ & \geq - |T_{k}^{i,p_{h}}| \cdot g_h(T)  +  \sum_{t \in T_{k}^{i,p_{h}}}(i \cdot g_{h}(T) - p_{h})^{2} \tag{As the human is $(0, g_{h}(T))$-conversation calibrated, $p_{h} = m_{k}^{i,p_{h}}$}
    \end{align*}
Using this analysis, we can write the difference in squared errors over the entire sequence $\barphk{T}{k}$ and $\barpmk{T}{k-1}$ as follows, where the first term comes from summing the above expression over all $i, p_h$:
    \begin{align*}
        &\SQE(\barpmk{T}{k-1}, y^{1:T}) - \SQE(\barphk{T}{k}, y^{1:T}) \\
        & = \sum_{\forall i, p_h} \left( \sum_{t \in T_{k}^{i,p_{h}}}(\pmk{t}{k-1} - y^{t})^{2} - \sum_{t \in T_{k}^{i,p_{h}}}(\phk{t}{k} - y^{t})^{2} \right) - \sum_{t \not \in \Tk{k}} (\barpmk{t}{k-1} - y^t)^2 - (\barphk{t}{k} - y^t)^2  \\
        & = \sum_{\forall i, p_h} \left( \sum_{t \in T_{k}^{i,p_{h}}}(\pmk{t}{k-1} - y^{t})^{2} - \sum_{t \in T_{k}^{i,p_{h}}}(\phk{t}{k} - y^{t})^{2} \right) \tag{as $\barpmk{t}{k-1} = \barphk{t}{k}$ for all $t \not \in \Tk{k}$, by definition} \\
        & = \sum_{\forall i, p_{h}} \left(- |T_{k}^{i,p_{h}}| \cdot g_h(T)  +  \sum_{t \in T_{k}^{i,p_{h}}}(i \cdot g_{h}(T) - p_{h})^{2}  \right) \tag{by the analysis above} \\
        & \geq -g_{h}(T)T + \sum_{\forall i, p_{h}} \sum_{t \in T_{k}^{i,p_{h}}}(i \cdot g_{h}(T) - p_{h})^{2}  \tag{As $g_{h}(T)$ is independent of $i$ and $p_{h}$, and $\sum_{\forall i, p_{h}}\left|T_{k}^{i,p_{h}}\right| \leq T$} \\
        & \geq -g_{h}(T)T + \sum_{\forall i, p_{h}} \sum_{t \in T_{k}^{i,p_{h}}}\mathbbm{1}[|i \cdot g_{h}(T) - \phk{t}{k}| \geq \epsilon - g_{h}(T)](i \cdot g_{h}(T) - p_{h})^{2}   \\
        & \geq -g_{h}(T)T + (\epsilon - g_{h}(T))^{2} \sum_{\forall i, p_{h}}\sum_{t \in T_{k}^{i,p_{h}}} \mathbbm{1}[|i \cdot g_{h}(T) - \phk{t}{k}| \geq \epsilon - g_{h}(T)]  
    \end{align*}

Note that, for all days in the subsequence $T_{k}^{i,p_{h}}$, in round $k-1$ the model predicted in bucket $B_{i}(\frac{1}{g_{h}(T)}) = i \cdot g_{h}(T)$, and therefore in each of these days, by the definition of our bucketing, $\pmk{t}{k-1} \geq (i-1) \cdot g_{h}(T)$ and $\pmk{t}{k-1} \leq i \cdot g_{h}(T)$. So consider any round $t \in T_{k}^{i,p_{h}}$. If $|\phk{t}{k} - \pmk{t}{k-1}| \geq \epsilon$, then we have:

\begin{align*}
    |\phk{t}{k} - \pmk{t}{k-1}| &\le  |\phk{t}{k} - i\cdot g_{h}(T)| + |i\cdot g_{h}(T) - \pmk{t}{k-1}|\\
    &= |\phk{t}{k} - i\cdot g_{h}(T)| + i\cdot g_{h}(T) - \pmk{t}{k-1}\\
    &\le |\phk{t}{k} - i\cdot g_{h}(T)|  + i\cdot g_{h}(T) - (i-1) \cdot g_{h}(T)\\
    &= |\phk{t}{k} - i\cdot g_{h}(T)| + g_{h}(T),\\
    \implies |\phk{t}{k} - i\cdot g_{h}(T)| & \ge |\phk{t}{k} - \pmk{t}{k-1}| - g_{h}(T) \ge \epsilon - g_{h}(T).
\end{align*}

Thus, if $|\phk{t}{k} - \pmk{t}{k-1}| \geq \epsilon$, then $|i \cdot g_{h}(T) - \phk{t}{k}| \geq \epsilon - g_{h}(T)$, $\forall t \in T_{k}^{i,p_{h}}$. Therefore the set of days for which the former condition holds is a subset of the latter condition, and we can write
    
\begin{align*}
    & -g_{h}(T)T + (\epsilon - g_{h}(T))^{2} \sum_{\forall i, p_{h}} \mathbbm{1}[|i \cdot g_{h}(T) - p_{h}| \geq \epsilon - g_{h}(T)] \cdot \left|T_{k}^{i,p_{h}}\right|  \\
    & \geq -g_{h}(T)T + (\epsilon - g_{h}(T))^{2} \sum_{\forall i, p_{h}} \sum_{t \in T_{k}^{i,p_{h}}} \mathbbm{1}[|\phk{t}{k} - \pmk{t}{k-1}| \geq \epsilon] \\
    & = -g_{h}(T)T + (\epsilon - g_{h}(T))^{2} |\Tk{k+1}| \tag{As on every day where there is a next round, the human and the model disagreed by at least $\epsilon$}
\end{align*}

As the human and the model are perfectly symmetrical, we also obtain the symmetrical result for the model.
\end{proof}

Next, we extend Lemma \ref{lem:mh} to the case in which the conversation calibration error is not $0$, but rather controlled by some function $f_h(\cdot)$. The idea is straightforward. We know from Lemma \ref{lem:mh} that squared error would decrease significantly if the human's predictions were perfectly conversation calibrated. In fact, all we know is that the human's predictions are \emph{close} (in $\ell_1$ distance) to perfectly conversation calibrated predictions. But this is good enough, because squared error is Lipschitz, and so small changes in predictions result in small changes in squared error. As a result, approximate conversation calibration is also enough to let us bound the decrease in error across adjacent rounds: 

\begin{theorem} If the Human is $(f_{h}(\cdot), g_h(\cdot))$-conversation-calibrated, then after engaging in the iterated calibration protocol for $T$ days:
\begin{align*}
    \SQE(\barphk{T}{k}, y^{1:T}) \leq  \SQE(\barpmk{T}{k-1}, y^{1:T}) -(\epsilon - g_{h}(T))^{2} |\Tk{k+1}| + g_h(T) T + 3\frac{f_{h}(g_{h}(T) \cdot T)}{g_{h}(T)}
\end{align*}
And if the Model is $(f_{m}(\cdot), g_m(\cdot))$-conversation-calibrated, then after engaging in the iterated calibration protocol for $T$ days: 
\begin{align*}
    \SQE(\barpmk{T}{k}, y^{1:T}) \leq  \SQE(\barphk{T}{k-1}, y^{1:T}) -(\epsilon - g_{m}(T))^{2}|\Tk{k+1}| + g_m(T) T + 3\frac{f_{m}(g_{m}(T) \cdot T)}{g_{m}(T)}
\end{align*}

\label{thm:cases}
\end{theorem}
\begin{proof}

Let $T_{m}({k,i}) = \{t: \pmk{t}{k-1} \in B_i\left(\frac{1}{g_{h}(T)}\right)\}$ be the subsequence of days in which the models predicts in bucket $B_{i}(\frac{1}{g_{h}(T)})$ at round $k-1$. 

Note that the human has distance to calibration of $f_{h}(|T_m(k, i)|)$ on every such subsequence defined this way. Therefore, for predictions $p_{h}^{1:T,k}$ from the human at round $k$: 

\begin{align*} 
\CalDist(p_{h}^{\Tk{k},k}, y^{1:T}) & =
 \min_{q^{1:T} \in C(y^{1:T})}\|p_{h}^{\Tk{k},k} - q^{1:T}\|_{1} \\
 \\ & \leq \sum_{i =1}^{\frac{1}{g_{h}(T)}}\min_{q^{1:|T_m(k, i)|} \in C^{ T_m(k, i) }(y^{1:T})}\|p^{1:T} - q_{v}^{1:T}\|_{1}\\
& \leq \sum_{i=1}^{\frac{1}{g_{h}(T)}} f_{h}(|T_m(k, i)|) \tag{By the calibration distance of the Human}\\ 
& \leq \frac{f_{h}(g_{h}(T) \cdot |\Tk{k}|)}{g_{h}(T)} \tag{By the assumption that $f_{h}$ is concave} \\
& \leq \frac{f_{h}(g_{h}(T) \cdot T)}{g_{h}(T)} 
\end{align*}

Let $q^{\Tk{k}}$ be the set of perfectly calibrated predictions that are $f_{h}(|T_m(k, i)|)$-close to $p_{h}^{1:T,k}$. Then, we have that 

\begin{align*}
& \SQErr(\bar{p}_{h}^{T, k}, y^{1:T}) - \SQErr(\bar{p}_{m}^{T, k-1}, y^{1:T}) =
\SQErr(p_{h}^{\Tk{k}, k}, y^{\Tk{k}}) - \SQErr(p_{m}^{\Tk{k},k-1}, y^{\Tk{k}}) \tag{By the fact that $\bar{p}_{h}$ and $\bar{p}_{m}$ are equal on all inactive days} \\ & \leq \SQErr(q^{\Tk{k}}, y^{\Tk{k}}) - \SQErr(p_{h}^{\Tk{k}, k-1}, y^{\Tk{k}}) + 3\frac{f_{h}(g_{h}(T) \cdot T)}{g_{h}(T)} \tag{By Lemma~\ref{lem:bound_error_diff}} \\
&  \leq - (\epsilon - g_{h}(T))^{2}|\Tk{k+1}| + g_h(T) T + 3\frac{f_{h}(g_{h}(T) \cdot T)}{g_{h}(T)} \tag{By Lemma~\ref{lem:mh}}.
\end{align*}

As the Human and the Model are symmetric, we also obtain the symmetric result for the Model.
\end{proof}

\begin{proof}[Proof of Theorem \ref{thm:canonical}]
By composing the two results in Theorem~\ref{thm:cases}, we see that 
\begin{align*}
    &\SQErr(\barphk{T}{k-2},y^{1:T}) - \SQErr(\barphk{T}{k}, y^{1:T})  \\
    &\geq(\epsilon - g_{h}(T))^{2} |\Tk{k+1}| + (\epsilon - g_{m}(T))^{2}|\Tk{k}|  - g_m(T) T - 3\frac{f_{m}(g_{m}(T) \cdot T)}{g_{m}(T)} - g_h(T) T - 3\frac{f_{h}(g_{h}(T) \cdot T)}{g_{h}(T)}\\
    &\geq \left((\epsilon - g_{m}(T))^{2} + (\epsilon - g_{h}(T))^{2}\right) |\Tk{k+1}| - (g_m(T) + g_{h}(T))T - 3\left(\frac{f_{m}(g_{m}(T) \cdot T)}{g_{m}(T)} + \frac{f_{h}(g_{h}(T) \cdot T)}{g_{h}(T)}\right).
\end{align*}

Now, consider any round $r$ such that $|\Tk{r}| \geq \delta T$. 
Then, we can apply the above expression recursively, we can bound the squared error of the human at round $r$ as: 

\begin{align*}
   &\SQErr(\barphk{T}{r}, y^{1:T}) \\
   &\leq \SQErr(\barphk{T}{2},y^{1:T}) - ((\epsilon - g_{m}(T))^{2} + (\epsilon - g_{h}(T))^{2})\left(\sum_{k=1, k \text{ even}}^{r} |\Tk{k}| \right) + (g_m(T) + g_{h}(T)) \left(\sum_{k=1,k \text{ even}}^{r} |\Tk{k}| \right) \\
   &\qquad + 3\left(\frac{f_{m}(g_{m}(T) \cdot T)}{g_{m}(T)} + \frac{f_{h}(g_{h}(T) \cdot T)}{g_{h}(T)}\right)\left(\sum_{k=1,k \text{ even}}^{r}1\right) \\
  & \leq \SQErr(\barphk{T}{2},y^{1:T}) - ((\epsilon - g_{m}(T))^{2} + (\epsilon - g_{h}(T))^{2})\left(\sum_{k=1,k \text{ even}}^{r} |\Tk{k}| \right) + (g_m(T) + g_{h}(T)) (r)T\\
  & \qquad + 3\left(\frac{f_{m}(g_{m}(T) \cdot T)}{g_{m}(T)} + \frac{f_{h}(g_{h}(T) \cdot T)}{g_{h}(T)}\right)(r) \tag{As $|\Tk{k}| \leq T$} \\
& \leq \SQErr(\barphk{T}{2},y^{1:T}) - ((\epsilon - g_{m}(T))^{2} + (\epsilon - g_{h}(T))^{2})(r)\delta T  + 2(g_m(T) + g_{h}(T)) (r)T\\
&\qquad + 3\left(\frac{f_{m}(g_{m}(T) \cdot T)}{g_{m}(T)} + \frac{f_{h}(g_{h}(T) \cdot T)}{g_{h}(T)}\right)(r) \tag{As for all $\Tk{k}$ such that $k\leq r$, $ |\Tk{k}| \geq \delta T$} \\
& \leq \SQErr(\barphk{T}{2},y^{1:T}) - 2r\epsilon^{2}\delta T  + 3r(g_m(T) + g_{h}(T))T + 3r\left(\frac{f_{m}(g_{m}(T) \cdot T)}{g_{m}(T)} + \frac{f_{h}(g_{h}(T) \cdot T)}{g_{h}(T)}\right) 
\\ & = \SQErr(\barphk{T}{2},y^{1:T}) - r\left(2\epsilon^{2}\delta T  + T\beta(T)\right)
\end{align*}
Finally we can compose this expression with one more instantiation of Theorem \ref{thm:cases}:
\begin{align*}
    \SQE(\barphk{T}{2}, y^{1:T}) &\leq  \SQE(\barpmk{T}{1}, y^{1:T}) -(\epsilon - g_{h}(T))^{2} |\Tk{2}| + g_h(T) T + 3\frac{f_{h}(g_{h}(T) \cdot T)}{g_{h}(T)} \\
    & \leq \SQE(\barpmk{T}{1}, y^{1:T}) -(\epsilon - g_{h}(T))^{2} \delta T + g_h(T) T + 3\frac{f_{h}(g_{h}(T) \cdot T)}{g_{h}(T)} \tag{as $|\Tk{2}| \geq \delta T$} \\
     & \leq \SQE(\barpmk{T}{1}, y^{1:T}) -\epsilon^{2} \delta T + 2g_h(T) T + 3\frac{f_{h}(g_{h}(T) \cdot T)}{g_{h}(T)}  \\
      & \leq \SQE(\barpmk{T}{1}, y^{1:T}) -\epsilon^{2} \delta T + T\beta(T)
\end{align*}

and get a final expression of:
\begin{align*}
    \SQErr(\barphk{T}{r}, y^{1:T}) & \leq \SQE(\barpmk{T}{1}, y^{1:T}) - \epsilon^{2} \delta T + T\beta(T)  - r\left(2\epsilon^{2}\delta T  + T\beta(T)\right)  \\
    & \leq \SQE(\barpmk{T}{1}, y^{1:T}) - (r+1)\left(\epsilon^{2}\delta T + T\beta(T)\right) 
\end{align*}

This completes the second part of the Theorem.

By definition, the squared error is non-negative. Therefore, we have that 

\begin{align*}
   &\SQErr(\barphk{T}{r},y^{1:T}) \leq \SQErr(\barphk{T}{2},y^{1:T}) - r\left(\epsilon^{2}\delta T  + T\beta(T)\right) \\
   \implies &0 \leq \SQErr(\barphk{T}{2},y^{1:T}) - r\left(\epsilon^{2}\delta T  + T\beta(T)\right) \\
  \implies &0 \leq T - r\left(\epsilon^{2}\delta T  + T\beta(T)\right) \tag{As the maximum squared error is $T$} \\
  \implies &r \leq \frac{1}{\epsilon^{2}\delta  + \beta(T)}
\end{align*}

This completes the first part of the Theorem.
\end{proof}

We now interrogate what conversation-calibration rates are sufficient to get fast convergence to the agreement bounds we quoted in Corollary \ref{cor:canonical}.

\begin{theorem} 
Fix any $0 < \alpha < 1$. There exists a constant $\gamma$ such that if the Human is $(f_h(\cdot),g_h(\cdot))$-conversation calibrated and the model is $(f_m(\cdot),g_m(\cdot))$-conversation calibrated such that:
$$f_h(\tau), f_m(\tau) \in O(\tau^\alpha) \text{ and }  g_m(\tau), g_h(\tau) \in O(\tau^{\gamma})\ \   $$
then for every $T \geq \Omega\left((\frac{1}{\delta\epsilon^{2}})^{\frac{2 - \alpha}{1 - \alpha}}\right)$, if the agreement protocol is run for at least $T$ days, then on a $1-\delta$ fraction of days, the two parties reach $\epsilon$-agreement after at most  $O\left(\frac{1}{\delta \epsilon^{2}}\right)$ rounds.
\end{theorem}
\begin{proof}

By Corollary~\ref{cor:canonical}, if $\beta(T) \leq \frac{\delta\epsilon^{2}}{2}$, the number of rounds until agreement is at most 
$O(\frac{1}{\delta \epsilon^{2}})$. Thus we will find a sufficiently large value of $T$ to ensure this. 
    
We have that $$\beta(T) = 3(2T^{\gamma} + 2T^{\alpha (\gamma +1) - \gamma - 1})$$

 This is minimized by solving for 
\begin{align*}
 \gamma = \alpha(\gamma+1) - \gamma - 1  \implies 2\gamma - \alpha\gamma = \alpha - 1  \implies \gamma = \frac{\alpha - 1}{2 - \alpha}
\end{align*}

Thus we get $\beta  = 12 T^{\frac{\alpha - 1}{2 - \alpha}}$. Now, in order to ensure that $\beta(T) \leq \frac{\delta \epsilon^{2}}{2}$, we need
    \begin{align*}
       12T^{\frac{\alpha - 1}{2 - \alpha}} \leq \frac{\delta \epsilon^{2}}{2} \implies T^{\frac{1 - \alpha}{2 - \alpha}} \geq \frac{24}{\delta \epsilon^{2}}
        \implies  T \geq \left(\frac{24}{\delta \epsilon^{2}}\right)^{\frac{2 - \alpha}{1 - \alpha}} = O\left(\frac{1}{\delta\epsilon^{2}}\right)^{\frac{2 - \alpha}{1 - \alpha}}.
    \end{align*}
\end{proof}

Finally, we turn to the algorithmic problem. There are existing simple, efficient algorithms that can make sequential predictions in adversarial environments that guarantee diminishing distance to calibration at favorable rates \cite{arunachaleswaran2024}. What we give here is an efficient reduction that takes as input an arbitrary initial model, and by reduction to a sequential prediction algorithm that achieves distance to calibration at some rate, outputs an algorithm for the model that can interact in Protocol \ref{alg:general-agreement} and against any sequence of predictions for the human, guarantee conversation calibration at the same rate. The reduction is straightforward: We use the initial model to make the round $1$ predictions, and then initialize a collection of distance-to-calibration algorithms, for each round $k$ and for each possible bucketing of the other agent's predictions. Then, at each round $k$, we predict according to the instance of the distance-to-calibration algorithm corresponding to that round and the bucketing of the other agent's most recent prediction.

\begin{algorithm}[ht]
\begin{algorithmic} 
 \STATE{\bf Input} {Sequence of outcomes $y^{1:T} \in \{0,1\}^T$}
  \STATE{\bf Output} {Sequence of predictions $p^{1:T} \in \{0, \frac{1}{m},..., 1\}^T$ for $m = 1/\sqrt{T}$}
 \FOR{$t = 1 \ldots, T$}{   
    \STATE Given look-ahead predictions $\tilde p^{1:t-1}$, define the look-ahead bias conditional on a prediction $p$ as:
    $$\alpha_{\tilde p^{1:t-1}}(p) := \sum_{s=1}^{t-1} \mathbb{I}[\tilde p^s = p] (\tilde p^s - y^s)$$
    \STATE Choose two adjacent points $p_i = \frac{i}{m}, p_{i+1} = \frac{i+1}{m}$ satisfying: $$\alpha_{\Tilde{p}^{1:t-1}}(p_i) \leq 0 \text{ and } \alpha_{\Tilde{p}^{1:t-1}}(p_{i+1}) \geq 0$$
    \STATE Arbitrarily predict ${p}^t = p_i$ or ${p}^t = p_{i+1}$\;
    \hspace{.2em} Upon observing the (adversarially chosen) outcome $y^t$, set look-ahead prediction $$\Tilde{p}^t = \argmin_{p \in \{p_i, p_{i+1}\}} |p - y^t|$$
 }
 \ENDFOR
 \end{algorithmic}
 \caption{Almost-One-Step-Ahead (\textsc{AOSA}) \cite{arunachaleswaran2024}}\label{alg:1}
\end{algorithm}

First we quote the distance to calibration guarantee of Almost-One-Step-Ahead (Algorithm \ref{alg:1})
\begin{theorem}[\cite{arunachaleswaran2024}] 
 Algorithm~\ref{alg:1} (Almost-One-Step-Ahead) guarantees that against any sequence of outcomes, $\CalDist(p^{1:T}, y^{1:T})\leq 2\sqrt{T} + 1$.  \label{thm:aost}
\end{theorem}

As stated, Algorithm \ref{alg:1} is defined as a function of the length $T$ of the sequence on which it will be evaluated. In our reduction, we will want guarantees that hold over many different sequences whose lengths we do not know ahead of time. However, it is not hard to convert Algorithm \ref{alg:1} into an algorithm that has similar bounds and does not require knowing $T$ ahead of time, using a doubling trick. We give such an algorithm in Algorithm \ref{alg:2} in Appendix \ref{app:full-info}.

\begin{theorem}
\label{thm:unknownT}
     Algorithm~\ref{alg:2} (Almost-One-Step-Ahead with unknown $T$) guarantees that against any sequence of outcomes, $\CalDist(p^{1:T}, y^{1:T})\leq O(\sqrt{T})$. 
\end{theorem}
The proof is in Appendix \ref{app:full-info}.

\begin{algorithm}[ht]
\begin{algorithmic}
    \STATE{\bf Input} {Baseline model algorithm $M_{0}$, D2C algorithm $D$, Discretization $g_{m}(T)$}
    \STATE We denote $D_{k,i}$ as an instantiation of $D$ which is given as input only the subsequence of days where $\phk{t}{k} \in [(i - 1) \cdot g_{m}(T), i \cdot g_{m}(T)]$, and denote $D_{k,i,t}$ be the prediction of $D_{k,i}$ at round $t$.
    \FOR{$t = 1, \ldots, T$}
        \STATE Receive $x^t_m$
        \STATE Send prediction $\pmk{t}{1} = M_{0}(x^{t}_{m})$ to human
        \FOR{$k = 3, 5, \ldots $}
        \STATE Initialize empty set $S$
         \STATE Observe human prediction $\phk{t}{k-1}$ 
         \IF{$ | \phk{t}{k-1} - \pmk{t}{k-2} | < \epsilon $}
        \STATE Predict $\phk{t}{k-1}$ and break out of loop
            \ENDIF
        \STATE Let $i$ be such that $\phk{t}{k-1} \in [(i - 1) \cdot g_{m}(T), i \cdot g_{m}(T)]$
        \IF{$D_{k-1,i}$ uninitialized}
        \STATE Initialize $D_{k-1,i}$
        \ENDIF
        \STATE Send prediction $\pmk{t}{k} = D_{k-1,i,t}$ to human
        \STATE $S \gets S \cup (k-1,i)$
         \IF{ $ | \phk{t}{k-1} - \pmk{t}{k} | < \epsilon $} 
        \STATE Predict $\phk{t}{k-1}$ and break out of loop
            \ENDIF
        \ENDFOR
    \STATE Observe $y^{t}$
    \FOR{$(k,i) \in S$ }
    \STATE Update $D_{k,i}$ with $(D_{k,i,t}, y^{t})$
    \ENDFOR
    \ENDFOR
\end{algorithmic}
\caption{\textsc{Converse($M_{0}, D, g_{m}(T)$)}: A reduction from an online decision-making algorithm to an algorithm with low conversation-calibration error} 
\label{alg:converse-reduction}
\end{algorithm}

Finally, we can give our reduction (Algorithm \ref{alg:converse-reduction}) that takes as input an initial model $M_0:\cX_m\rightarrow \cY$, a sequential prediction algorithm $D$ with a concave bound $f_m(\cdot)$ on its distance to calibration, and a bucketing function $g_m(\cdot)$. We show that for any discretization function $g_m$, Algorithm \ref{alg:converse-reduction} guarantees $(f_m(\cdot),g_m(\cdot))$ conversation calibration against any sequence of outcomes and predictions of the human. 
We will refer to the prediction made by Algorithm \ref{alg:converse-reduction} in round $1$ (for any day $t$) as $\textsc{Converse}_1$. 

\begin{theorem} \label{thm:reduction}
If $D$ has worst-case $\CalDist$ of $f_{m}(\cdot)$, then for any bucketing function $g_m(\cdot)$,
    \textsc{Converse($M_{0}, D, g_{m}(\cdot)$)} is $(f_{m}(\cdot),g_{m}(\cdot))$-conversation-calibrated in the worst case over label outcomes and conversations. Moreover the first round prediction of $\textsc{Converse}$ are the same as the prediction of the base model $M_0$ for all $t$: $\textsc{Converse}_{1}(x^t_m) = M_{0}(x^t_m)$, for all $t$. 
\end{theorem}
\begin{proof}
We first must show that \textsc{Converse($M_0, D, g_m(\cdot)$)} is $f_m(\cdot), g_m(\cdot)$-conversation-calibrated. Observe that Algorithm \ref{alg:converse-reduction} instantiates a new copy of the algorithm $D$ for each round $k$, bucket $i$ pair corresponding to each of the sets $T_m(k, i)$.
This immediately implies that the sequence of predictions made by each instance $D_{k, i}$ satisfies 
\[ \CalDist(p_h^{T_m(k, i), k}, y^{T_m(k, i)}) \leq f(|T_m(k, i)|), \]
by assumption that $D$ has worst-case distance to calibration of $f_m(\cdot)$. Since this holds for the sequence of predictions for all round $k$, bucket $i$ pairs, by its corresponding instance of $D_{k, i}$, we have that \textsc{Converse($M_{0}, D, g_{m}(\cdot)$)} is $(f_{m}(\cdot),g_{m}(\cdot))$-conversation-calibrated. That $\textsc{Converse}_{1}(x_m^{1:T}) = M_{0}(x_m^{1:T})$, for all $t$ follows by construction.
\end{proof}

\begin{corollary}
    \textsc{Converse($M_{0}, \textsc{AOSA}, T^{\frac{-1}{3}}$)} is $(\sqrt{T},T^{\frac{-1}{3}})-$conversation-calibrated, and $\textsc{Converse}_{1}(x^t_m) = M_{0}(x^t_m)$, for all $t$. \label{corr:converse-bound}
\end{corollary}

Finally, we end with a corollary putting all of our results together. We can take any initial model $M_0$ and efficiently convert it into a protocol that can engage in conversations with a human. If the human satisfies conversation calibration (a significantly weaker assumption than Bayesian rationality), then not only will the conversations halt quickly, but they will result in outcomes that are only more accurate than either the human or the model's initial judgments. This holds despite the fact that we have no assumptions on the form of the initial model $M_0$, so it can be the result of an arbitrarily sophisticated machine learning process. We state the corollary as if both parties are implemented using \textsc{Converse} to be concrete about computational tractability and so that we can be specific about rates, but this is not necessary --- the only important thing is that both parties are conversation calibrated.

\begin{corollary}
    If the human runs \textsc{Converse($M^{h}_{0}, \textsc{AOSA}, g_h(\cdot)$)} and the model runs \\ \textsc{Converse($M^{m}_{0}, \textsc{AOSA}, g_m(\cdot)$)} for $g_m(T) = g_h(T) = T^{-\frac{1}{3}}$, then: 
    \begin{itemize}
        \item For any $\epsilon,\delta \in [0,1]$, on a $1-\delta$ fraction of days, they reach $\epsilon$-agreement after at most $K$ rounds of conversation where: $K \leq \frac{1}{\epsilon^{2}\delta  - 6T^{-\frac{1}{3}}}$
        \item For the subsequence of days that make it to round $k$ s.t. $|\Tk{k}| \geq \delta T$ with associated human prediction subsequences $p_h^{\Tk{k}, k}$ and outcome subsequences $y^{\Tk{k}}$, we have that for $M_0 \in \{M_0^m,M_0^h\}$: $\frac{\SQErr(p_h^{\Tk{k}, k},y^{\Tk{k}})}{T} \leq \frac{\SQErr(M_{0},y^{\Tk{k}})}{T} - i(\epsilon^{2}\delta - \frac{12}{T^{\frac{1}{3}}})$.
    \end{itemize}
\end{corollary}

\section{Agreement in $d$ Dimensions}
\label{sec:d-dimensions}
We now extend our results from $1$ dimensional label spaces to $d$ dimensional label spaces: $\cY = [0,1]^d$ (the setting given in Definition \ref{def:setting-dimensions}). As in the previous section, here predictions are the same as messages ($\pmk{k}{t} = \hat{y}_{m}^{k,t}$ and $\phk{k}{t} = \hat{y}_{h}^{k,t}$), and therefore in this section we will refer to both these terms as $\pmk{k}{t}$ (and $\phk{k}{t}$) for simplicity.

At a high level, our argument will be similar. We will measure the error of our predictions using the sum of the squared error in each of the coordinates of our predictions, which will also serve as our potential function. We will continue to argue that at each round, either many conversations end, or else the squared error of the predictions must substantially improve, limiting the number of rounds of conversation that can occur. Of course, the maximum squared error is now $d$, rather than $1$, and so the number of rounds until agreement will be larger by a factor of $d$. There is another step in the argument within which one must be careful not to lose another factor of $d$. For tractability, we have only asked for conversation calibration conditions to hold \emph{marginally} on each coordinate of our predictions. So, we need to argue that error decreases coordinate-wise. But imagine a sequence of predictions at round $k$ on which we have not reached $\epsilon$-agreement. It might be that each prediction agrees with the previous round's predictions on all but a single coordinate --- and hence it might be that for any \emph{particular} coordinate, there is in fact $\epsilon$-disagreement with the prior round's prediction \emph{in that coordinate} on only a $1/d$ fraction of the rounds. We are able to avoid losing another factor of $d$ in our analysis by keeping more careful track of the error in each coordinate --- since if the disagreements are uniformly spread across all of the coordinates, although it is true that we do not improve the error by as much in each coordinate, we are able to improve the error in all coordinates simultaneously.

In our analysis, we will often need to focus on a single coordinate $j$ of the multi-dimensional prediction or label:  we will write $\phk{t}{k}[j]$ or $y[j]$ to denote the value at this coordinate. We measure accuracy using the following multi-dimensional extension of our squared error definition:

\begin{definition}[Multi-Dimension Squared Error]
    The squared error of a sequence of $d-$dimensional predictions $p$ with respect to the $d-$dimensional outcomes $y$ is:
    \[ \SQE(p^{1:T},y^{1:T}) = \sum_{j \in [d]}\SQE(p^{1:T}[j],y^{1:T}[j]). \]
\end{definition}

\begin{theorem} 
\label{thm:dimensions} 
If the Human is $(f_h, g_h)$-conversation-calibrated and the Model is $(f_{m}, g_m)$-conversation-calibrated in $d$ dimensions, then for any $\epsilon,\delta \in [0,1]$, on a $1-\delta$ fraction of days, they reach $\epsilon$-agreement after at most $K$ rounds of conversation for 
$$K \leq \frac{d}{\epsilon^{2}\delta  - \beta(T)}$$
 where $\beta(T)= 3d\left(g_m(T) + g_{h}(T) + \frac{f_{m}(g_{m}(T) \cdot T)}{g_{m}(T) \cdot T} + \frac{f_{h}(g_{h}(T) \cdot T)}{g_{h}(T) \cdot T}\right)$, a term that will tend to $0$ for appropriately instantiated functions $g$ and $f$.

Furthermore, for any round $k$ such that $|\Tk{k}| \geq \delta T$, we have that 
$$\frac{\SQErr(\barphk{T}{k},y^{1:T})}{T} \leq \min \left( \frac{\SQErr(\barpmk{T}{1},y^{1:T})}{T}, \frac{\SQErr(\barphk{T}{2}, y^{1:T})}{T} \right) - k(\epsilon^{2}\delta - \beta(T)).$$
In other words, each round of conversation is error improving compared to the initial predictions of the human (or the model), with the error improving at a rate that is linear in the number of rounds of conversation.
\end{theorem}

A corollary of this theorem is that after $T$ is taken to be sufficiently large, agreement occurs rapidly on almost every day, and for $(1-\delta)$ of the days, each further round of conversation leads to an $\epsilon^{2}\delta$ decrease in squared error.

\begin{corollary}
\label{cor:dimensions}
When $\beta(T) \leq \frac{\delta\epsilon^{2}}{2}$ , on a $1 - \delta$ fraction of days, the number of rounds until agreement is at most $K \leq \frac{2d}{\delta \epsilon^{2}}$.
\end{corollary}

Finally, in Corollary \ref{cor:reduction-dimensions} we give a reduction that allows us to convert an arbitrary model into an algorithm that satisfies $(\sqrt{T}, T^{\frac{-1}{3}})$-conversation calibration in the $d$-dimensional setting.

We can now state our main work-horse lemma, which again holds for \emph{perfectly} conversation calibrated predictions. It states that at any round $k$, the squared error of the vector-valued predictions must decrease compared to the squared error at the previous round, in proportion to $\epsilon^2$ and the fraction of days that do not lead to agreement at round $k$. 
We then extend the argument to predictions that have conversation-calibration error that is controlled by some function $f_h(\cdot)$. Since the arguments follow a similar analysis to those in Section \ref{sec:full-info}, applied once to each dimension $d$, we defer all proofs to Appendix \ref{app:d-dimensions}.
\begin{lemma}
   If the human is $(0,g_{h}(T))$-conversation-calibrated for $d$-dimensional vector predictions, then for any even $k$, 

    \begin{align*} \SQE(\barphk{T}{k},y^{1:T}) \leq \SQE(\barpmk{T}{k-1},y^{1:T}) -
   (\epsilon - g_{h}(T))^{2}|\Tk{k+1}| + dg_h(T)T
   \end{align*}
  
   And if the model is $(0, g_{h}(T))$-conversation-calibrated, for any odd $k$,
   \begin{align*} \SQE(\barpmk{T}{k},y^{1:T}) \leq \SQE(\barphk{T}{k-1},y^{1:T}) -
   (\epsilon - g_{m}(T))^{2}|\Tk{k+1}| + dg_m(T)T
   \end{align*}

   \label{lem:mhd} 
\end{lemma}

\begin{theorem} If the Human is $(f_{h}(\cdot), g_h(\cdot))$-conversation-calibrated in $d$ dimensions, then after engaging in the iterated calibration protocol for $T$ days: 

\begin{align*}
    \SQE(\barphk{T}{k}, y^{1:T}) \leq  \SQE(\barpmk{T}{k-1}, y^{1:T}) -(\epsilon - g_{h}(T))^{2}|\Tk{k+1}| + d \cdot g_h(T) T + 3d \cdot \frac{f_{h}(g_{h}(T) \cdot T)}{g_{h}(T)}
\end{align*}
And if the Model is $(f_{m}(\cdot), g_m(\cdot))$-conversation-calibrated in $d$ dimensions, then after engaging in the iterated calibration protocol for $T$ days: 
\begin{align*}
    \SQE(\barpmk{T}{k}, y^{1:T}) \leq  \SQE(\barphk{T}{k-1}, y^{1:T}) -(\epsilon - g_{m}(T))^{2}|\Tk{k+1}| + d \cdot g_m(T) T + 3d \cdot \frac{f_{m}(g_{m}(T) \cdot T)}{g_{m}(T)}.
\end{align*}


\label{thm:cases-dimensions}
\end{theorem}

Now, similarly to Section \ref{sec:full-info}, we introduce our reduction (Algorithm \ref{alg:converse-reduction-dimensions}) that takes as input an initial model $M_0:\cX_m\rightarrow \cY$, a sequential prediction algorithm $D$ with a concave bound $f_m(\cdot)$ on its distance to calibration, and a bucketing function $g_m(\cdot)$. We show that Algorithm \ref{alg:converse-reduction-dimensions} efficiently guarantees $(f_m(\cdot),g_m(\cdot))$-conversation calibration against any sequence of outcomes and predictions of the human.

\begin{theorem} \label{thm:reduction-dimensinoal}
If $D$ has worst-case $\CalDist$ of $f_{m}(\cdot)$, then for any bucketing function $g_m(\cdot)$,
    \textsc{Converse-dDim($M_{0}, D, g_{m}(\cdot)$)} is $(f_{m}(\cdot),g_{m}(\cdot))$-conversation-calibrated,
    and for any sequence of labels $y^{1:T}$, the first round prediction of $\textsc{Converse-dDim}$ is the same as the prediction of the base model $M_0$ for all $t$: $\textsc{Converse-dDim}_{1}(x^t_m) = M_{0}(x^t_m)$, for all $t$. 
\end{theorem}

\begin{corollary} \label{cor:reduction-dimensions}
    Algorithm \ref{alg:converse-reduction-dimensions} \textsc{Converse-dDim}($M_0, \textsc{AOSA}, T^{\frac{-1}{3}}$) is $(\sqrt{T}, T^{\frac{-1}{3}})$-conversation-calibrated.
\end{corollary}

We conclude with a final corollary putting the above together. Any arbitrary baseline model can be efficiently converted into a protocol that interacts with a human, and if this human satisfies our conversation-calibration condition, conversations will reach agreement quickly.

\begin{corollary}
    If the human runs \textsc{Converse-dDim($M^{h}_{0}, \textsc{AOSA}, g_h(\cdot)$)} and the model runs \textsc{Converse-dDim($M^{m}_{0}, \textsc{AOSA}, g_m(\cdot)$)} for $g_m(T) = g_h(T) = T^{-\frac{1}{3}}$, then: 
    \begin{itemize}
        \item For any $\epsilon,\delta \in [0,1]$, on a $1-\delta$ fraction of days, they reach $\epsilon$-agreement after at most $K$ rounds of conversation where: $K \leq \frac{d}{\epsilon^{2}\delta  - T^{-\frac{1}{3}}}$.
        \item For the subsequence of days that make it to round $k$ s.t. $|\Tk{k}| \geq \delta T$ with associated human prediction subsequences $p_h^{\Tk{k}, k}$ and states of nature subsequences $y^{\Tk{k}}$, we have that $\frac{\SQErr(p^{\Tk{k}, k},y^{\Tk{k}})}{T} \leq \frac{\SQErr(M_{0},y^{\Tk{k}})}{T} - i(\epsilon^{2}\delta - \frac{12}{T^{\frac{1}{3}}})$ for $M_0 \in \{M^{m}_{0},M^{h}_{0}\}$.
    \end{itemize}
\end{corollary}

\section{Agreement when Communicating Decisions}
\label{sec:action}

We now turn our attention to the action feedback setting (Setting \ref{def:setting-action}). Recall that in this setting, the label space $\cY \subseteq [0,1]^d$ is high dimensional, and the parties communicate with one another not by providing point predictions $\hat y \in \cY$, but rather by communicating the action $a$ in an action space $\cA$ that is utility maximizing according to their predictions. 
In this section, the messages $\pmk{t}{k}$ and $\phk{t}{k}$ denote the actions which the Human and Model communicate at each round. Note that by definition of Setting \ref{def:setting-action}, $\pmk{t}{k}$ is the optimal action given the Model's prediction of the label vector $\ymk{t}{k}$ (and the equivalent statement holds for the Human).

Rather than arguing that the squared error of the predictions decreases at each round of conversation, we will argue that the \emph{utility} of the sequence of communicated actions will increase at each iteration. Towards this end we will define shorthand notation that expresses the \emph{summed} utility of a sequence of actions over time, with respect to a sequence of outcomes.
\begin{definition}
Fix any utility function $U$ as defined in Definition \ref{def:utility}. We extend our notation to allow $U$ to take as input a \emph{sequence} of communicated actions $p^{1:T}$ and a corresponding sequence of outcomes $y^{1:T}$ by letting this denote the \emph{summed} utility as computed over this sequence: 
    \begin{align*}
        U(p^{1:T},y^{1:T}) = \sum_{t=1}^{T}U(p^{t},y^{t})
    \end{align*}
\end{definition}

Now, we can state the main theorem of this section.
\begin{theorem}
\label{thm:agreement-action}
    If the Human is $f_h(\cdot)$-decision-conversation-calibrated and the Model is $f_m(\cdot)$-decision-conversation-calibrated, then on a $1 - \delta$ fraction of days, they reach $\epsilon$-agreement in at most
    $$K \leq \frac{1}{2\epsilon \delta - \gamma(T)} + 1$$ rounds, where $\gamma(T) = \frac{2 L d |A|^{2}\cdot f_h(\frac{T}{|A|^{2}}) + 2 L d |A|^{2}\cdot f_m(\frac{T}{|A|^{2}})}{T}$ is a term that will tend to $0$ as $T$ grows large.
    Furthermore, for any even round $k$ such that $|\Tk{k}| \geq \delta T$, 
    \begin{align*}
        U(\bar{p}_h^{1:T, k}, y^{1:T}) \geq \max \left( U(\barpmk{T}{1}, y^{1:T}), U(\bar{p}_h^{1:T, 2}, y^{1:T}) \right)  + (k-1)T(2\epsilon \delta - \gamma(T)).
    \end{align*} 
\end{theorem}

\begin{corollary} \label{cor:decisions}
    When $\gamma(T) \leq \eps \delta$, on a $1-\delta$ fraction of days, the number of rounds until agreement is at most
    \begin{align*}
        K \leq \frac{1}{\eps \delta} + 1.
    \end{align*}
    And for any even round $k$ such that $|\Tk{k}| \geq \delta T$, \begin{align*}
        U(\bar{p}_h^{1:T, k}, y^{1:T}) \geq U(\bar{p}_h^{1:T, 2}, y^{1:T}) + k\epsilon \delta T
    \end{align*} 
\end{corollary}

The proof follows a similar structure to the proof of our agreement theorem in the canonical setting (Theorem \ref{thm:canonical}). At a high level, we analyze each \emph{round} $k$ of communication separately, across days. Intuitively there are again two cases. In the first case, most of the conversations that make it to round $k$ end in agreement. Again, this is a good case, as we wish to show that most conversations end in agreement quickly. In the remaining case, most of the predictions made at round $k$ $\epsilon$-disagree. Here our argument differs: Since the parties are not communicating their ($d$-dimensional) predictions directly, we cannot argue that the squared error of the predictions at round $k$ decreases. However, our notion of conversation-decision-calibration does allow us to argue that the average \emph{utility} of the predictions made at round $k$ increases substantially compared to the prior round. Thus the downstream utility of the human takes the role of squared error in our potential argument (and is what allows us to argue that the conversations are utility increasing). In fact, because our utility functions are linear, compared to the canonical setting, this allows us to get an improved rate of convergence --- depending now on $1/\epsilon$ rather than $1/\epsilon^2$. The below lemma formalizes the progress that we make at round $k$ of a conversation, across days:

\begin{lemma} \label{lem:action_mh}
If the Human is $f_h(\cdot)$-decision-conversation-calibrated, then after engaging in Protocol~\ref{alg:general-agreement} instantiated in the action feedback setting (Definition \ref{def:agree-action}) for $T$ days, for all \emph{even} rounds $k$: 

\begin{align*}
    U(\bar{p}_h^{T, k},y^{1:T}) - U(\bar{p}_m^{T, k-1},y^{1:T}) \geq \epsilon |\Tk{k+1}| - 2 L d |A|^{2}\cdot f_h\left(\frac{T}{|A|^{2}}\right)
\end{align*}

Furthermore, if the Model is $f_m(\cdot)$-decision-conversation-calibrated, then after engaging in Protocol~\ref{alg:general-agreement} instantiated in the action feedback setting (Definition \ref{def:agree-action}) for $T$ days, for all \emph{odd} rounds $k>1$:  

\begin{align*}
    U(\bar{p}_m^{T, k},y^{1:T}) - U(\bar{p}_h^{T, k-1},y^{1:T}) \geq \epsilon |\Tk{k+1}| - 2 L d |A|^{2}\cdot f_m\left(\frac{T}{|A|^{2}}\right)
\end{align*}
\end{lemma}
\begin{proof}
Let $T_{k}^{a_{h},a_{m}} = \{t: \phk{t}{k} = a_{h} \textit{ and } \pmk{t}{k-1} = a_{m}\}$ be the subsequence of days such that the human sends the message $a_{h}$ in round $k$ and the model sends the message $a_{m}$ in round $k-1$. 

By definition, for all $t \in T_{k}^{a_{h},a_{m}}$, $\argmax_{a \in \mathcal{A}}U(a,\phk{t}{k}) = a_{h}$ and $\argmax_{a \in \mathcal{A}}U(a,\pmk{t}{k-1}) = a_{m}$. Then, we can write the difference in utilities as

\begin{align*}
        & U(\barphk{T}{k}, y^{1:T}) - U(\barpmk{T}{k-1}, y^{1:T}) \\
        &= U(\phk{\Tk{k}}{k},y^{\Tk{k}}) - U(\pmk{\Tk{k}}{k-1},y^{\Tk{k}}) + \sum_{t \not \in \Tk{k}} U(\barphk{t}{k} , y^t) - U(\barpmk{t}{k-1} , y^t) \tag{by definition of $\Tk{k}$} \\
        &= U(\phk{\Tk{k}}{k},y^{\Tk{k}}) - U(\pmk{\Tk{k}}{k-1},y^{\Tk{k}}) \tag{because for all $t \not \in \Tk{k}$, it is the case that $\barphk{t}{k} = \barpmk{t}{k-1}$} \\
        & = \sum_{a_{h},a_{m} \in \mathcal{A}}\sum_{t \in T_{k}^{a_{h},a_{m}}}U(a_{h},y^{t}) - \sum_{a_{h},a_{m} \in \mathcal{A}}\sum_{t \in T_{k}^{a_{h},a_{m}}}U(a_{m},y^{t}) \\
        & = \sum_{a_{h},a_{m} \in \mathcal{A}}U(a_{h},\sum_{t \in T_{k}^{a_{h},a_{m}}}y^{t}) - \sum_{a_{h},a_{m} \in \mathcal{A}}U(a_{m},\sum_{t \in T_{k}^{a_{h},a_{m}}}y^{t}) \tag{By the linearity of $U(a,\cdot)$} \\
        & \geq \sum_{a_{h},a_{m} \in \mathcal{A}}U \left( a_{h},\sum_{t \in T_{k}^{a_{h},a_{m}}}\phk{t}{k} \right) -  \sum_{a_{h},a_{m} \in \mathcal{A}}U \left( a_{m},\sum_{t \in T_{k}^{a_{h},a_{m}}}\phk{t}{k} \right) - \\ & 2L \cdot \|\sum_{t \in T_{k}^{a_{h},a_{m}}}\phk{t}{k} - \sum_{t \in T_{k}^{a_{h},a_{m}}}y^{t}\|_{1} \tag{By the $L$-lipschitzness of $U(a,\cdot)$} \\
        & \geq \sum_{a_{h},a_{m} \in \mathcal{A}}U \left( a_{h},\sum_{t \in T_{k}^{a_{h},a_{m}}}\phk{t}{k} \right) -  \sum_{a_{h},a_{m} \in \mathcal{A}}U \left( a_{m},\sum_{t \in T_{k}^{a_{h},a_{m}}}\phk{t}{k} \right) - 2L \cdot \sum_{j \in [d]}(\sum_{t \in T_{k}^{a_{h},a_{m}}}\phk{t}{k}[j] - \sum_{t \in T_{k}^{a_{h},a_{m}}}y^{t}[j]) \\
        & \geq \sum_{a_{h},a_{m} \in \mathcal{A}}U \left( a_{h},\sum_{t \in T_{k}^{a_{h},a_{m}}}\phk{t}{k} \right) -  \sum_{a_{h},a_{m} \in \mathcal{A}}U \left( a_{m},\sum_{t \in T_{k}^{a_{h},a_{m}}}\phk{t}{k} \right) - 2Ld \cdot f_{h}(T^{h,k}_{a_{m},a_{h}})\tag{By the DC-calibration guarantee of the Human} \\
        & =  \sum_{a_{h},a_{m} \in \mathcal{A}} \sum_{t \in T_{k}^{a_{h},a_{m}}}U(a_{h},\phk{t}{k})  - \sum_{a_{h},a_{m} \in \mathcal{A}}  \sum_{t \in T_{k}^{a_{h},a_{m}}}U(a_{m},\phk{t}{k})  - \sum_{a_{h},a_{m} \in \mathcal{A}} 2Ld \cdot f_{h}(|T_{k}^{a_{h},a_{m}}|) \tag{By the linearity of $U(a,\cdot)$} \\
        & = \sum_{a_{h},a_{m} \in \mathcal{A}} \left(\sum_{t \in T_{k}^{a_{h},a_{m}}} (U(a_{h},\phk{t}{k}) - U(a_{m},\phk{t}{k})\right)  - \sum_{a_{h},a_{m} \in \mathcal{A}} 2Ld \cdot f_{h}(|T_{k}^{a_{h},a_{m}}|) \tag{By the linearity of $U(a,\cdot)$} \\ \\
        & \geq \sum_{a_{h},a_{m} \in \mathcal{A}}\sum_{t \in T_{k}^{a_{h},a_{m}}} \mathbbm{1}[U(a_{h},\phk{t}{k}) - U(a_{m},\phk{t}{k}) \geq \epsilon]  \cdot \epsilon \\ &  + \sum_{a_{h},a_{m} \in \mathcal{A}}\sum_{t \in T_{k}^{a_{h},a_{m}}}\mathbbm{1}[U(a_{h},\phk{t}{k}) - U(a_{m},\phk{t}{k}) < \epsilon] \left (U(a_{h},\phk{t}{k}) - U(a_{m},\phk{t}{k})\right)  - \sum_{a_{h},a_{m} \in \mathcal{A}} 2Ld \cdot f_{h}(|T_{k}^{a_{h},a_{m}}|) \\
        & \geq \sum_{a_{h},a_{m} \in \mathcal{A}}\sum_{t \in T_{k}^{a_{h},a_{m}}} \mathbbm{1}[U(a_{h},\phk{t}{k}) - U(a_{m},\phk{t}{k}) \geq \epsilon]  \cdot \epsilon - \sum_{a_{h},a_{m} \in \mathcal{A}} 2Ld \cdot f_{h}(|T_{k}^{a_{h},a_{m}}|) \tag{As $a_{h}$ is the best response under $\phk{t}{k}$, $\forall t \in [T^{h,k}_{a_{m},a_{h}}]$} \\
        & = \epsilon \cdot |\Tk{k+1}| - 2Ld\sum_{a_{h},a_{m} \in \mathcal{A}}  f_{h}(|T_{k}^{a_{h},a_{m}}|) \tag{As $\Tk{k+1}$ is exactly the days on which the human at round $k$ does not agree with the model}
        \\ & \geq \epsilon \cdot \Tk{k+1} - 2 Ld |A|^{2}\cdot f_h\left(\frac{|\Tk{k}|}{|A|^{2}}\right) \tag{By the concavity of $f_{h}$} \\
        & \geq \epsilon \cdot |\Tk{k+1}| - 2 Ld |A|^{2}\cdot f_h\left(\frac{T}{|A|^{2}}\right) 
    \end{align*}
    As the model and the human are symmetric, we also attain the symmetric result for the model. 
\end{proof}

We can now prove the theorem, by iteratively applying Lemma \ref{lem:action_mh} to each round of conversation.

\begin{proof}[Proof of Theorem \ref{thm:agreement-action}]

By composing both parts of Lemma~\ref{lem:action_mh}, we have that, for any even $k$,
\begin{align*}
    {U}\left(\bar{p}_h^{T, k},y^{1:T}\right) - {U}\left(\bar{p}_h^{T, k-2},y^{1:T}\right) &= \sum_{t \in \Tk{k}} U(\phk{t}{k}, y^t) - U(\phk{t}{k-2}, y^t) + \sum_{t \not \in \Tk{k}} U(\bar{p}_h^{t, k}, y^t) - U(\bar{p}_h^{t, k-2}, y^t)
    \\ 
    & = U\left(\phk{\Tk{k}}{k},y^{\Tk{k}}\right) - U\left(\phk{\Tk{k}}{k-2},y^{\Tk{k}}\right) \\
    &\geq \epsilon |\Tk{k+1}| + \epsilon |\Tk{k}| - 2 L d |A|^{2}\cdot f_h\left(\frac{T}{|A|^{2}}\right) - 2 L d |A|^{2}\cdot f_m\left(\frac{T}{|A|^{2}}\right),
\end{align*}
where the second equality holds because by definition of $\bar{p}_h^{T, k}$, all $\bar{p}_h^{t, k} = \bar{p}_h^{t, k-2}$ for all $t \not \in \Tk{k}$.
We can now recursively apply this expression to see:
    \begin{align*}
        {U}\left(\bar{p}_h^{T, k},y^{1:T}\right) - {U}\left(\bar{p}_h^{T, 2},y^{1:T}\right) &= {U}\left(\bar{p}_h^{T, k},y^{1:T}\right) - {U}\left(\phk{T}{2},y^{1:T}\right) \\
        & \geq \sum_{q=1, q \text{ even}}^{k-1} \left(\epsilon |\Tk{q+1}| + \epsilon |\Tk{q}| - 2 L d |A|^{2}\cdot f_h\left(\frac{T}{|A|^{2}}\right)  - 2 L d |A|^{2}\cdot f_m\left(\frac{T}{|A|^{2}}\right) \right),
    \end{align*}
where the equality holds as no agreement can be reached, by definition, before round 2 on any day $t$.
 
    Let us consider any round $r$ in which $|\Tk{r}| \geq \delta T$. We have that: 
      {\small
    \begin{align*}
      {U} \left(\bar{p}_h^{T, r},y^{1:T}\right)- {U}\left(\phk{T}{2},y^{1:T}\right)  & \geq \sum_{q=1,q \text{ even}}^{r-1} \left(\epsilon |\Tk{q+1}| + \epsilon |\Tk{q}| - 2 L d |A|^{2}\cdot f_h\left(\frac{T}{|A|^{2}}\right)  - 2 L d |A|^{2}\cdot f_m\left(\frac{T}{|A|^{2}}\right) \right) \\
    &  \geq - 2 \left(r-1\right) L d |A|^{2}\cdot f_h\left(\frac{T}{|A|^{2}}\right) - 2  \left(k-1\right) L d |A|^{2}\cdot f_m\left(\frac{T}{|A|^{2}}\right) \\ & \hspace{2cm} + \epsilon \sum_{q=1, q \text{ even}}^{k-1} \left(|\Tk{q+1}| + T\right) \\
    &  \geq - 2  \left(r-1\right) L d |A|^{2}\cdot f_h\left(\frac{T}{|A|^{2}}\right) - 2  \left(r-1\right) L d |A|^{2}\cdot f_m\left(\frac{T}{|A|^{2}}\right) + \epsilon \sum_{q=1, q \text{ even}}^{r-1} \left(2\delta T\right) \tag{As $|\Tk{q}|$ is $\geq \delta T$}\\
    &   \geq \left(r-1\right) \left(2\epsilon \delta T   -2 L d |A|^{2}\cdot f_h\left(\frac{T}{|A|^{2}}\right) - 2 L d |A|^{2}\cdot f_m\left(\frac{T}{|A|^{2}}\right) \right)  \\
        &  \geq \left(r-1\right) \left(2\epsilon \delta T   - T\gamma\left(T\right) \right)  
    \end{align*} }

    Finally, we can compose this expression with one more instantiation of Lemma \ref{lem:action_mh}:
    \begin{align*}
        U(\barpmk{T}{2},y^{1:T}) - U(\barpmk{T}{1},y^{1:T}) \geq \epsilon |\Tk{2}| - 2 L d |A|^{2}\cdot f_h\left(\frac{T}{|A|^{2}}\right)
    \end{align*}

    and get a final expression of:
    \begin{align*}
         {U} \left(\barphk{T}{k},y^{1:T}\right) - {U}\left(\barpmk{T}{1},y^{1:T}\right) \geq 
         r (2 \eps \delta T - T \gamma(T)).
    \end{align*}

    This proves the second result in the Theorem.

    However, we also have that $U\left(\bar{p}_h^{T, k},y^{1:T}\right) \leq T$.
    Therefore, we have that
    
    \begin{align*}
     & U\left(\bar{p}_h^{T, k},y^{1:T}\right) - U\left(\phk{T}{2},y^{1:T}\right) \geq \left(k-1\right) \left(2\epsilon \delta T   - T\gamma\left(T\right) \right) \\
      & \implies T - U\left(\phk{T}{2},y^{1:T}\right) \geq \left(k-1\right) \left(2\epsilon \delta T   - T\gamma\left(T\right) \right) \\
      & \implies T  \geq (k-1) \left(2\epsilon \delta T   - T\gamma(T\right) ) \tag{As $U(\cdot, \cdot) \geq 0$}\\
      & \implies k  \leq \frac{1}{2 \epsilon \delta - \gamma(T)} + 1
    \end{align*}
   
This proves the first result in the Theorem.
\end{proof}

We now turn to the algorithmic reduction that allows us to convert a model into an algorithm capable of maintaining conversation-decision calibrated predictions. To do so, we need to define some formalism to be able to express the guarantees of the algorithm of \cite{noarov2023high}, which informally, is able to maintain $d$ dimensional predictions that are \emph{unbiased} conditional on an arbitrary collection of specified \emph{events}. We define a special case of these events below, which is strictly less general than the type of events supported by \cite{noarov2023high}, but sufficient for our usage.

\begin{definition}[Event indicator $E(c^{t},\hat{y}^{t})$] 
    For any $t$, the event indicator function $E: \cC \times \cY \rightarrow \{0,1\}$
    takes as input the context $c^{t}$ and prediction $\hat{y}^{t}$ in round $t$, and outputs a binary indicator of whether or not event $E$ is active. 
\end{definition}

We write a collection of events as $\mathcal{E}$. We can now state the guarantees of the algorithm given in \cite{noarov2023high}:

\begin{theorem}[\cite{noarov2023high}] \label{thm:georgy}
    Given a convex compact $d$-dimensional real valued prediction space and a collection $\mathcal{E}$ of events of size $\left|\mathcal{E}\right|$, for any $0 < \alpha < 1$, the algorithm \textsc{UnbiasedPrediction} outputs, for any sequence of adaptively chosen labels, a sequence of $d$-dimensional predictions $\hat{y}^{1},\ldots,\hat{y}^{T}$ satisfying with probability $1-\alpha$, for every event $E \in \mathcal{E}$ and every coordinate $j \in [d]$: 
    \begin{equation}
        \left| \sum_{t=1}^{T}E(c^{t},\hat{y}^{t})] (\hat{y}^{t}[j] - y^t[j]) \right| \leq O\left(\log(d\left| \mathcal{E}\right| T) +  \sqrt{T\log\left(\frac{|\mathcal{E} | d}{\alpha}\right)} \right)
    \end{equation}
    The per-round running time of \textsc{UnbiasedPrediction} is polynomial in $d$ and $|\mathcal{E}|$. 
\end{theorem}

\textsc{UnbiasedPrediction} is instantiated with the set of events $\mathcal{E}$ and $\alpha$ and, on every day, takes as input a context pair needed to evaluate each event. 

In our reduction we will run a different instantiation of the UnbiasedPrediction algorithm for each round $k$, and for the $k$th instantiation, the contexts at each day $t$ will be the conversation $C^{t,1:k-1}$ that has taken place on that day so far.

We want our predictions at each round to be unbiased conditional on events which are defined by the human's recommended action in the previous round, and the model's recommended action in this round. We define the following event set accordingly:

\begin{definition}[Action-Conversation Events]
For each pair of actions $a_h, a_m \in \cA$ and each round $k$ define the event:
$$E_{a_{h},a_{m},k}(\hat{y}^{t,k},C^{t,1:k-1}) = \mathbbm{1}[\argmax_{a \in \cA} U(a, \hat{y}^{t,k}) = a_{m}]\cdot\mathbbm{1}[\phk{t}{k-1} = a_{h}] $$
Let $\mathcal{E}_{k} : = \{ E_{a_{h},a_{m},k} \forall a_{h},a_{m} \in \cA \}$.
\end{definition}

 We are now ready to define our reduction in Algorithm \ref{alg:converse-reduction-action}.

\begin{algorithm}[ht]
\begin{algorithmic}
    \STATE{\bf Input} {Baseline model algorithm $M_{0}$, Discretization $g_{m}(T)$}
    \FOR{$t = 1, \ldots, T$}
        \STATE Receive $x^t_m$
        \STATE Send prediction $\pmk{t}{1} = M_{0}(x^{t}_{m})$ to the human 
        \FOR{$k = 2, 4, 6, \ldots $}
            \STATE{$L \gets k$}
            \IF{$D_{k+1}$ uninitialized}
            \STATE Initialize $D_{k+1} = \textsc{UnbiasedPrediction}(\mathcal{E}_{k+1},\alpha)$
            \ENDIF
             \STATE Observe human action recommendation $\phk{t}{k}$ 
             \IF{$ \phk{t}{k} = \pmk{t}{k-1} $ or $|U(\phk{t}{k},\ymk{t}{k-1}) - U(\pmk{t}{k-1},\ymk{t}{k-1})| \leq \epsilon $}
            \STATE Predict $\phk{t}{k}$ and break out of loop
                \ENDIF
            \STATE Set prediction $\ymk{t}{k+1} = D_{k+1}(C^{t,1:k-1})$
            \STATE Send recommendation $\pmk{t}{k+1} = \argmax_{a \in \cA}U(a,\ymk{t}{k+1})$ to human
        \ENDFOR
    \STATE Observe $y^{t}$
    \FOR{$k \in 2, 4, \ldots, L$ }
    \STATE Update $D_{k+1}$ with $y^{t}$
    \ENDFOR
    \ENDFOR
\end{algorithmic}
\caption{\textsc{Converse-Action($M_{0},\alpha$)}} 
\label{alg:converse-reduction-action}
\end{algorithm}

\begin{theorem} \label{thm:reduction-action}
\textsc{Converse-Action($M_{0},\alpha$)} is $O\left(\log(2d|\cA|^{2}T + \sqrt{T\ln\left(\frac{|\cA|^{2}d}{\alpha}\right)}\right)$-DC-calibrated with probability $1 - \alpha$, and for any sequence of labels $y^{1:T}$, its first-round prediction is the same as the prediction of the base model $M_0$ for all $t$: $\textsc{Converse-Action($M_{0},\alpha$)}_{1}(x^t_m) = M_{0}(x^t_m)$, for all $t$. 
\end{theorem}
\begin{proof}
By construction, in each odd round $k$, \textsc{Converse-Action($M_{0},\alpha$)} runs $\textsc{UnbiasedPrediction}(\mathcal{E}_{k},\alpha)$ with subsequences defined by $\mathcal{E}_{k}$ in order to obtain predictions. By Theorem~\ref{thm:georgy}, in each round, the bias on subsequences defined by the model's action recommendation and the human's action recommendation on the previous round is $O\left(\log(2d\left| \mathcal{E}\right| T) +  \sqrt{T\ln\left(\frac{|\mathcal{E} | d}{\alpha}\right)} \right)$. Thus the algorithm is $O\left(\log(2d\left| \mathcal{E}\right| T) +  \sqrt{T\ln\left(\frac{|\mathcal{E} | d}{\alpha}\right)} \right)$-DC-calibrated.

The second result follows directly from the definition of \textsc{Converse-Action($M_{0},\alpha$)}.
\end{proof}

\begin{theorem}
\label{thm:blorb}
Fix an $L$-Lipschitz utility function $U$.  If the human runs \textsc{Converse-Action($M^{h}_{0},\alpha$)} and the model runs \textsc{Converse-Action($M^{m}_{0},\alpha$)}, then, if $T \geq \frac{O\left(L^{2} d^{3} |A|^{5} (1 + \log(\frac{1}{\alpha})) \right)}{\epsilon^{2} \delta^{2}}$, with probability $\geq 1 - 2\alpha$, on a $1 - \delta$ fraction of days, the number of rounds until agreement is at most 
    \[ K \leq \frac{1}{\epsilon \delta} + 1\]
      Furthermore, for any round $k$ such that $|\Tk{k}| \geq \delta T$, 
        $$U(p_{h}^{1:T,k},y^{1:T}) \geq U(p_{m}^{\Tk{k},k-1},y^{1:T}) + k\epsilon \delta T$$
\end{theorem}
We defer the proof to Appendix \ref{app:action}.

\section{Bayesian Agreement Theorems} \label{sec:bayesian}
In this section we show how to recover one-shot Agreement Theorems for Bayesians with a common prior, in the style of past work \cite{aumann1976,geanakoplos1982we,aaronson2004complexity,kong2023false,frongillo2023agreement}. In most of this paper, we have studied a repeated interaction across many days, within an environment about which we have made no assumptions. Our theorems hinged on tractable calibration conditions that we imposed on the participants. In contrast, past work on agreement theorems has assumed two interlocutors who share common and complete knowledge of a \emph{prior distribution} from which instances are drawn, and are perfect Bayesians --- at each round of conversation, they condition on everything they have observed (the features of the instance they have seen, as well as the transcript of the conversation), and report their posterior expectation of the label. The strength of the approach that we have taken in most of this paper is that we do not need to assume any distributional knowledge (or even the existence of a distribution), and our assumptions on the agents are tractable (in contrast to an assumption that the agents can compute posterior distributions, which is in general intractable in large state spaces). On the other hand, our guarantees are necessarily about sequences of many interactions, whereas past work on Aumannian agreement theorems give guarantees for conversations about \emph{single} instances, that hold with high probability over the draw of the instance from the prior distribution. 

In this section, we show that our theorems are strictly more general than this one-shot setting, in that all of our theorems can be ``lifted'' to the one-shot setting if we are willing to make the assumption (as past work does) that instances are drawn from a commonly known prior and that the agents report correct posterior expectations. To demonstrate this, we prove two things:

\begin{enumerate}
    \item First, we show that in the sequential setting, if the instance at each round is drawn independently from a known prior distribution, then an Agent who reports the posterior expectation of the label at each round of conversation (conditional on everything they have observed so far, including the transcript of the conversation) will satisfy our various notions of conversation calibration, no matter how their interlocutor is behaving. This result is in the spirit of \cite{dawid1982well}, and our analysis proceeds according to the following thought experiment: when arguing that the Bayesian is conversation-calibrated at some round $k$ of the conversation, we imagine that at each day $t$, the label $y^t$ is re-drawn from the Bayesian's posterior distribution on $y^t$ at round $k$. This does not change the joint distribution on transcripts, and so any statement that is true of transcripts under this thought experiment is true under the original transcript distribution. But within this thought experiment, the Bayesian is always announcing the true  mean of the label distribution just before the label is sampled --- (conversation) calibration bounds therefore follow from standard Martingale concentration arguments. 
\item Next, we observe that if two Bayesians are interacting with one another in the sequential setting, and the instance is drawn i.i.d. at each day, then the conversation that they have at each day $t$ is statistically independent of all previous days. We know (from part 1) that if we allow them to interact across sufficiently many days, the transcript of their conversations will be arbitrarily well conversation calibrated, and hence in the canonical setting, they will agree on a $1-\delta$ fraction of days after $k = 1/\epsilon^2\delta$ many rounds. Similar guarantees with different bounds hold in each of our other settings.  However, because the conversations at each round are identically and independently distributed, the transcript distribution is permutation invariant --- and hence the two Bayesians will agree on the \emph{first} day after at most $k = 1/\epsilon^2\delta$ many rounds, with probability $1-\delta$ over the selection of a day from the transcript, which is equivalent to a $1-\delta$ probability guarantee over the draw of the instance from the underlying prior. 
\end{enumerate}
Hence we conclude that our theorems extend to the 1-shot Bayesian setting and generalize and extend past work on Bayesian agreement. In particular we give quantitative convergence bounds in the style of \cite{aaronson2004complexity} that are independent of the complexity of the instance, but are able to recover theorems not just in the cannonical setting, but in the $d$-dimensional and action feedback settings as well.

\subsection{Bayesians are Conversation Calibrated} 

In this section we begin by showing that if the instance at each day is drawn from a prior distribution $\cD$, and one of the Agents is a Bayesian who correctly computes predictions as posterior expectations given the prior $\cD$ and all observed evidence, then when interacting with any other agent, they are guaranteed to maintain conversations that satisfy any of our calibration conditions. We start by defining how a Bayesian learner interacts in a conversation. 

\begin{definition}[Bayesian Learner] \label{def:bayesian-learner}
Fix a prior  $\cD \in \Delta(\cX_h\times \cX_m \times \cY)$ specifying a joint distribution over features observable to both the human and the model and labels. 
    We say that a human (respectively, model) is a Bayesian Learner with prior $\cD$ if given a known algorithm for the model, for all $t, k >0$, given observable features $x^t$, message transcript $\mu^{1:t-1}$, prediction transcript $\pi^{1:t-1}_h$ of human predictions (respectively, $\pi^{1:t-1}_m$ of model predictions) through day $t-1$, and conversation $C^t_{1:k-1}$, they make a prediction as
    \begin{align*}
        \yhk{t}{k} = \E_{\cD}[Y | x^t, \mu^{1:t-1}, \pi^{1:t-1}_h, C^t_{1:k-1}]\quad (\text{respectively, } \ymk{t}{k} = \E_{\cD}[Y | x^t, \mu^{1:t-1}, \pi^{1:t-1}_m, C^t_{1:k-1}] ).
    \end{align*}
\end{definition}

\begin{protocol}[ht]
\begin{algorithmic}
    \STATE{ {\bf Input} $(\cD, \Omega_h, \Omega_m, \cY, \textsc{Agree}_{\epsilon}$) }
    \FOR{each day $t = 1, \ldots$}
        \STATE Receive $x^t = (x^t_h,x^t_m, y^t) \sim \cD$. The model sees $x^t_m$ and the human sees $x^t_h$.
        \FOR{each round $k = 1, 2, \ldots,L$}
            \IF { $k$ is odd}
                \STATE The Model predicts $\ymk{t}{k} \in \cY$, and sends the Human $\pmk{t}{k} \in \Omega_m$ 
                \IF{ $ \textsc{Agree}_{\eps}( \phk{t}{k-1}, \yhk{t}{k-1}, \pmk{t}{k}, \ymk{t}{k} ) $} 
                \STATE Return $\pmk{t}{k}$ and break out of loop
                \ENDIF
            \ENDIF
            \IF{ $k$ is even}
                \STATE The Human  predicts $\yhk{t}{k} $, and sends the model $\phk{t}{k} \in \Omega_h$
                \IF{ $ \textsc{Agree}_{\eps}( \phk{t}{k}, \yhk{t}{k}, \pmk{t}{k-1}, \ymk{t}{k-1} ) $} 
                    \STATE Return $\pmk{t}{k-1}$ and break out of loop
                \ENDIF
            \ENDIF
        \ENDFOR
        \STATE{The Human and Model observe $y^t \in \cY$}
    \ENDFOR

\end{algorithmic}
\caption{\textsc{Bayesian $\eps-$Agreement Protocol}}  \label{alg:bayesian-agreement}
\end{protocol}

Protocol \ref{alg:bayesian-agreement} is the same as our general agreement protocol (Protocol \ref{alg:general-agreement}), except that the instance at each day $t$ is drawn i.i.d. from a prior distribution $\cD$, rather than being chosen by an adversary. We will prove the following theorem, which states that if the human is a Bayesian learner, then they will satisfy strong conversation calibration constraints of various forms.

\begin{theorem} \label{thm:bayes-calibrated}
    Consider an interaction over $T$ rounds under Protocol \ref{alg:bayesian-agreement}.
    If the human (respectively, model) is a Bayesian Learner (Definition \ref{def:bayesian-learner}), then for any model (respectively, human) algorithm, for any $n > 0$, with probability $1 - \delta$, they are 
     \begin{itemize}
         \item $\left( O(T^{\frac34} (\log\frac{dn}{\delta})^{\frac{1}{4}}), \frac{1}{n}\right)$-conversation calibrated, and
         \item $\left(2 \sqrt{ 2 T \log \frac{d |\cA|^2 }{\delta} }\right)$-DC-conversation-calibrated.
     \end{itemize}
\end{theorem}

First we formalize a simple observation in the following lemma. It states that if we resample the label every day after the $j^\text{th}$ round of conversation \emph{from the posterior distribution on the label conditional on the transcript of interaction so far}, that this does not change the distribution of transcripts. An upshot of this lemma is that all of our subsequent analysis can proceed under this resampling thought experiment. 

\begin{lemma} \label{lem:resampling-transcript} 
    Let $\cD$ be a probability distribution over space $\cX_m \times \cX_h \times \cY$ and fix a day $t \in [T]$. 
    Fix a transcript through day $t-1$: $\pi^{1:t-1}$.
    \begin{itemize}
        \item Consider an interaction at day $t$ under Protocol \ref{alg:bayesian-agreement}.
        Let $\pi^t$ be the transcript of day $t$ from this interaction.
        \item Fix an arbitrary round $j$. Consider an interaction when $(x_m, x_h, y^t)$ is sampled from $\cD$ at the beginning of day $t$ and then the human and model correspond according to Protocol \ref{alg:bayesian-agreement} until round $j$. Then, in round $j$, the outcome is resampled from the posterior distribution conditional on the information observed by the human so far: $y' \sim \cD_{\cY} | x^t_h, \mu^{1:t-1}, {\pi}^{1:t-1}_h, C^{t}_{1:j-1}, \pmk{t}{j}$. Let $\bar{\pi}^t_{j}$ be the transcript of day $t$ from this interaction, with $y^t$ replaced with $y'$.
    \end{itemize}
    For all rounds $k$,
     \begin{align*}
        \Pr_{\cD}[\pi^{t, 1:k}] = \Pr_{\cD}[\bar{\pi}^{t,1:k}_{j}].
    \end{align*}
\end{lemma}
The proof can be found in Appendix \ref{app:bayesian}

  

Lemma \ref{lem:resampling-transcript} tells us that we can proceed in our analysis by imagining that at any round $j$ on which the Bayesian learner sends a message, they send a message that is consistent with the \emph{true} label expectation at that round, as we can imagine that the label is resampled according to its posterior expectation. This means that the Bayesian's forecasts are unbiased, and so by Azuma's inequality, the average of the Bayesian's forecasts should equal the average of the realized label up to small error terms on any sequence that is sufficiently long. Thus the rest of the analysis consists of identifying sufficiently long sequences on which bounding the bias of the Bayesian's predictions in this way is sufficient to bound each notion of calibration error. This is enough to straightforwardly give us a bound on the Bayesian's decision conversation calibration error, since DC-conversation-calibration error is simply the maximum bias in any coordinate of the learner's predictions conditional on the best response action defined by the Bayesian's prediction and the action communicated by the other agent; thus there are only $|\cA|^2$ many sequences on which we need to bound the bias, and the result will follow from Azuma's inequality and a union bound. However, conversation calibration  is defined in terms of \emph{distance to calibration}, which is more subtle. Distance to calibration is upper bounded by expected calibration error (ECE), however the empirical ECE of a Bayesian will in general \emph{not} be bounded, as they might make a different prediction at every round, and hence there will be no sequences of fixed predictions of length $> 1$, and hence we have no ability to invoke concentration. Instead, we will bound the Bayesian's \emph{bucketed expected calibration error}, defined next, and use this to upper bound distance to calibration.

\begin{definition}[Bucketed Expected Calibration Error] 
Given a sequence of predictions $p^{1:T}$ and outcomes $y^{1:T}$, the expected calibration error with respect to bucketing coarseness $n$ (Definition \ref{def:bucketing}) is
\[
\ECE(p^{1:T}, y^{1:T}; n) = \sum_{i = 1}^n \left| \sum_{t=1}^T \mathbbm{1}[p^t \in  B_n(i)] (p^t - y^t) \right|.
\]
\end{definition}

\begin{lemma} \label{lem:distance-to-bucketed}
    Fix a sequence of of predictions $p^{1:t}$ and outcomes $y^{1:T}$. Then, $\CalDist(p^{1:T}, y^{1:T}) \leq \ECE(p^{1:T}, y^{1:T}; n) + \frac{T}{n}$.
\end{lemma}
The proof is in Appendix \ref{app:bayesian}.

Bounding bucketed calibration error for a Bayesian can be done via Azuma's inequality: it now reduces to bounding the empirical bias of the predictions conditional on the bucket of the prediction, which for a bucketing parameter $n$ consists of $n$ subsequences, each of which we can apply Azuma's inequality to. The final bounds come from optimizing $n$, trading off the need to sum over the magnitude of the bias on each sequence defined by a bucketing (which is costlier for larger $n$) and the need to bound distance to calibration using Lemma \ref{lem:distance-to-bucketed} (which is costlier for smaller $n$). The details are in Appendix \ref{app:bayesian}.

\subsection{An Online to One-Shot Reduction}

In this section, we show that if an instance is drawn from a commonly known prior, and \emph{both} agents are Bayesian, then all of our theorems that bound the conversation length $K$ for a $1-\delta$ fraction of conversations over an arbitrarily long sequence of length $T$ in fact hold for a \emph{single} conversation, with probability $1-\delta$ over the draw of the instance from the prior distribution. The idea is straightforward: We can \emph{imagine} an arbitrarily long sequence of conversations over many days. Because we showed that Bayesians satisfy our notions of conversation calibration with parameters growing sublinearly with $T$, our theorems apply with the  error terms going to $0$ as $T$ grows large, and we can conclude that their conversations are short for a $1-\delta$ fraction of days. But we can also observe that because the instances are drawn i.i.d. from a fixed prior, and in such a setting Bayesians need not condition on any information from prior days, the conversation on each day is distributed identically. Hence it must be that \emph{each} conversation (and in particular the first) is bounded with probability $1-\delta$ over the prior. We therefore conclude that our theorems hold for a single conversation between Bayesians.

\begin{protocol}[ht]
\begin{algorithmic}
    \STATE{ {\bf Input} $(\Omega_h, \Omega_m, \cY, \textsc{Agree}_{\epsilon}$, $\cD \in \Delta(\cX_h \times \cX_m \times \cY)$, instance for which you want agreement: $(x^*_h, x^*_m, y^*) \sim \cD$ }
    \STATE{ {\bf Parameter} agreement tolerance: $\eps$, failure probability: $\delta$, number of samples: $T$}
  \STATE{Let $(x^1_h, x^1_m, y^1)=(x^*_h, x^*_m, y^*)$}
  \STATE{For $t \in \{2,\ldots,T\}$ draw $(x^t_h, x^t_m, y^t) \sim \cD$}
    \FOR{each day $t = 1, \ldots,T$}
        \STATE Model observes $x^t_m$ and Human observes $x^t_h$.
        \FOR{each round $k = 1, 2, \ldots,L$}
            \IF { $k$ is odd}
                \STATE The Model predicts $\ymk{t}{k} \in \cY$, and sends the Human $\pmk{t}{k} \in \Omega_m$ 
                \IF{ $ \textsc{Agree}_{\eps}( \phk{t}{k-1}, \yhk{t}{k-1}, \pmk{t}{k}, \ymk{t}{k} ) $} 
                \STATE Return $\pmk{t}{k}$ and break out of loop
                \ENDIF
            \ENDIF
            \IF{ $k$ is even}
                \STATE The Human  predicts $\yhk{t}{k} $, and sends the model $\phk{t}{k} \in \Omega_h$
                \IF{ $ \textsc{Agree}_{\eps}( \phk{t}{k}, \yhk{t}{k}, \pmk{t}{k-1}, \ymk{t}{k-1} ) $} 
                    \STATE Return $\pmk{t}{k-1}$ and break out of loop
                \ENDIF
            \ENDIF
        \ENDFOR
        \STATE{The Human and Model observe $y^t \in \cY$}
    \ENDFOR

\end{algorithmic}
\caption{\textsc{Online-to-One-Shot: General Bayesian $\eps-$Agreement Protocol}}  \label{alg:offline}
\end{protocol}

We define a hypothetical conversation protocol (Protocol \ref{alg:offline}) that takes as input a prior distribution $\cD$ and a single instance $(x^*_h, x^*_m, y^*)$ drawn from the prior distribution that we want fast agreement on. The hypothetical protocol runs our agreement protocol for $T$ rounds, using the supplied instance $(x^*_h, x^*_m, y^*)$  on day $1$, and using freshly sampled instances from the prior at all subsequent days. Note that we will never run Protocol \ref{alg:offline} --- in particular, in reality, we do not want to have to know the label $y^*$ before the Bayesians converse --- but it will be a useful thought experiment. 

A fixed Human algorithm, denoted $H$, a fixed Model algorithm, denoted $M$, a prior $\cD$, and hypothetical Protocol~\ref{alg:offline} together define a distribution over transcripts. Of particular interest to us will be the distribution over conversation lengths at each round. Let $\ell_{t}(H,M,\cD)$ represent the conversation length at round $t$ of the transcript induced by $H$, $M$, and $\cD$ in Protocol~\ref{alg:offline}. We first observe that the conversation lengths are identically distributed at each day of Protocol \ref{alg:offline}, since the instances each day are i.i.d.:

\begin{lemma} If $H$ and $M$ are Bayesian learners, then 
$\Pr(\ell_{t_{1}}(H,M,\cD) \geq k) = \Pr(\ell_{t_{2}}(H,M,\cD) \geq k)$, $\forall t_{1}, t_{2} \in [1,\ldots,T], k \in \mathbb{N}$. 
\label{lem:permute}
\end{lemma}

\begin{proof}
Because the instance at each day $t$ is drawn i.i.d., the predictions of a Bayesian Learner in round $k$ are a function only of the prior $\cD$, the feature vector they observe ($x_{h}^{t}$ or $x_{m}^{t}$) and the conversation up to that round $C^{t}_{1:k-1}$. Therefore, given two Bayesian learners, $\ell_{t}(H,M,\cD)$ is a function only of $(x_{m}^{t},x_{h}^{t},y^{t})$. But $(x_{m}^{t},x_{h}^{t},y^{t})$ are i.i.d. for all $t$. Therefore, $\Pr(\ell_{t_{1}}(H,M,\cD) \geq k) = \Pr(\ell_{t_{2}}(H,M,\cD) \geq k)$, $\forall t_{1}, t_{2} \in [1,\ldots,T], k \in \mathbb{N}$.  
\end{proof}

Next, we show that in the limit as the number of rounds $T$ in the hypothetical Protocol \ref{alg:offline} tends to infinity, we can give a high probability bound on the length $K$ of the first conversation in Protocol \ref{alg:offline} --- i..e the conversation pertaining to the relevant instance $(x^*_h, x^*_m, y^*)$. This follows because 1) Bayesians become increasingly conversation calibrated as $T$ grows large, and so we can apply our theorems establishing that a $1-\delta$ \emph{fraction} of conversations in Protocol \ref{alg:offline} are short, and because 2) all conversation lengths are identically distributed, so if most conversations are short, it must also be that the \emph{first} conversation is short with high probability. 

\begin{theorem}\label{thm:bb}
    Fix any $\epsilon, \delta \in [0,1]$ and any instance $(x^*_h, x^*_m, y^*) \sim \cD$.  
    If the Human and the Model are both Bayesian learners, then under Protocol \ref{alg:offline}, in the limit as $T\rightarrow \infty$ they will reach $\epsilon-$agreement with probability $1 - \delta$ on day $1$ (i.e. the day corresponding to the instance $(x^*_h, x^*_m, y^*)$) within
    \begin{itemize}
        \item $K \leq \frac{3d}{\epsilon^{2}\delta}$ rounds in the full feedback setting.
        \item $K \leq \frac{3}{2\eps \delta}+1 $ rounds in the action feedback setting.
    \end{itemize}
\end{theorem}

\begin{proof}[Proof of Theorem \ref{thm:bb}]
\textbf{The Full Feedback Setting:}
    By Theorem~\ref{thm:bayes-calibrated}, if we run the protocol for $T$ rounds, then with probability $1-2\delta/3$ both the Human and the Model are $\left(2\left(2T \log(\frac{3d T^{3/7}}{\delta})\right)^\frac{1}{4}, T^{-\frac{3}{7}}\right)$-conversation-calibrated. Assume for now that these calibration bounds hold.

    Note that Protocol~\ref{alg:offline} is simply a special case of Protocol~\ref{alg:general-agreement}, in which $(x^{t}_h, x^{t}_m, y^{t})$ are drawn  from a fixed distribution. Therefore, the guarantees from Theorem~\ref{thm:dimensions} hold, and we have that, any $\epsilon,\delta \in [0,1]$, on a $1-\delta/3$ fraction of days, they reach $\eps$-agreement after at most $K$ rounds of conversation for 
    $K \leq \frac{3d}{\epsilon^{2}\delta  - \beta(T)}$
    and where $\beta(T)= 3d\left(2T^{-\frac{3}{7}} +  \frac{4 (2T\cdot T^{-\frac{3}{7}}\log(\frac{3d T^{\frac{3}{7}}}{\delta}))^{\frac{1}{4}}}{T \cdot T^{-\frac{3}{7}}}\right)$. But $\lim_{T\rightarrow \infty}\beta(T) = 0$, and so we have that for every $\eta > 0$,  $K < \frac{3d}{\epsilon^{2}\delta  - \eta}$ Therefore we must have  $K \leq \frac{3d}{\epsilon^{2}\delta}$.

    Now, note that by Lemma~\ref{lem:permute}, the distribution over conversation lengths at each day is identical. Therefore, we have that

$$    \Pr_{t \sim Unif(1:T)}\left[\ell_{t} \geq \frac{3d}{\epsilon^{2}\delta}\right] \leq \frac{\delta}{3} 
        \implies 
        \Pr\left[\ell_{1} \geq \frac{3d}{\epsilon^{2}\delta}\right] \leq \frac{\delta}{3} $$

Summing up all three failure probabilities, we have that
 $$\Pr\left[\ell_{1} \geq \frac{3d}{\epsilon^{2}\delta}\right] \leq \delta$$
 
\textbf{The Action Feedback Setting}
    By Theorem \ref{thm:bayes-calibrated}, if we run the protocol for $T$ rounds, the human and model are both 
    $(2 \sqrt{2T \log \frac{3d |A|^2}{\delta}})$-decision-conversation calibrated with probability $1-\frac{2\delta}{3}$. Assume for now these two calibration bounds hold. 
    We instantiate Theorem \ref{thm:agreement-action}: the human and model will reach $\eps-$agreement on a $1 - \delta/3$ fraction of days, after at most
    \begin{align*}
        K \leq \frac{1}{2\epsilon \frac{\delta}{3} - \gamma(T)} + 1
    \end{align*}
    rounds of conversation, where $\gamma(T) = \frac{4 L d |A|^{2}\cdot\sqrt{2T \log \frac{3d |A|^2}{\delta}}) + 4 L d |A|^{2}\cdot \sqrt{2T \log \frac{3d |A|^2}{\delta}})}{T}$. Here $\lim_{T\rightarrow \infty} \gamma(T) = 0$. So once again we have that for every $\eta > 0$, $K < \frac{1}{2\epsilon \frac{\delta}{3} - \eta} + 1$. Hence it must be that:
    $K \leq  \frac{3}{2\epsilon \delta} + 1$.
    
    Now, note that by Lemma~\ref{lem:permute}, the distribution over conversation lengths at each day is identical. Therefore, we have that

    $$    \Pr_{t \sim Unif(1:T)}\left[\ell_{t} \geq   \frac{3}{2\epsilon \delta} + 1\right] \leq \frac{\delta}{3} 
        \implies 
        \Pr\left[\ell_{1} \geq  \frac{3}{2\epsilon \delta} + 1\right] \leq \frac{\delta}{3}$$
    Summing up over all three failure probabilities yields

    $$   \Pr\left[\ell_{1} \geq  \frac{3}{2\epsilon \delta} + 1\right] \leq \delta $$


\end{proof}

Finally we note that since we have proven that agreement happens quickly with high probability over the draw of the instance from the prior on \emph{the first round} of Protocol \ref{alg:offline} in the limit as $T$ grows large, but the interaction at round $1$ is independent of $T$, there is no need to run the protocol for more than a single round --- we have proven agreement theorems in the ``one-shot'' setting of prior work \cite{aumann1976,geanakoplos1982we,aaronson2004complexity,frongillo2023agreement}.

\begin{corollary}
\label{corr:Bayesian}
    Fix any $\epsilon, \delta \in [0,1]$ and any instance $(x^*_h, x^*_m, y^*) \sim \cD$. 
    If the Human and the Model are both Bayesian learners, then under Protocol \ref{alg:offline} with $T = 1$, they will reach $\epsilon-$agreement on the instance $(x^*_h, x^*_m, y^*)$, with probability $1 - \delta$ after at most
    \begin{itemize} 
        \item $K \leq \frac{3d}{\epsilon^{2}\delta}$ rounds in the full feedback setting.
        \item $K \leq \frac{3}{2\eps \delta}+1 $ rounds in the action feedback setting.
    \end{itemize}
\end{corollary}

\section{Discussion and Conclusion}
Bayesian rationality is an attractive, canonical model of optimal learning that has been adopted in many economic models, including not just agreement (as we study in this paper), but also Bayesian Persuasion \cite{kamenica2011bayesian}, reputation systems \cite{mailath2006repeated}, and social herding \cite{banerjee1992simple}. While attractive, Bayesian reasoning is not computationally or statistically tractable, and so models that assume perfect Bayesian agents are either limited to speaking of extremely simple prior distributions or require making implausible assumptions on the knowledge and computational power of the agents. Motivated in part by these concerns, there is also a large literature that studies learning under simple behavioral assumptions (dating back to \cite{simon1955behavioral,tversky1992advances}) --- but these models are generally incompatible with Bayesian reasoning, and hence are inherently less canonical --- they require making choices about how to model agent behavior that have no firm theoretical grounding.

Our work suggests a third approach: We make computationally and statistically tractable calibration assumptions that are strict relaxations of Bayesian rationality, and hence are satisfied by perfect learners, but do not require implausible assumptions. In the case of agreement theorems, we have shown that these tractable calibration conditions were \emph{all that was needed} from Bayesian rationality, in that we are able to prove (and generalize) agreement theorems that recover the same quantitative bounds that were known under full Bayesian rationality under our weaker assumptions. Is this a more general phenomenon? Perhaps in many other settings in which Bayesian rationality was previously thought to be a necessary modeling assumption, the same results can be obtained under significantly weaker calibration-based assumptions that can be guaranteed by efficient online calibration algorithms of various flavors. 

\subsubsection*{Acknowledgments}
We thank Joe Halpern and Rakesh Vohra for useful comments. We also thank Kavya Ravichandran for helpful discussions on human AI collaboration. This work was supported in part by the Simons Collaboration on the Theory of Algorithmic Fairness, NSF grants FAI-2147212 and CCF-2217062, the Hans Sigrist Prize, and an OpenAI SuperAlignment Fast Grant. 

\bibliography{bib}

\begin{thebibliography}{42}
\providecommand{\natexlab}[1]{#1}
\providecommand{\url}[1]{\texttt{#1}}
\expandafter\ifx\csname urlstyle\endcsname\relax
  \providecommand{\doi}[1]{doi: #1}\else
  \providecommand{\doi}{doi: \begingroup \urlstyle{rm}\Url}\fi

\bibitem[Aaronson(2005)]{aaronson2004complexity}
Scott Aaronson.
\newblock The complexity of agreement.
\newblock In \emph{Proceedings of the thirty-seventh annual ACM symposium on Theory of computing}, pages 634--643, 2005.

\bibitem[Alur et~al.(2024)Alur, Raghavan, and Shah]{alur2024human}
Rohan Alur, Manish Raghavan, and Devavrat Shah.
\newblock Human expertise in algorithmic prediction.
\newblock \emph{arXiv preprint arXiv:2402.00793}, 2024.

\bibitem[Arunachaleswaran et~al.(2025)Arunachaleswaran, Collina, Roth, and Shi]{arunachaleswaran2024}
Eshwar~Ram Arunachaleswaran, Natalie Collina, Aaron Roth, and Mirah Shi.
\newblock An elementary predictor obtaining $2\sqrt{T}+1$ distance to calibration.
\newblock \emph{IEEE Symposium on Discrete Algorithms (SODA)}, 2025.
\newblock URL \url{https://arxiv.org/abs/2402.11410}.

\bibitem[Aumann(1976)]{aumann1976}
Robert~J. Aumann.
\newblock {Agreeing to Disagree}.
\newblock \emph{The Annals of Statistics}, 4\penalty0 (6):\penalty0 1236 -- 1239, 1976.
\newblock \doi{10.1214/aos/1176343654}.
\newblock URL \url{https://doi.org/10.1214/aos/1176343654}.

\bibitem[Banerjee(1992)]{banerjee1992simple}
Abhijit~V Banerjee.
\newblock A simple model of herd behavior.
\newblock \emph{The Quarterly Journal of Economics}, 107\penalty0 (3):\penalty0 797--817, 1992.

\bibitem[B{\l}asiok et~al.(2023)B{\l}asiok, Gopalan, Hu, and Nakkiran]{blasiok2023unifying}
Jaros{\l}aw B{\l}asiok, Parikshit Gopalan, Lunjia Hu, and Preetum Nakkiran.
\newblock A unifying theory of distance from calibration.
\newblock In \emph{Proceedings of the 55th Annual ACM Symposium on Theory of Computing}, pages 1727--1740, 2023.

\bibitem[Camara et~al.(2020)Camara, Hartline, and Johnsen]{camara2020mechanisms}
Modibo~K Camara, Jason~D Hartline, and Aleck Johnsen.
\newblock Mechanisms for a no-regret agent: Beyond the common prior.
\newblock In \emph{2020 ieee 61st annual symposium on foundations of computer science (focs)}, pages 259--270. IEEE, 2020.

\bibitem[Chan et~al.(2024)Chan, Chen, Su, Yu, Xue, Zhang, Fu, and Liu]{chan2024chateval}
Chi-Min Chan, Weize Chen, Yusheng Su, Jianxuan Yu, Wei Xue, Shanghang Zhang, Jie Fu, and Zhiyuan Liu.
\newblock Chateval: Towards better llm-based evaluators through multi-agent debate.
\newblock In \emph{The Twelfth International Conference on Learning Representations}, 2024.

\bibitem[Collina et~al.(2024)Collina, Roth, and Shao]{collina2023efficient}
Natalie Collina, Aaron Roth, and Han Shao.
\newblock Efficient prior-free mechanisms for no-regret agents.
\newblock \emph{The ACM Conference on Economics and Computation (EC)}, 2024.

\bibitem[Dagan et~al.(2024)Dagan, Daskalakis, Fishelson, Golowich, Kleinberg, and Okoroafor]{dagan2024improved}
Yuval Dagan, Constantinos Daskalakis, Maxwell Fishelson, Noah Golowich, Robert Kleinberg, and Princewill Okoroafor.
\newblock Improved bounds for calibration via stronger sign preservation games.
\newblock \emph{arXiv preprint arXiv:2406.13668}, 2024.

\bibitem[Dawid(1982)]{dawid1982well}
A.~P. Dawid.
\newblock The well-calibrated bayesian.
\newblock \emph{Journal of the American Statistical Association}, 77\penalty0 (379):\penalty0 605--610, 1982.
\newblock \doi{10.1080/01621459.1982.10477856}.
\newblock URL \url{https://www.tandfonline.com/doi/abs/10.1080/01621459.1982.10477856}.

\bibitem[Dawid(1985)]{dawid1985calibration}
A.~P. Dawid.
\newblock {Calibration-Based Empirical Probability}.
\newblock \emph{The Annals of Statistics}, 13\penalty0 (4):\penalty0 1251 -- 1274, 1985.
\newblock \doi{10.1214/aos/1176349736}.
\newblock URL \url{https://doi.org/10.1214/aos/1176349736}.

\bibitem[Deshpande et~al.(2022)Deshpande, Mossel, and Sohn]{deshpande2022agreement}
Yash Deshpande, Elchanan Mossel, and Youngtak Sohn.
\newblock Agreement and statistical efficiency in bayesian perception models.
\newblock \emph{arXiv preprint arXiv:2205.11561}, 2022.

\bibitem[Du et~al.(2024{\natexlab{a}})Du, Ngo, and Wu]{du2024reconciling}
Ally~Yalei Du, Dung~Daniel Ngo, and Zhiwei~Steven Wu.
\newblock Reconciling model multiplicity for downstream decision making.
\newblock \emph{arXiv preprint arXiv:2405.19667}, 2024{\natexlab{a}}.

\bibitem[Du et~al.(2024{\natexlab{b}})Du, Li, Torralba, Tenenbaum, and Mordatch]{du2024improving}
Yilun Du, Shuang Li, Antonio Torralba, Joshua~B Tenenbaum, and Igor Mordatch.
\newblock Improving factuality and reasoning in language models through multiagent debate.
\newblock In \emph{Forty-first International Conference on Machine Learning}, 2024{\natexlab{b}}.

\bibitem[Foster and Kakade(2006)]{foster2006calibration}
Dean~P Foster and Sham~M Kakade.
\newblock Calibration via regression.
\newblock In \emph{2006 IEEE Information Theory Workshop-ITW'06 Punta del Este}, pages 82--86. IEEE, 2006.

\bibitem[Foster and Vohra(1999)]{FV99}
Dean~P. Foster and Rakesh Vohra.
\newblock Regret in the on-line decision problem.
\newblock \emph{Games and Economic Behavior}, 29\penalty0 (1):\penalty0 7--35, 1999.
\newblock ISSN 0899-8256.
\newblock \doi{https://doi.org/10.1006/game.1999.0740}.
\newblock URL \url{https://www.sciencedirect.com/science/article/pii/S0899825699907406}.

\bibitem[Foster and Vohra(1998)]{foster1998asymptotic}
Dean~P Foster and Rakesh~V Vohra.
\newblock Asymptotic calibration.
\newblock \emph{Biometrika}, 85\penalty0 (2):\penalty0 379--390, 1998.

\bibitem[Frongillo et~al.(2023)Frongillo, Neyman, and Waggoner]{frongillo2023agreement}
Rafael Frongillo, Eric Neyman, and Bo~Waggoner.
\newblock Agreement implies accuracy for substitutable signals.
\newblock In \emph{Proceedings of the 24th ACM Conference on Economics and Computation}, pages 702--733, 2023.

\bibitem[Gale and Kariv(2003)]{gale2003bayesian}
Douglas Gale and Shachar Kariv.
\newblock Bayesian learning in social networks.
\newblock \emph{Games and economic behavior}, 45\penalty0 (2):\penalty0 329--346, 2003.

\bibitem[Garg et~al.(2019)Garg, Kim, and Reingold]{garg2019tracking}
Sumegha Garg, Michael~P Kim, and Omer Reingold.
\newblock Tracking and improving information in the service of fairness.
\newblock In \emph{Proceedings of the 2019 ACM Conference on Economics and Computation}, pages 809--824, 2019.

\bibitem[Geanakoplos and Polemarchakis(1982)]{geanakoplos1982we}
John~D Geanakoplos and Heraklis~M Polemarchakis.
\newblock We can't disagree forever.
\newblock \emph{Journal of Economic theory}, 28\penalty0 (1):\penalty0 192--200, 1982.

\bibitem[Globus-Harris et~al.(2024)Globus-Harris, Gupta, Kearns, and Roth]{globusharris2024model}
Ira Globus-Harris, Varun Gupta, Michael Kearns, and Aaron Roth.
\newblock Model ensembling for constrained optimization, 2024.
\newblock URL \url{https://arxiv.org/abs/2405.16752}.

\bibitem[Gopalan et~al.(2023)Gopalan, Hu, Kim, Reingold, and Wieder]{gopalan2023loss}
Parikshit Gopalan, Lunjia Hu, Michael~P Kim, Omer Reingold, and Udi Wieder.
\newblock Loss minimization through the lens of outcome indistinguishability.
\newblock In \emph{14th Innovations in Theoretical Computer Science Conference (ITCS 2023)}, 2023.

\bibitem[Gupta et~al.(2022)Gupta, Jung, Noarov, Pai, and Roth]{gupta2022online}
Varun Gupta, Christopher Jung, Georgy Noarov, Mallesh~M Pai, and Aaron Roth.
\newblock Online multivalid learning: Means, moments, and prediction intervals.
\newblock \emph{Innovations in Theoretical Computer Science (ITCS)}, 2022.

\bibitem[H{\'e}bert-Johnson et~al.(2018)H{\'e}bert-Johnson, Kim, Reingold, and Rothblum]{hebert2018multicalibration}
Ursula H{\'e}bert-Johnson, Michael Kim, Omer Reingold, and Guy Rothblum.
\newblock Multicalibration: Calibration for the (computationally-identifiable) masses.
\newblock In \emph{International Conference on Machine Learning}, pages 1939--1948. PMLR, 2018.

\bibitem[Hu and Wu(2024)]{HW24}
Lunjia Hu and Yifan Wu.
\newblock Predict to minimize swap regret for all payoff-bounded tasks.
\newblock \emph{IEEE Symposium on Foundations of Computer Science (FOCS)}, 2024.

\bibitem[Kamenica and Gentzkow(2011)]{kamenica2011bayesian}
Emir Kamenica and Matthew Gentzkow.
\newblock Bayesian persuasion.
\newblock \emph{American Economic Review}, 101\penalty0 (6):\penalty0 2590--2615, 2011.

\bibitem[Kleinberg et~al.(2023)Kleinberg, Leme, Schneider, and Teng]{kleinberg2023u}
Bobby Kleinberg, Renato~Paes Leme, Jon Schneider, and Yifeng Teng.
\newblock U-calibration: Forecasting for an unknown agent.
\newblock In \emph{The Thirty Sixth Annual Conference on Learning Theory}, pages 5143--5145. PMLR, 2023.

\bibitem[Kong and Schoenebeck(2023)]{kong2023false}
Yuqing Kong and Grant Schoenebeck.
\newblock False consensus, information theory, and prediction markets.
\newblock In \emph{14th Innovations in Theoretical Computer Science Conference (ITCS 2023)}, volume 251, page~81. Schloss Dagstuhl--Leibniz-Zentrum f $\{$$\backslash$" u$\}$ r Informatik, 2023.

\bibitem[Liang et~al.(2023)Liang, He, Jiao, Wang, Wang, Wang, Yang, Tu, and Shi]{liang2023encouraging}
Tian Liang, Zhiwei He, Wenxiang Jiao, Xing Wang, Yan Wang, Rui Wang, Yujiu Yang, Zhaopeng Tu, and Shuming Shi.
\newblock Encouraging divergent thinking in large language models through multi-agent debate.
\newblock \emph{arXiv preprint arXiv:2305.19118}, 2023.

\bibitem[Mailath and Samuelson(2006)]{mailath2006repeated}
George~J Mailath and Larry Samuelson.
\newblock \emph{Repeated Games and Reputations: Long-Run Relationships}.
\newblock Oxford University Press, 2006.

\bibitem[Mossel et~al.(2014)Mossel, Sly, and Tamuz]{mossel2014asymptotic}
Elchanan Mossel, Allan Sly, and Omer Tamuz.
\newblock Asymptotic learning on bayesian social networks.
\newblock \emph{Probability Theory and Related Fields}, 158\penalty0 (1):\penalty0 127--157, 2014.

\bibitem[Noarov et~al.(2023)Noarov, Ramalingam, Roth, and Xie]{noarov2023high}
Georgy Noarov, Ramya Ramalingam, Aaron Roth, and Stephan Xie.
\newblock High-dimensional prediction for sequential decision making, 2023.
\newblock URL \url{https://arxiv.org/abs/2310.17651}.

\bibitem[Qiao and Valiant(2021)]{qiao2021stronger}
Mingda Qiao and Gregory Valiant.
\newblock Stronger calibration lower bounds via sidestepping.
\newblock In \emph{Proceedings of the 53rd Annual ACM SIGACT Symposium on Theory of Computing}, pages 456--466, 2021.

\bibitem[Qiao and Zheng(2024)]{qiao2024distance}
Mingda Qiao and Letian Zheng.
\newblock On the distance from calibration in sequential prediction.
\newblock \emph{arXiv preprint arXiv:2402.07458}, 2024.

\bibitem[Roth and Shi(2024)]{RS24}
Aaron Roth and Mirah Shi.
\newblock Forecasting for swap regret for all downstream agents.
\newblock \emph{The ACM Conference on Economics and Computation (EC)}, 2024.

\bibitem[Roth et~al.(2023)Roth, Tolbert, and Weinstein]{roth2023reconciling}
Aaron Roth, Alexander Tolbert, and Scott Weinstein.
\newblock Reconciling individual probability forecasts, 2023.
\newblock URL \url{https://arxiv.org/abs/2209.01687}.

\bibitem[Sandroni et~al.(2003)Sandroni, Smorodinsky, and Vohra]{sandroni2003calibration}
Alvaro Sandroni, Rann Smorodinsky, and Rakesh~V Vohra.
\newblock Calibration with many checking rules.
\newblock \emph{Mathematics of operations Research}, 28\penalty0 (1):\penalty0 141--153, 2003.

\bibitem[Simon(1955)]{simon1955behavioral}
Herbert~A Simon.
\newblock A behavioral model of rational choice.
\newblock \emph{The quarterly journal of economics}, pages 99--118, 1955.

\bibitem[Tversky and Kahneman(1992)]{tversky1992advances}
Amos Tversky and Daniel Kahneman.
\newblock Advances in prospect theory: Cumulative representation of uncertainty.
\newblock \emph{Journal of Risk and uncertainty}, 5:\penalty0 297--323, 1992.

\bibitem[Zhao et~al.(2021)Zhao, Kim, Sahoo, Ma, and Ermon]{zhao2021calibrating}
Shengjia Zhao, Michael Kim, Roshni Sahoo, Tengyu Ma, and Stefano Ermon.
\newblock Calibrating predictions to decisions: A novel approach to multi-class calibration.
\newblock \emph{Advances in Neural Information Processing Systems}, 34:\penalty0 22313--22324, 2021.

\end{thebibliography}
\bibliographystyle{plainnat}
\newpage\appendix
\section{Additional Material from Section~\ref{sec:full-info}}
\label{app:full-info}

\begin{lemma} If $m = \frac{1}{T}\sum_{t=1}^{T}y^{t}$, then for any constant $x$,
\begin{equation}
\SQE(x,y^{1:T}) - \SQE(m,y^{1:T})  = \sum_{t=1}^{T}(x-m)^{2}
\end{equation} 
\label{lem:squares_diff}
\end{lemma}
\begin{proof}
\begin{align*}
\SQE(x,y^{1:T}) - \SQE(m,y^{1:T}) & = 
    \sum_{t=1}^{T}(x-y_{t})^{2} - \sum_{t=1}^{T}(m-y_{t})^{2}  &  
    \\  &  =\sum_{t=1}^{T}(m-y_{t} + x - m)^{2} - \sum_{t=1}^{T}(m-y_{t})^{2}
    \\  &  =\sum_{t=1}^{T}(m-y_{t})^{2} + \sum_{t=1}^{T}(x-m)^{2} + \sum_{t=1}^{T}2(m-y_{t})(x-m) - \sum_{t=1}^{T}(m-y_{t})^{2}
    \\  &  =\sum_{t=1}^{T}(x-m)^{2} + \sum_{t=1}^{T}2(m-y_{t})(x-m) 
    \\  &  =\sum_{t=1}^{T}(x-m)^{2} + 2(x-m) \sum_{t=1}^{T}(m-y_{t})
    \\  &  =\sum_{t=1}^{T}(x-m)^{2} 
\end{align*}
\end{proof}

\begin{lemma}
Let $T_{k}^{i,p_{h}} = \{t: \phk{t}{k} = p_{h} \textit{ and } \pmk{t}{k-1} \in B_{i}(\frac{1}{g(T)})\}$ be the subsequence of days such that the human predicts $p_{h}$ in round $k$ and the model predicts in bucket $B_{i}(\frac{1}{g(T)})$ in round $k-1$. Let $m_{k}^{i,p_h} = \frac{\sum_{t \in T_{k}^{i,p_{h}}}y^{t}}{|T_{k}^{i,p_{h}}|} $ be the true mean on this subsequence.
If the human is $(\cdot, g_{h}(T))$-conversation calibrated, then

\begin{equation}
    \sum_{t \in T_{k}^{i,p_{h}}}(\pmk{t}{k-1} - y^{t})^{2} - \sum_{t \in T_{k}^{i,p_{h}}}(i \cdot g_{h}(T) - y^{t})^{2}  \geq - g_{h}(T) \cdot |T_{k}^{i,p_{h}}|
\end{equation}
\label{lem:v1}
\end{lemma}
\begin{proof}

Note that for any $t$ such that $\ell_{t} \geq k$, $(i-1) \cdot g_{h}(T) \leq \pmk{t}{k-1} \leq  i \cdot g_{h}(T)$, by the human's bucketing condition. Therefore, we also have that $(\pmk{t}{k-1})^{2} \geq ((i-1)g_{h}(T))^{2}$.

\begin{align*}
   &  \sum_{t \in T_{k}^{i,p_{h}}}(\pmk{t}{k-1} - y^{t})^{2} - \sum_{t \in T_{k}^{i,p_{h}}}(i \cdot g_{h}(T) - y^{t})^{2} \\
   & =   \sum_{t \in T_{k}^{i,p_{h}}} \left( (\pmk{t}{k-1})^{2} -2\pmk{t}{k-1}y^{t}  + (y^{t})^{2} \right) - \sum_{t \in T_{k}^{i,p_{h}}}\left((i \cdot g_{h}(T))^{2} -2 (i \cdot g_{h}(T))y^{t} + (y^{t})^{2}\right) \tag{Expanding}\\
    & =  \sum_{t \in T_{k}^{i,p_{h}}} (\pmk{t}{k-1})^{2}  - (i \cdot g_{h}(T))^{2} -2\pmk{t}{k-1}y^{t}   + 2 (i \cdot g_{h}(T))y^{t}  \tag{Cancelling out the $(y^{t})^{2}$}
    \\ &  \geq  \sum_{t \in T_{k}^{i,p_{h}}} ((i-1)g_{h}(T))^{2}  - (i \cdot g_{h}(T))^{2} - 2\pmk{t}{k-1}y^{t}    + 2 (i \cdot g_{h}(T))y^{t}  \tag{As $(\pmk{t}{k-1})^{2} \geq ((i-1)g_{h}(T))^{2}$}
    \\ &  \geq  \sum_{t \in T_{k}^{i,p_{h}}} ((i-1)g_{h}(T))^{2}  - (i \cdot g_{h}(T))^{2} - 2(i \cdot g_{h}(T)))y^{t}    + 2 (i \cdot g_{h}(T))y^{t}  \tag{As $\pmk{t}{k-1} \leq  i \cdot g_{h}(T)$}
     \\ &  = \sum_{t \in T_{k}^{i,p_{h}}} ((i-1)g_{h}(T))^{2}  - (i \cdot g_{h}(T))^{2}  
    \\ &  = \sum_{t \in T_{k}^{i,p_{h}}} \left( (i-1)^{2} - i^{2} \right) g_{h}(T)^{2}  
    \\ &  = \sum_{t \in T_{k}^{i,p_{h}}} (1 - 2i)  g_{h}(T)^{2}  
    \\ &  \geq \sum_{t \in T_{k}^{i,p_{h}}} (1 - \frac{2}{g_{h}(T)})  g_{h}(T)^{2}  \tag{As $i \leq \frac{1}{g_{h}(T)}$}
      \\ &  =\sum_{t \in T_{k}^{i,p_{h}}} (g_{h}(T)^{2}  - g_{h}(T)) 
        \\ &  \geq - |T_{k}^{i,p_{h}}| \cdot  (g_{h}(T)) 
\end{align*}

\end{proof}

\begin{lemma}
    Consider any sequence of predictions and labels $p^{1:T}, y^{1:T}$ such that $p$ is perfectly calibrated on $y$, and some other sequence of predictions $q^{1:T}$ such that $||p^{1:T} - q^{1:T}|| \leq \gamma$. Then, $$\sum_{t=1}^T(q^{t} - y^{t})^{2} - \sum_{t=1}^T(p^{t} - y^{t})^{2} \leq 3\gamma$$ \label{lem:bound_error_diff}
\end{lemma}
\begin{proof}

\begin{align*}
     \sum_{t=1}^T(p^{t} - y^{t})^{2} - \sum_{t=1}^T\left(q^{t} - y^{t}\right)^{2} &=  \sum_{t=1}^T(p^{t})^{2} + (y^{t})^{2} - 2p^{t}y^{t} - \left(\sum_{t=1}^T(q^{t})^{2} + (y^{t})^{2} - 2q^{t}y^{t}\right) \\
     & =  \sum_{t=1}^T((p^{t})^{2} -(q^{t})^{2}) + \sum_{t=1}^T(2q^{t}y^{t}  - 2p^{t}y^{t}) \\
     & \leq  \sum_{t=1}^T((p^{t})^2 - (q^{t})^2) + 2\gamma  \tag{as $\|p^{1:T} - q^{1:T}\| \leq \gamma$ and $y^t \in [0, 1]$} \\
     & \leq  \sum_{t=1}^T((p^{t}) -(q^{t})) + 2\gamma  \tag{as $p^t, q^t \in [0, 1]$} \\
     & = 3\gamma.
\end{align*}
   
\end{proof}

\begin{algorithm}[ht]
\begin{algorithmic}
    \STATE{\bf Input} Sequence of outcomes $y^{1:T} \in \{0,1\}^{T}$, where $T$ is unknown a priori
    \STATE{\bf Output} Sequence of predictions $p^{1:T} \in \{0, \frac{1}{m},..., 1\}^T$ for some discretization parameter $m>0$
    \STATE{$t \gets 1$}
    \STATE{$\bar{T} \gets 1$}
    \WHILE{$t \leq T$}{
    \STATE{$t_{0} \gets t$}
    \STATE{$\bar{T} \gets 2\cdot \bar{T}$}
    \WHILE{$t \leq \bar{T}$ \AND $t \leq T$}{   
    \STATE Given look-ahead predictions $\tilde p^{t_{0}:t-1}$, define the look-ahead bias conditional on a prediction $p$ as:
    $$\alpha_{\tilde p^{1:t-1}}(p) := \sum_{s=t_{0}}^{t-1} \mathbb{I}[\tilde p^s = p] (\tilde p^s - y^s)$$
    \STATE Choose two adjacent points $p_i = \frac{i}{m}, p_{i+1} = \frac{i+1}{m}$ satisfying: $$\alpha_{\Tilde{p}^{t_{0}:t-1}}(p_i) \leq 0 \text{ and } \alpha_{\Tilde{p}^{t_{0}:t-1}}(p_{i+1}) \geq 0$$
    \STATE Arbitrarily predict ${p}^t = p_i$ or ${p}^t = p_{i+1}$\;
    \STATE Upon observing the (adversarially chosen) outcome $y^t$, set look-ahead prediction $$\Tilde{p}^t = \argmin_{p \in \{p_i, p_{i+1}\}} |p - y^t|$$
    \STATE $t \gets t+1$
 }
 \ENDWHILE}
 \ENDWHILE
\end{algorithmic}
\caption{\textsc{AOST}: Almost-One-Step-Ahead with unknown $T$} \label{alg:2}
\end{algorithm}

\begin{proof}[Proof of Theorem \ref{thm:unknownT}]
    Consider running Algorithm~\ref{alg:2} with some sequence of outcomes $y^{1:T}$. Note that, by construction, when the algorithm terminates, $\frac{\bar{T}}{2} \leq T \leq \bar{T}$. We can upper bound the distance to calibration as
    \begin{align*}
      \CalDist(p^{1:T}y^{1:T}) & \leq \sum_{i=1}^{\log_{2}(\bar{T})-1}\CalDist(p^{2^{i-1}:2^{i}}, y^{2^{i-1}:2^{i}}) + \CalDist(p^{\frac{\bar{T}}{2}:T}, y^{\frac{\bar{T}}{2}:T}) \\
      & \leq  \sum_{i=1}^{\log_{2}(\bar{T})-1}(2 \sqrt{2^{i-1}} + 1) + (2 \sqrt{T} + 1) \tag{Algorithm runs a separate version of AOST for each $\bar{T}$, and by Theorem~\ref{thm:aost}} \\
     & = \sum_{i=1}^{\log_{2}(\bar{T})-1}( 2^{\frac{i+1}{2}} + 1) + (2 \sqrt{T} + 1) \\
     & = \log_{2}(\bar{T}) - 1 + \sum_{i=1}^{\log_{2}(\bar{T})-1}( 2^{\frac{i+1}{2}}) + (2 \sqrt{T} + 1) \\
     & = \log_{2}(\bar{T})+ 2 \sqrt{T}  + \sum_{i=1}^{\log_{2}(\bar{T})-1}( (\sqrt{2})^{i+1})  \\
    & = \log_{2}(\bar{T}) + 2\sqrt{T} + (\sqrt{2})^{2} \cdot \frac{\sqrt{2}^{\log_{2}(\bar{T})-1}-1}{\sqrt{2} - 1} \leq \log_{2}(\bar{T}) + 2\sqrt{T} + 2 \cdot \frac{\sqrt{\bar{T}}-1}{\sqrt{2}-1} \tag{As this is a geometric series} \\
      & \leq \log_{2}(2T) + 2\sqrt{T} + 2 \cdot \frac{\sqrt{2T}-1}{\sqrt{2}-1}  \\
      & = O(\sqrt{T}) \\
    \end{align*}
\end{proof}

\section{Additional Material from Section \ref{sec:d-dimensions}} \label{app:d-dimensions}

\begin{proof}[Proof of Lemma \ref{lem:mhd}]

Let $T_{k}^{i,p_{h}}[j] = \{t: \phk{t}{k}[j] = p_{h} \textit{ and } \pmk{t}{k-1}[j] \in B_{i}(\frac{1}{g(T)})\}$ be the subsequence of days such that the $\phk{t}{k}[j] = p_{h}$, and $\pmk{t}{k-1}[j] \in B_{i}(\frac{1}{g(T)})$. Let $m_{k}^{i,p_h}[j] = \frac{\sum_{t \in T_{k}^{i,p_{h}}[j]}y^{t}}{|T_{k}^{i,p_{h}}[j]|} $ be the true mean on this subsequence. The difference in squared error of predictions in dimension $j$ over this subsequence can be written as

\begin{align*}
 & \sum_{t \in T_{k}^{i,p_{h}}[j]}(\pmk{t}{k-1}[j] - y^{t}[j])^{2} - \sum_{t \in T_{k}^{i,p_{h}}}(\phk{t}{k}[j] - y^{t}[j])^{2} 
 \\ & = \left[\sum_{t \in T_{k}^{i,p_{h}}}(\pmk{t}{k-1}[j] - y^{t}[j])^{2} - \sum_{t \in T_{k}^{i,p_{h}}[j]}(m_{k}^{i,p_h}[j] - y^{t}[j])^{2} \right]  \\ & - \left[\sum_{t \in T_{k}^{i,p_{h}}[j]}(\phk{t}{k}[j] - y^{t}[j])^{2}  - \sum_{t \in T_{k}^{i,p_{h}}[j]}(m_{k}^{i,p_h}[j] - y^{t}[j])^{2} \right]  \tag{Adding and subtracting $\sum_{t \in T_{k}^{i,p_{h}}[j]}(m_{k}^{i,p_h}[j] - y^{t}[j])^{2}$}
 \\ & \geq \left[\sum_{t \in T_{k}^{i,p_{h}}[j]}(i \cdot g_{h}(T) - y^{t}[j])^{2} -|T_{k}^{i,p_{h}}[j]| \cdot g_h(T)  - \sum_{t \in T_{k}^{i,p_{h}}[j]}(m_{k}^{i,p_h}[j] - y^{t}[j])^{2} \right] \\ &  \qquad\qquad-\left[\sum_{t \in T_{k}[j]^{i,p_{h}}}(\phk{t}{k}[j] - y^{t}[j])^{2}  - \sum_{t \in T_{k}[j]^{i,p_{h}}}(m_{k}[j]^{i,p_h} - y^{t}[j])^{2} \right] \tag{By Lemma~\ref{lem:v1}}
  \\ & =\left[\sum_{t \in T_{k}[j]^{i,p_{h}}}(i \cdot g_{h}(T) - m_{k}[j]^{i,p_{h}})^{2} - |T_{k}^{i,p_{h}}[j]| \cdot g_h(T) \right] -  \left[\sum_{t \in T_{k}^{i,p_{h}}[j]}(\phk{t}{k}[j] - y^{t}[j])^{2}  - \sum_{t \in T_{k}^{i,p_{h}}[j]}(m_{k}^{i,p_h}[j] - y^{t}[j])^{2} \right]  \tag{By Lemma~\ref{lem:squares_diff}}
    \\ & = \left[\sum_{t \in T_{k}[j]^{i,p_{h}}}(i \cdot g_{h}(T) - m_{k}[j]^{i,p_{h}})^{2} - |T_{k}^{i,p_{h}}[j]| \cdot g_h(T)   \right] - \left[\sum_{t \in T_{k}^{i,p_{h}}[j]}(p_h - y^{t}[j])^{2}  - \sum_{t \in T_{k}^{i,p_{h}}[j]}(m_{k}^{i,p_h}[j] - y^{t}[j])^{2} \right] \tag{As by definition of $T_{k}^{i,p_{h}}[j]$, $\phk{t}{k}[j] = p_{h}$} 
      \\ & \geq \left[\sum_{t \in T_{k}^{i,p_{h}}[j]}(i \cdot g_{h}(T) - m_{k}^{i,p_{h}}[j])^{2} - |T_{k}^{i,p_{h}}[j]| \cdot g_h(T) \right]-  \left[\sum_{t \in T_{k}^{i,p_{h}}[j]}(p_h - m_{k}^{i,p_{h}}[j])^{2} \right]  \tag{By Lemma~\ref{lem:squares_diff}}
           \\ & \geq - |T_{k}^{i,p_{h}}[j]| \cdot g_h(T)  +  \sum_{t \in T_{k}^{i,p_{h}}[j]}(i \cdot g_{h}(T) - p_{h})^{2}   \tag{As the human is $(0, g_{h}(T))$-conversation calibrated, $p_{h} = m_{k}^{i,p_{h}}[j]$}
    \end{align*}
Using this analysis, we can write the difference in the squared errors over the the entire sequences $\barphk{T}{k}$ and $\barpmk{T}{k-1}$ as follows, where the first term comes from summing the above expression for over all $i,p_{h}$:
    \begin{align*}
        & \SQE(\barpmk{T}{k-1}, y^{1:T}) - \SQE(\barphk{T}{k}, y^{1:T}) \\
        & = \sum_{\forall i, p_{h}} \left(- |T_{k}^{i,p_{h}}[j]| \cdot g_h(T)  +  \sum_{t \in T_{k}^{i,p_{h}}[j]}(i \cdot g_{h}(T) - p_{h})^{2}  \right) + \sum_{t \not \in \Tk{k}} (\pmk{t}{k-1} - y^t)^2 - (\phk{t}{k} - y^t)^2 \\
        & = \sum_{\forall i, p_{h}} \left(- |T_{k}^{i,p_{h}}[j]| \cdot g_h(T)  +  \sum_{t \in T_{k}^{i,p_{h}}[j]}(i \cdot g_{h}(T) - p_{h})^{2}  \right) \tag{as $\barpmk{t}{k-1} = \barphk{t}{k}$ for all $t \not \in \Tk{k}$, by definition } \\
        & = \sum_{\forall i, p_{h}} - |T_{k}^{i,p_{h}}[j]| \cdot g_h(T)  +  \sum_{t \in T_{k}^{i,p_{h}}[j]}(i \cdot g_{h}(T) - p_{h})^{2} \\
        & \geq -g_{h}(T)T + \sum_{\forall i, p_{h}} \sum_{t \in T_{k}^{i,p_{h}}[j]}(i \cdot g_{h}(T) - p_{h})^{2}  \tag{As $g_{h}(T)$ is independent of $i$ and $p_{h}$, and $\sum_{\forall i, p_{h}}\left|T_{k}^{i,p_{h}}[j]\right| \leq T$} \\
        & \geq -g_{h}(T)T + \sum_{\forall i, p_{h}} \sum_{t \in T_{k}^{i,p_{h}}[j]}\mathbbm{I}[|i \cdot g_{h}(T) - \phk{t}{k}[j]| \geq \epsilon - g_{h}(T)](i \cdot g_{h}(T) - p_{h})^{2}   \\
        & \geq -g_{h}(T)T + (\epsilon - g_{h}(T))^{2} \sum_{\forall i, p_{h}}\sum_{t \in T_{k}^{i,p_{h}}} \mathbbm{I}[|i \cdot g_{h}(T) - \phk{t}{k}[j]| \geq \epsilon - g_{h}(T)]  
    \end{align*}

Note that, for all days in the subsequence $T_{k}^{i,p_{h}}[j]$, in round $k-1$ the model predicted in bucket $B_{i}(\frac{1}{g_{h}(T)}) = i \cdot g_{h}(T)$ in dimension $j$, and therefore in each of these days, by the definition of our bucketing, $\pmk{t}{k-1}[j] \geq (i-1) \cdot g_{h}(T)$ and $\pmk{t}{k-1}[j] \leq i \cdot g_{h}(T)$. So consider any round $t \in T_{k}^{i,p_{h}}[j]$. If $|\phk{t}{k}[j] - \pmk{t}{k-1}[j]| \geq \epsilon$, then we have:

\begin{align*}
    |\phk{t}{k}[j] - \pmk{t}{k-1}[j]| &\le  |\phk{t}{k}[j] - i\cdot g_{h}(T)| + |i\cdot g_{h}(T) - \pmk{t}{k-1}[j]|\\
    &= |\phk{t}{k}[j] - i\cdot g_{h}(T)| + i\cdot g_{h}(T) - \pmk{t}{k-1}[j]\\
    &\le |\phk{t}{k}[j] - i\cdot g_{h}(T)|  + i\cdot g_{h}(T) - (i-1) \cdot g_{h}(T)\\
    &= |\phk{t}{k}[j] - i\cdot g_{h}(T)| + g_{h}(T),\\
    \implies |\phk{t}{k}[j] - i\cdot g_{h}(T)| & \ge |\phk{t}{k}[j] - \pmk{t}{k-1}[j]| - g_{h}(T) \ge \epsilon - g_{h}(T).
\end{align*}

Thus, if $|\phk{t}{k}[j] - \pmk{t}{k-1}[j]| \geq \epsilon$, then $|i \cdot g_{h}(T) - \phk{t}{k}[j]| \geq \epsilon - g_{h}(T)$, $\forall t \in T_{k}^{i,p_{h}}[j]$. Therefore the set of days for which the former condition holds is a subset of the latter condition, and we can write
    
\begin{align*}
    & -g_{h}(T)T + (\epsilon - g_{h}(T))^{2} \sum_{\forall i, p_{h}} \mathbbm{I}[|i \cdot g_{h}(T) - p_{h}| \geq \epsilon - g_{h}(T)] \cdot \left|T_{k}^{i,p_{h}}[j]\right|  \\
    &\qquad\qquad\qquad \geq -g_{h}(T)T + (\epsilon - g_{h}(T))^{2} \sum_{\forall i, p_{h}} \sum_{t \in T_{k}^{i,p_{h}}[j]} \mathbbm{I}[|\phk{t}{k}[j] - \pmk{t}{k-1}[j]| \geq \epsilon] \\
\end{align*}

Thus we have that 
\begin{align*}
 &  \sum_{\forall i, p_{h}} \left(\sum_{t \in T_{k}^{i,p_{h}}[j]}(\pmk{t}{k-1}[j] - y^{t}[j])^{2} - \sum_{t \in T_{k}^{i,p_{h}}}(\phk{t}{k}[j] - y^{t}[j])^{2}\right)\\ &\qquad\qquad\qquad \geq  -g_{h}(T)T + (\epsilon - g_{h}(T))^{2} \sum_{\forall i, p_{h}} \sum_{t \in T_{k}^{i,p_{h}}[j]} \mathbbm{I}[|\phk{t}{k}[j] - \pmk{t}{k-1}[j]| \geq \epsilon]
\end{align*}

Summing this up for all dimensions $j$:

\begin{align*}
 &  \sum_{\forall j \in [d]}\left(\sum_{\forall i, p_{h}} \left(\sum_{t \in T_{k}^{i,p_{h}}[j]}(\pmk{t}{k-1}[j] - y^{t}[j])^{2} - \sum_{t \in T_{k}^{i,p_{h}}}(\phk{t}{k}[j] - y^{t}[j])^{2}\right)\right)\\ & \geq  -dg_{h}(T)T + (\epsilon - g_{h}(T))^{2} \sum_{\forall j \in [d]} \left(\sum_{\forall i, p_{h}} \sum_{t \in T_{k}^{i,p_{h}}[j]} \mathbbm{I}[|\phk{t}{k}[j] - \pmk{t}{k-1}[j]| \geq \epsilon]\right)
 \\ & =  -dg_{h}(T)T + (\epsilon - g_{h}(T))^{2} \sum_{\forall i, p_{h}} \sum_{\forall j \in [d]} \sum_{t \in T_{k}^{i,p_{h}}[j]} \mathbbm{I}[|\phk{t}{k}[j] - \pmk{t}{k-1}[j]| \geq \epsilon]
  \\ & \geq -dg_{h}(T)T + (\epsilon - g_{h}(T))^{2}  \sum_{\forall i, p_{h}} \sum_{t \in T_{k}^{i,p_{h}}} \mathbbm{I}[\exists j \in [d] .s.t. |\phk{t}{k}[j] - \pmk{t}{k-1}[j]| \geq \epsilon] \\
   \\ & = -dg_{h}(T)T + (\epsilon - g_{h}(T))^{2} \sum_{t \in T_{k}} \mathbbm{I}[\exists j \in [d] .s.t. |\phk{t}{k}[j] - \pmk{t}{k-1}[j]| \geq \epsilon] \\
    \\ & \geq -dg_{h}(T)T + (\epsilon - g_{h}(T))^{2} |\Tk{k+1}| \tag{As for every day that proceeds further than round $k$, there is $\epsilon$ disagreement in at least one coordinate}
\end{align*}

As the human and the model are perfectly symmetrical, we also obtain the symmetrical result for the model.
\end{proof}

\begin{proof}[Proof of Theorem \ref{thm:dimensions}]
By composing the two results in Theorem~\ref{thm:cases-dimensions}:

\begin{align*}
  & \SQE(\barphk{T}{k-2}, y^{1:T}) - \SQE(\barphk{T}{k}, y^{1:T}) \\ & \geq   (\epsilon - g_{h}(T))^{2}|\Tk{k+1}| +(\epsilon - g_{m}(T))^{2}|\Tk{k}|  - d \cdot (g_h(T) + g_{m}(T)) T - 3d \cdot (\frac{f_{h}(g_{h}(T) \cdot T)}{g_{h}(T)} + (\frac{f_{m}(g_{m}(T) \cdot T)}{g_{m}(T)})
\end{align*}

Now, consider any round $r$ such that $|\Tk{r}| \geq \delta T$. By applying this expression recursively, we can bound the squared error of the human at round $r$ by 

\begin{align*}
   &\SQErr(\barphk{T}{r},y^{1:T}) \\
   &\leq \SQErr(\barphk{T}{2},y^{1:T}) - ((\epsilon - g_{m}(T))^{2} + (\epsilon - g_{h}(T))^{2})\left(\sum_{k=1}^{r}|\Tk{k}|\right) + d(g_m(T) + g_{h}(T)) \left(\sum_{k=1}^{r}|\Tk{k}|  \right) \\
   &\qquad + 3d\left(\frac{f_{m}(g_{m}(T) \cdot T)}{g_{m}(T)} + \frac{f_{h}(g_{h}(T) \cdot T)}{g_{h}(T)}\right)\left(\sum_{k=1}^{r}1\right) \\
  & \leq \SQErr(\barphk{T}{2},y^{1:T}) - ((\epsilon - g_{m}(T))^{2} + (\epsilon - g_{h}(T))^{2})\left(\sum_{k=1}^{r}|\Tk{k}|\right) + d(g_m(T) + g_{h}(T)) (r)T\\
  & \qquad + 3d\left(\frac{f_{m}(g_{m}(T) \cdot T)}{g_{m}(T)} + \frac{f_{h}(g_{h}(T) \cdot T)}{g_{h}(T)}\right)(r) \tag{As $|\Tk{k}| \leq T$} \\
& \leq \SQErr(\barphk{T}{2},y^{1:T}) - ((\epsilon - g_{m}(T))^{2} + (\epsilon - g_{h}(T))^{2})(r)\delta T  + 2d(g_m(T) + g_{h}(T)) (r)T\\
&\qquad + 3d\left(\frac{f_{m}(g_{m}(T) \cdot T)}{g_{m}(T)} + \frac{f_{h}(g_{h}(T) \cdot T)}{g_{h}(T)}\right)(r) \tag{As for all $\Tk{k}$ such that $k\leq r$, $|\Tk{k}| \geq \delta T$} \\
& \leq \SQErr(\barphk{T}{2},y^{1:T}) - (\epsilon^{2})(r)\delta T  + 3d(g_m(T) + g_{h}(T)) (r)T\\
&\qquad + 3d\left(\frac{f_{m}(g_{m}(T) \cdot T)}{g_{m}(T)} + \frac{f_{h}(g_{h}(T) \cdot T)}{g_{h}(T)}\right)(r)\\
& \leq \SQErr(\barphk{T}{2},y^{1:T}) - (r) \left( (\epsilon^{2})\delta T  - 3d(g_m(T) + g_{h}(T)) T - 3d\left(\frac{f_{m}(g_{m}(T) \cdot T)}{g_{m}(T)} + \frac{f_{h}(g_{h}(T) \cdot T)}{g_{h}(T)}\right)\right) \\
& = \SQErr(\barphk{T}{2},y^{1:T}) - (r) \left( (\epsilon^{2})\delta T - T\beta(T)\right)
\end{align*}

Finally we can compose this expression with one more instantiation of Theorem \ref{thm:cases-dimensions}:
\begin{align*}
    \SQE(\barphk{T}{2}, y^{1:T}) &\leq  \SQE(\barpmk{T}{1}, y^{1:T}) -(\epsilon - g_{h}(T))^{2} |\Tk{2}| + dg_h(T) T + 3d \frac{f_{h}(g_{h}(T) \cdot T)}{g_{h}(T)} \\
    & \leq \SQE(\barpmk{T}{1}, y^{1:T}) -(\epsilon - g_{h}(T))^{2} \delta T + dg_h(T) T + 3d\frac{f_{h}(g_{h}(T) \cdot T)}{g_{h}(T)} \tag{as $|\Tk{2}| \geq \delta T$} \\
     & \leq \SQE(\barpmk{T}{1}, y^{1:T}) -\epsilon^{2} \delta T + 2d g_h(T) T + 3d \frac{f_{h}(g_{h}(T) \cdot T)}{g_{h}(T)}  \\
      & \leq \SQE(\barpmk{T}{1}, y^{1:T}) -\epsilon^{2} \delta T + T\beta(T)
\end{align*}

and get a final expression of:
\begin{align*}
    \SQErr(\barphk{T}{r}, y^{1:T}) & \leq \SQE(\barpmk{T}{1}, y^{1:T}) - \epsilon^{2} \delta T + T\beta(T)  - r\left(\epsilon^{2}\delta T  + T\beta(T)\right)  \\
    & =\SQE(\barpmk{T}{1}, y^{1:T}) - (r+1)\left(\epsilon^{2}\delta T + T\beta(T)\right) 
\end{align*}

This completes the second part of the Theorem. 

By definition, the squared error is non-negative. Therefore, we have that 

\begin{align*}
   & 0  \leq \SQErr(\barphk{T}{2}, y^{1:T}) - r\left( (\epsilon^{2})\delta T - T\beta(T)\right) \\
   & \implies r \leq \frac{\SQErr(\barphk{T}{2},y^{1:T})}{\epsilon^{2}\delta T - T\beta(T)}
   \\ & \implies r \leq \frac{d \cdot T}{\epsilon^{2}\delta T - T\beta(T)}
    \\ & \implies r \leq \frac{d}{\epsilon^{2}\delta  - \beta(T)}
\end{align*}

This completes the first part of the Theorem.
\end{proof}

\begin{proof}[Proof of Theorem \ref{thm:cases-dimensions}]
Let $T_m(k,i,j) = \left\{t \in \Tk{k} ~|~ \pmk{t}{k-1}[j] \in B_i(1/g(T))\right\}$ be the subsequence of days where the $j$'th coordinate of the predictions of the model at round $k-1$ falls in bucket $i$ and the conversation reaches round $k$.
Note that by the definition of conversation calibration in $d$ dimensions (Definition~\ref{def:conversation-calibration-dimensional}), we have that 
$$ \CalDist(\phk{T_m(k, i,j)}{k}[j] ,y^{T_m(k, i,j)}[j] ) \leq f(|T_m(k, i,j)|)$$

Therefore, for predictions $p_{h}^{1:T,k}[j]$ from the human at round $k$ in dimension $j$: 

\begin{align*} 
\CalDist(p_{h}^{\Tk{k},k}[j], y^{\Tk{k}}[j]) & =
 \min_{q^{1:T} \in C(y^{\Tk{k}}[j])}\|p_{h}^{\Tk{k},k}[j] - q_{j}^{1:T}\|_{1} 
   \tag{For $1$-dimensional predictions $q_{j}$}\\  &
 \leq \sum_{i=1}^{\frac{1}{g_{h}(T)}}\min_{q_{j}^{T_{m}(k,i,j)} \in C^{T_{m}(k,i,j)}(y^{\Tk{k}}[j])}\|p_{h}^{T_{m}(k,i,j),k}[j] - q_{j}^{T_{m}(k,i,j)}\|_{1}\\
& \leq \sum_{i=1}^{\frac{1}{g_{h}(T)}} f_{h}(|T_{m}(k,i,j))|) \tag{By the calibration distance of the Human}\\ 
& \leq  \frac{f_{h}(g_{h}(T) \cdot |\Tk{k}|)}{g_{h}(T)} \tag{By the assumption that $f_{h}$ is concave} \\
& \leq \frac{f_{h}(g_{h}(T) \cdot T)}{g_{h}(T)} 
\end{align*}

Let $q^{1:T,k,j}$ be a set of perfectly calibrated predictions that are $f_{h}(|T_{m}^{k,i,j}|)$-close to $p_{h}^{\Tk{k},k}[j]$. Furthermore, let $q^{1:T,k}$ be the set of $d$-dimensional predictions such that $q^{1:T,k}[j] = q^{1:T,k,j}$. Then, we have: 

\begin{align*}
\SQErr(p_{h}^{\Tk{k},k}, y^{\Tk{k}}) & = \sum_{j \in [d]}\SQErr(p_{h}^{\Tk{k},k}[j], y^{\Tk{k}}[j]) \\ & \leq \sum_{j \in [d]}\left(\SQErr(q^{1:T,k,j}, y^{\Tk{k}}[j]) + 3\frac{f_{h}(g_{h}(T) \cdot T)}{g_{h}(T)}\right) \tag{by Lemma~\ref{lem:bound_error_diff}} \\
& = \SQErr(q^{1:T,k}, y^{\Tk{k}}) + 3d \cdot \frac{f_{h}(g_{h}(T) \cdot T)}{g_{h}(T)}  \\
&  \leq \SQErr(p_{m}^{\Tk{k},k-1},y^{\Tk{k}}) - (\epsilon - g_{h}(T))^{2}|\Tk{k+1}| + dg_h(T) T + 3d\frac{f_{h}(g_{h}(T) \cdot T)}{g_{h}(T)} \tag{By Lemma~\ref{lem:mhd}} \\
\end{align*}

Finally, note that $\SQErr(p_{h}^{\Tk{k},k}, y^{\Tk{k}}) - \SQErr(p_{m}^{\Tk{k},k-1},y^{\Tk{k}}) = \SQErr(\bar{p}_{h}^{T,k}, y^{1:T}) - \SQErr(\bar{p}_{m}^{T,k-1},y^{1:T})$, as $\bar{p}_{h}$ and $\bar{p}_{m}$ are equal on all inactive days. Therefore, we have that

\begin{align*}
    \SQErr(\bar{p}_{h}^{T,k}, y^{1:T}) \leq \SQErr(\bar{p}_{m}^{T,k-1},y^{1:T}) - (\epsilon - g_{h}(T))^{2}|\Tk{k+1}| + dg_h(T) T + 3d\frac{f_{h}(g_{h}(T) \cdot T)}{g_{h}(T)}
\end{align*}

As the Human and the Model are symmetric, we also obtain the symmetric result for the Model.
\end{proof}

\begin{algorithm}[ht]
\begin{algorithmic}
    \STATE{\bf Input} {Baseline model algorithm $M_{0}$, D2C algorithm $D$, Discretization $g_{m}(T)$}
    \STATE We denote $D^j_{k,i}$ as an instantiation of $D$ which is given as input only the subsequence of days where $\phk{t}{k}[j] \in [(i - 1) \cdot g_{m}(T), i \cdot g_{m}(T)]$, and denote $D^j_{k,i,t}$ be the prediction of $D^j_{k,i}$ at round $t$.
    \FOR{$t = 1, \ldots, T$}
        \STATE Receive $x^t_m$
        \STATE Send prediction $\pmk{t}{1}[j] = M_{0}(x^{t}_{m})[j]$ to human for each $j \in [d]$ 
        \FOR{$k = 3, 5, \ldots $}
            \FOR{$j = 1, 2, \ldots, d$}
            \STATE Initialize empty set $S$
             \STATE Observe human prediction $\phk{t}{k-1}[j]$ 
             \IF{$ | \phk{t}{k-1}[j] - \pmk{t}{k-2}[j] | < \epsilon $}
            \STATE Predict $\phk{t}{k-1}[j]$ and break out of loop
                \ENDIF
            \STATE Let $i$ be such that $\phk{t}{k-1}[j] \in [(i - 1) \cdot g_{m}(T), i \cdot g_{m}(T)]$
            \IF{$D^j_{k-1,i}$ uninitialized}
            \STATE Initialize $D^j_{k-1,i}$
            \ENDIF
            \STATE Send prediction $\pmk{t}{k}[j] = D^j_{k-1,i,t}$ to human
            \STATE $S \gets S \cup (k-1,i, j)$
             \IF{ $ | \phk{t}{k-1}[j] - \pmk{t}{k}[j] | < \epsilon $} 
            \STATE Predict $\phk{t}{k-1}[j]$ and break out of loop
                \ENDIF
            \ENDFOR
        \ENDFOR
    \STATE Observe $y^{t}$
    \FOR{$(k,i, j) \in S$ }
    \STATE Update $D^j_{k,i}$ with $(D^j_{k,i,t}, y^{t})$
    \ENDFOR
    \ENDFOR
\end{algorithmic}
\caption{\textsc{Converse-dDim($M_{0}, D, g_{m}(T)$)}: A reduction from an online decision-making algorithm to an algorithm with low conversation-calibration error in $d$ dimensions} 
\label{alg:converse-reduction-dimensions}
\end{algorithm}

\begin{proof}[Proof of Theorem \ref{thm:reduction-dimensinoal}]
Algorithm \ref{alg:converse-reduction-dimensions} instantiates a copy of algorithm $D$ for each round $k$, bucket $i$, and coordinate $j$ pair corresponding to each of the sets $T_m(k, i, j)$. Therefore, we have that for each round $k$, bucket $i$, coordinate $j$,
    \begin{align*}
        \CalDist(\phk{T_m(k, i,j)}{k}[j] ,y^{T_m(k, i,j)}[j] ) \leq f(|T_m(k, i,j)|)
    \end{align*}
by assumption that $D$ has worst-case distance to calibration $f_m(\cdot)$. 

The fact that $\textsc{Converse-dDim}_{1}(x_m^{t})[j] = M_{0}(x_m^{t})[j]$, for all $t$ follows by construction.
\end{proof}

\section{Additional Material from Section \ref{sec:action}
}
\label{app:action}

\begin{proof}[Proof of Theorem \ref{thm:blorb}]
By Corollary~\ref{cor:decisions}, When $\gamma(T) \leq \eps \delta$, on a $1-\delta$ fraction of days, the number of rounds until agreement is at most

       $ K \leq \frac{1}{\eps \delta} + 1$, where $\gamma(T) = \frac{2 L d |A|^{2}\cdot f_h(\frac{T}{|A|^{2}}) + 2 L d |A|^{2}\cdot f_m(\frac{T}{|A|^{2}})}{T}$. Instantiating this bound with the high-probability result for Theorem~\ref{thm:reduction-action}, we have that, with probability $1 - \alpha$: 
   \begin{align*}
& \gamma(T) = \frac{4 L d |A|^{2}\cdot O\left(\log(2d\left| \mathcal{E}\right| T) +  \sqrt{T\ln(\frac{|\mathcal{E} | d}{\alpha})} \right)}{T} \\
&  = \frac{4 L d |A|^{2}\cdot O\left(\log(d\left|A\right|^{2} T) +  \sqrt{T\ln(\frac{|A|^{2} d}{\alpha})} \right)}{T} \tag{By the definition of $\mathcal{E}$ in our setting} \\
& = \frac{O\left(L d |A|^{2} \cdot \log(d\left| A\right|^{2} T) + L d |A|^{2} \cdot \sqrt{T\ln(\frac{|A|^{2} d}{\alpha})} \right)}{T} \\
& = \frac{O\left(L d |A|^{2} \cdot \log(d\left| A\right|^{2} T)\right)}{T} + \frac{O\left(L d |A|^{2} \cdot \sqrt{\log(\frac{|A|^{2} d}{\alpha})} \right)}{\sqrt{T}} \\
& \leq \frac{O\left(L d |A|^{2} \cdot \log(d\left| A\right|^{2}) + L d |A|^{2} \cdot \sqrt{\log(\frac{|A|^{2} d}{\alpha})}\right)}{\sqrt{T}} \\
   \end{align*} 
   Thus, to set $\gamma(T) \leq \epsilon \delta$ w.p. $\geq 1 - \alpha$, it is sufficient to set
   \begin{align*}
     &  \frac{O\left(L d |A|^{2} \cdot \log(d\left| A\right|^{2})\right)}{\sqrt{T}} + \frac{O\left(L d |A|^{2} \cdot \sqrt{\log(\frac{|A|^{2} d}{\alpha})} \right)}{\sqrt{T}} \leq \epsilon \delta \\
    & \implies   \frac{O\left(L d |A|^{2} \cdot \log(d\left| A\right|^{2}) + L d |A|^{2} \cdot \sqrt{\log(\frac{|A|^{2} d}{\alpha})} \right)}{\epsilon \delta} \leq \sqrt{T} \\
    & \implies T \geq \frac{O\left(L^{2} d^{2} |A|^{4} \cdot \log^{2}(d\left| A\right|^{2}) + L^{2} d^{2} |A|^{4} \cdot \log(\frac{|A|^{2} d}{\alpha}) \right)}{\epsilon^{2} \delta^{2}} \\
     & \implies T \geq \frac{O\left(L^{2} d^{3} |A|^{5} (1 + \log(\frac{1}{\alpha})) \right)}{\epsilon^{2} \delta^{2}}
   \end{align*}
\end{proof}

\section{Additional Material from Section \ref{sec:bayesian}}
\label{app:bayesian}

\subsection{Bayesians are Conversation Calibrated}
\begin{proof}[Proof of Lemma \ref{lem:resampling-transcript}]
    We want to show that for any round $j$ when the resampling might occur, $\Pr_{\cD}[\pi^{t,1:k}] = \Pr_{\cD}[\bar{\pi}^{t,1:k}_{j}]$ for all $k$.
   For $k < j$, the claim follows immediately since there is no difference in the two sampling protocols.  In round $j$, the claim that  $\Pr_{\cD}[\pi^{t,1:j}] = \Pr_{\cD}[\bar{\pi}^{t,1:j}_j]$ follows from a generic statement about resampling from posterior distributions that we formalize in Lemma \ref{lem:resampling}: the joint distribution on any pair of random variables $(A, B)$ is unchanged if we first sample a pair $(A, B')$ and then sample $B$ from its posterior distribution conditional on $A$. In this case, $A$ is the distribution on the transcript $\pi^{t,1:j}$ excluding the label $y^t$ and $B$ is the label $y^t$. For rounds $k> j$ the claim holds since the distribution for the remaining interaction at round $t$ is fixed once we fix $\pi^{t,1:k-1}$

\end{proof}

\begin{proof}[Proof of Lemma \ref{lem:distance-to-bucketed}]
    For any bucket $i \in [n]$, define the average outcome when the prediction falls in a bucket $B_n(i)$:
    \begin{align*}
        \bar{y}_i = \sum_{t=1}^T \frac{ \mathbbm{1}[p^t \in B_n(i)] }{ \sum_{t'=1}^{T} \mathbbm{1}[ p^{t'} \in B_n(i)] } y^t
    \end{align*}

    Similarly, let $\bar{p}_i$ define the average prediction in bucket $i$.
    Consider the sequence $q^{1:T}$ where $q^t = \bar{y}_i (p^t)$, where $p^t \in B_n(i)$. Observe that $q^{1:T}$ is perfectly calibrated.
    \begin{align*}
       \CalDist(p^{1:T}, y^{1:T}) &\leq \| p^{1:T} - q^{1:T} \|_1 \\
        &= \sum_{t=1}^T |p^t - q^t| \\
        & = \sum_{t=1}^T \sum_{i \in [n]} \mathbbm{1}[p^t \in B_n(i)]|p^t - \bar{y}^t_i| \\
        &\leq \sum_{t=1}^T \sum_{i \in [n]} \mathbbm{1}[p^t \in B_n(i)] \left( |p^t - \bar{p}_i| + |\bar{p}_i - \bar{y}_i| \right) \tag{by the triangle inequality} \\
        &= \sum_{i \in [n]} \sum_{t=1}^T \mathbbm{1}[p^t \in B_n(i)] \left( |p^t - \bar{p}_i| \right) + \sum_{i \in [n]} \left| \sum_{t=1}^T \mathbbm{1}[p^t \in B_n(i)] \left( \bar{p}_i - \bar{y}_i \right) \right| \tag{by the fact that $\bar{p}_i$ and $\bar{y}_i$ are constant for each $i \in [n]$} \\
        &\leq \frac{T}{n} + \sum_{i \in [n]} \left| \sum_{t=1}^T \mathbbm{1}[p^t \in B_n(i)] \left( \bar{p}_i - \bar{y}_i \right) \right| \tag{by the fact that $|p - \bar{p}_i| \leq \frac{1}{n}$ for all $p \in B_n(i)$} \\
        & = \frac{T}{n} +  \ECE(p^{1:T}, y^{1:T}; n) .
    \end{align*}
\end{proof}

\begin{proof}[Proof of Theorem \ref{thm:bayes-calibrated}]
Consider a modified interaction under Protocol \ref{alg:bayesian-agreement} when, in each day in round $j$ (if the conversation reaches round $j$), the outcome is resampled according to the information seen by the human so far: $y' \sim \cD_{\cY} | x^t_h, \mu^{1:t-1}, \bar{\pi}^{1:t-1}_h, C^{t}_{1:j-1}, \pmk{t}{j}$. Let $\hat{\pi}^{j}$ be the transcript from this interaction. 
First, we will show that $\Pr_{\cD}[\pi] = \Pr_{\cD}[\hat{\pi}^{j}]$, where $\pi$ is the transcript under the unmodified Protocol \ref{alg:bayesian-agreement}.

Let $\hat{\pi}^{1:t,j}$ denote the transcript of this interaction up to day $t$. Note that this is distinct from $\bar{\pi}^{t,j}$, which denotes the transcript of an interaction only on day $t$ where the resampling only occurs in round $j$.
We will proceed via induction over days. 
\begin{itemize}[leftmargin=3ex]
    \item \textbf{Base Case}:  $\Pr_{\cD}[\pi^{1:1}] = \Pr_{\cD} [\hat{\pi}^{1:1,j}]$. 
    
    \textit{Proof}: 
    On day $t=1$, we have $\Pr_{\cD}[\pi^{1}] = \Pr_{\cD}[\bar{\pi}^{1,j}]$, by Lemma~\ref{lem:resampling-transcript}. Note that $\bar{\pi}^{1,j} = \bar{\pi}^{1:1,j} = \hat{\pi}^{1:1,j}$, and therefore $\Pr_{\cD}[\pi^{1:1}] = \Pr[\hat{\pi}^{1:1,j}]$.
    \item \textbf{Inductive Step}:  If $\Pr_{\cD}[\pi^{1:t}] = \Pr_{\cD}[\hat{\pi}^{1:t,j}]$, then $\Pr_{\cD}[\pi^{1:t+1}] = \Pr_{\cD}[\hat{\pi}^{1:t+1,j}]$.
    
    \textit{Proof}: 
    Observe that the state of the model algorithm in any round $t+1$ is a function only of the algorithm $M$ and the transcript until that round: $\pi^{1:t}$ or $\bar{\pi}^{1:t}$.
    By the Inductive Hypothesis, $\Pr_{\cD}[\pi^{1:t}] = \Pr_{\cD}[\hat{\pi}^{1:t,j}]$ -- and consequently, since the model algorithm $M$ is the fixed between both interactions, therefore,  $\Pr_{\cD}[\pi^{t+1,j}] = \Pr_{\cD}[\bar{\pi}^{t+1,j}]$.
    By Lemma~\ref{lem:resampling-transcript}, this is equal to $\Pr_{\cD}[\pi^{t+1}]$. As $\Pr_{\cD}[\hat{\pi}^{1:t,j}] = \Pr_{\cD}[\pi^{1:t}]$ and $\Pr_{\cD}[\bar{\pi}^{t+1,j}] = \Pr_{\cD}[\pi^{t+1}]$, we have that $\Pr_{\cD}[\pi^{1:t+1}] = \Pr_{\cD}[\hat{\pi}^{1:t+1,j}]$.
\end{itemize}

Now, all that remains to show is that the human has low calibration error in transcript $\hat{\pi}(j)$. We will want to do so for each of our notions of conversation-calibration error. 

\paragraph{Conversation-Calibration Error} 
Fix some arbitrary round $k$. 
We will proceed by bounding the expected bucketed calibration error of the human conditioned on the model's previous message, and then applying Lemma \ref{lem:distance-to-bucketed} to show that this also bounds the human's conversation-calibration error. 

Fix some bucketing coarseness $m$ for the expected calibration error of the human and some bucket $v_h \in \cB_m$.
Fix some bucketing coarseness $n$ and some bucket $v_m \in \cB_n$.

We can then define a conditioning event $E: \Pi \to [0, 1]$, defined as $E(\pi^{1:t}) = \mathbb{I}[ \phk{t}{k} \in v_h$, $\pmk{t}{k-1} \in v_m ]$.  Recall that the human is a Bayesian Learner. Therefore, their prediction is deterministic at the beginning of round $k$, since it is simply the posterior mean of the distribution conditioned on the model's predictions through round $k-1$. 

Thus, we can instantiate Lemma \ref{lem:azuma-app} with this event $E(\cdot)$: with probability $1 - \delta$,
\begin{align*}
     \left| \sum_{t=1}^T  E(\pi^{1:t-1}) \cdot \left( y^t(\pi^{1:t-1})[j] - \E_{y \sim \cD} [y[j] | \pi^{1:t-1}] \right) \right| \leq 2 \sqrt{ 2T \ln \frac{1}{\delta} }.
\end{align*}

Taking the union bound over all $j \in [d]$, $v_m \in \cB_n$, we see that the magnitude of the bias of the human's predictions in coordinate $j$ conditional on making a prediction in bucket $v_h$ in round $k$ is, with probability $1 - \delta$, bounded by
\begin{align*}
    2 \sqrt{ 2T \ln \frac{dn}{\delta} }.
\end{align*}

We can then sum across all buckets of the human's prediction $[m]$ to see that the expected calibration error of the human is bounded by
\begin{align*}
    \ECE(\yhk{1:T}{k} y^{1:T}; m) \leq \sum_{j \in [m]} 2 \sqrt{ 2T \ln \frac{dn}{\delta} } =  2m \sqrt{ 2T \ln \frac{dn}{\delta}}.
\end{align*}

Applying Lemma \ref{lem:distance-to-bucketed}, we can bound the human's conversation-calibration error for fixed $k$, and all $i \in [n]$ and $j \in [d]$ as:
\begin{align*}
    \CalDist(\hat{y}_h^{T_m(k, i, j)}[j], y^{T_m(k, i, j)}[j]) \leq \frac{T}{m} +  2m \sqrt{ 2T \ln \frac{dn}{\delta}}.
\end{align*}

Finally, setting the number of buckets $\cB_m$ optimally as:
\begin{align*}
    \frac{T}{m} &= 2m \sqrt{ 2T \ln \frac{dn}{\delta}} \\
    &\downarrow \\
    \frac{T}{2 \sqrt{2T \ln \frac{dn}{\delta}}} &= m^2 \\
    \frac{\sqrt{T}}{2 (2T \ln \frac{dn}{\delta})^{1/4}} &= m \\
     \frac{T^{\frac14}}{2 (2 \ln \frac{dn}{\delta})^{\frac14}}&= m
\end{align*}

We have a final bound of, for all $j \in [d], i \in \cB_n$:
\begin{align*}
    \CalDist(\hat{y}_h^{T_m(k, i, j)}[j], y^{T_m(k, i, j)}[j]) \leq O( T^{\frac34} (\ln\frac{dn}{\delta})^{\frac{1}{4}} ).
\end{align*}

We have shown this for an arbitrary round $k$. The claim that the human has low conversation-calibration error in the transcript $\hat{\pi}(k)$ holds for any round $k$ when the resampling might occur, and so we have that with probability $1-\delta$, 
the human is $\left( O(T^{\frac34} (\ln\frac{dn}{\delta})^{\frac{1}{4}}), \frac{1}{n}\right)$-conversation-calibrated. 

\paragraph{DC-Calibration Error}

Fix some arbitrary round $k$.
We proceed by bounding the magnitude of the bias of the predictions in each coordinate  conditioned on the model's recommended action in the previous round and the best response to the human's prediction. 
Fix $a, a' \in \cA$. 
We can then define a conditioning event $E: \Pi \to [0, 1]$, where $E(\pi^{1:t}) = \mathbb{I}[ \pmk{t}{k-1} = a, \phk{t}{k} = a' ]$. Recall that the human is a Bayesian Learner. Therefore, their prediction is deterministic at the beginning of round $j$, since it is simply the posterior mean of the distribution conditioned on the model's predictions through round $k-1$. 
Thus, we can instantiate Lemma \ref{lem:azuma-app} with this event $E(\cdot)$ and see that with probability $1 - \delta$,
\begin{align*}
    | \sum_{t=1}^T  E(\pi^{1:t-1}) \cdot \left( y^t(\pi^{1:t-1})[j] - \E_{y \sim \cD} [y[j] | \pi^{1:t-1}] \right) | \leq 2 \sqrt{ 2T \ln \frac{1}{\delta} }.
\end{align*}
Taking the union bound over all $j \in [d], a, a' \in \cA$, we see that the DC-calibration in round $k$ is, with probability $1-\delta$, bounded by
\begin{align*}
    2 \sqrt{ 2T \ln \frac{d|A|^2}{\delta} }.
\end{align*}

We have shown this for an arbitrary round $k$. The claim that the human has low DC-calibration error in the transcript $\hat{\pi}(k)$ holds for any round $k$ when the resampling might occur, and so we have that, with probability $1-\delta$,
the human is $(2 \sqrt{ 2T \ln \frac{d|A|^2}{\delta}})$-DC-conversation-calibrated.

\end{proof}

\begin{theorem}[Azuma's Inequality] \label{thm:azuma}
    Let $\{ X_0, X_1, \ldots \}$ be a martingale sequence such that $|X_{i+1} - X_i| < c$ for all $i$, then,
    \begin{align*}
        \Pr[X_n - X_0 \geq \epsilon] \leq \exp{ \left(- \frac{\epsilon^2}{2c^2n} \right)}.
    \end{align*}
\end{theorem}

An immediate corollary of Theorem \ref{thm:azuma} follows from appropriately setting parameters.
\begin{corollary}
Letting $X_0 = 0, \eps = c\sqrt{2 n \ln \frac{1}{\delta}}$, then we have for any $\delta \in (0, 1)$, with probability $1 - \delta$, 
\begin{align*}
    X_n \leq  c \sqrt{2 n \ln \frac{1}{\delta} } .
\end{align*}
\end{corollary}

\begin{lemma} \label{lem:azuma-app}
    Let $E: \Pi \to [0, 1]$ represent any  conditioning event. 
    Consider the random process $\{\cZ^t\}$ adapted to the sequence of random variables $\pi^t$ for $t \geq 1$ and let
    \begin{align*}
        \cZ^t \coloneqq Z^{t-1} + E(\pi^{1:t-1}) \cdot \left( y^t(\pi^{1:t-1}) - \E_{y \sim \cD} [y | \pi^{1:t-1}] \right) 
    \end{align*}
    Then,
    \begin{align*}
        \sum_{t=1}^T  E(\pi^{1:t-1}) \cdot \left( y^t(\pi^{1:t-1}) - \E_{y \sim \cD} [y | \pi^{1:t-1}] \right) \leq 2 \sqrt{ 2T \ln \frac{1}{\delta} }, 
    \end{align*}
    with probability $1 - \delta$ over the randomness of $\cD$ and $\pi^{1:t-1}$. 
\end{lemma}
\begin{proof}
    First, observe that the above sequence is a martingale as
    $\E_{\cD} [  E(\pi^{1:t-1}) \cdot (y^t(\pi^{1:t-1}) - \E_{y \sim \cD} [ y | \pi^{1:t-1} ]  ] = 
    E(\pi^{1:t-1}) \cdot \E_{\cD} [ (y^t(\pi^{1:t-1}) - \E_{y \sim \cD} [ y | \pi^{1:t-1} ]  ] = 0$, since $E(\pi^{1:t-1})$ is a constant at the start day $t$ as it does not depend on the outcome $y^t$. Thus, $\E_{\cD} [ Z^{t+1} ] = Z^t$.
    Next, observe that since the outcomes $y \in [-1, 1]$, we have the bounded difference condition: $|Z^t - Z^{t-1}| < 2$ for all $t$.
    We can then instantiate Azuma's Inequality with $n = T$ and $c = 2$ to get the claim. 
\end{proof}

\begin{lemma}[Resampling]
    Let $\cD$ be a probability distribution over space $\cA \times \cB$. 
    For all $(a, b)$,
    \begin{align*}
        \Pr_{(a, b) \sim \cD, b’ \sim \cD | a} [ (a, b’) ] = \Pr_{(a, b) \sim \cD} [(a, b)].
    \end{align*} \label{lem:resampling}
\end{lemma}
\begin{proof} 
    \begin{align*}
       \Pr_{(a, b) \sim \cD, b’ \sim \cD | a} [ (a, b’) ] &  =  \Pr_{a\sim \cA} [a] \cdot \Pr_{b' \sim \cD|a} [b'] \\
    &  = \Pr_{a\sim \cA} [a] \cdot \Pr_{b \sim \cD|a} [b] \\
       &  = \Pr_{(a,b) \sim \cD} [(a,b)] \\
    \end{align*}
\end{proof}

\section{Extension to Multiple Agents} \label{sec:multiagent}

    In this section, we will extend our results to settings in which there are multiple agents interacting and aiming to reach agreement, rather than just two. For simplicity, we will restrict our attention here to the canonical setting (studied in Section \ref{sec:full-info}), but our treatment here is meant to be exemplary: all of the settings we study can be efficiently extended to the $n$ agent case in a similar manner.  
    We will refer to the total number of agents as $n$. 
    Since all agents in the canonical setting are symmetric (i.e. their calibration conditions and message spaces are the same), we will refer to them simply as agents in this section, rather than distinguishing between a specific number of humans and models.
    Informally, the results follows the same techniques as previously. We imagine a setting in which all $n$ agents are \emph{marginally} conversation calibrated with respect to the $n-1$ other agents. In fact, our results require only a weaker condition ---- there should be some distinguished agent (agent 1) that satisfies $n-1$ marginal conversation calibration conditions with respect to his $n-1$ interlocutors, but the other agents only need to be conversation calibrated with respect to agent $1$. Our algorithmic reduction will efficiently convert a model into one that can maintain all $n-1$ conversation calibration conditions simultaniously, so the algorithm can always serve the role of the distinguished agent, which allows us to make strictly weaker assumptions on the other parties. We note that once all agents $\epsilon/2$ agree with agent 1, they must also (by the triangle inequality) $\epsilon$-agree with each other pairwise. Our analysis proceeds by showing that in any round in which an agent substantially disagrees with agent 1, the squared error of their predictions must improve relative to agent $1$'s predictions; similarly agent $1$'s predictions must improve relative to any other agent with which he disagrees substantially frequently. 
    
    We will first adapt our notation to handle $n$ agents. We refer to the message space of an agent as $\Omega_a$.

    \begin{definition}[Agreement for $n$ Agents]
        Given an agreement condition for two parties $\textsc{Agree}$, we define an agreement condition for $n$ parties as the function: $\textsc{n-Agree}_{\eps, \textsc{Agree} }: (\Omega_a \times \cY)^n \to \{0, 1\}$ defined as:
        \begin{align*}
             \textsc{n-Agree}_{\eps, \textsc{Agree}} ( p_1, y_1, \ldots, p_{n}, y_{n}  ) = 
            \begin{cases}
                1, & \sum_{r \in \{2, \ldots, n\}} \textsc{Agree}_{\eps/2}(p_1, p_1, p_r, y_r) = n - 1 \\
                0, & \text{ otherwise.}
            \end{cases}
        \end{align*}
    \end{definition}

    \begin{protocol}[ht]
\begin{algorithmic}
    \STATE{ {\bf Input} $(\Omega_a, \cY, \textsc{Agree}$) }
    \FOR{each day $t = 1, \ldots$}
        \STATE Receive $x^t = (x^t_1,\ldots x^t_n)$. Agent $r$ sees $x^t_r$.
        \FOR{each round $k = 1, 2, \ldots,L$}
            \STATE{ Set $i = k \mod n$}
            \STATE{ Agent $i$ predicts $\hat{y}^{t, k}_i $ and sends all other agents $p^{t, k}_i \in \Omega_a$}
            \IF{ $ \textsc{n-Agree}_{\eps, \textsc{Agree}}( p^{t, k}_i, \hat{y}^{t, k}_i, p^{t, k-1}_{i - 1 \mod n}, \hat{y}^{t, k-1}_{i-1 \mod n}, \ldots, p^{t, k-(n-1)}_{i - (n-1) \mod n}, \hat{y}^{t, k-(n-1)}_{i-(n-1) \mod n}   ) =1$} 
                \STATE{ Return $p^{t,k}_i$ and break out of loop}
            \ENDIF
        \ENDFOR
        \STATE{Agents observe $y^t \in \cY$}
    \ENDFOR

\end{algorithmic}
\caption{\textsc{General $\eps-$Agreement Protocol for Many Agents}}  \label{alg:general-agreement-many}
\end{protocol}

We will need to slightly modify our conversation-calibration definitions to handle the general case of $n$ agents. The idea is the same - an agent $r$ is conversation-calibrated with respect to agent $s$ if their predictions are calibrated conditional on the most recent message sent by agent $s$.
The only difference is  superficial - in the indexing of the subsequences of days of interest, which is made slightly more complicated by the introduction of multiple agents. 
In Protocol $\ref{alg:general-agreement-many}$, an agent $r$  speaks in rounds $k$ such that $k \equiv r \mod n$. 
For an arbitrary agent $s$, the most recent time they will have spoken prior to some round $k$ (when an agent $r$ is speaking) is: $k - ((r - s) \mod n)$. 

\begin{definition}[Conversation-Calibrated Predictions with Many Agents]
\label{def:conversation-calibration-many}
Fix an error function $f:\{1, \ldots, T\} \rightarrow \mathbb{R}$ and bucketing function $g: \{1, \ldots, T\} \rightarrow (0,1]$. Given a prediction transcript $\pi^{1:T}$ resulting from an interaction in the canonical setting (Definition \ref{def:setting-canonical}) with $n$ agents, an agent $r$ is $(f, g)$-conversation-calibrated with respect to agent $s$ if for all rounds $k \equiv r \mod n$ and buckets $i \in \{1, \ldots, 1/g(T)\}$:
\begin{align*}
    \CalDist(p_r^{T_s(k, i)}{k},y^{T_s(k, i)}) \leq f(|T_s(k, i)|),
\end{align*} 
where $T_s(k, i) = \left\{t \in \Tk{k} ~|~ p_s^{t, k - ( (r - s) \mod n)} \in B_i(1/g(T))\right\}$ is the subsequence of days where the conversation reaches round $k$ and the most recent prediction of agent $s$ falls in bucket $i$.
\end{definition}


\begin{theorem} \label{thm:canonical-multi}
    If agent 1 is $(f(\cdot), g(\cdot))-$conversation-calibrated with respect to agents $2, \ldots, n$ and agents $2, \ldots, n$ are all $(f(\cdot), g(\cdot))-$conversation-calibrated with respect to agent 1, then:
    for any $\eps, \delta \in [0, 1]$, on a $1 - \delta$ fraction of days, all agents reach $\eps-$agreement after at most $K$ rounds of conversation for
    \[ K \leq \frac{n}{ \frac{\epsilon^{2}\delta}{4} - \eta(T)}, \]
    where $\eta(T) = n\left(3 g(T) + 6 \frac{f(g(T))}{T g(T)}\right)$.
\end{theorem}

\begin{lemma} \label{lem:canonical-multi}
    If agent 1 is marginally $(f(\cdot), g(\cdot))-$conversation-calibrated with respect to agents $2, \ldots, n$ and agents $2, \ldots, n$ are all $(f(\cdot), g(\cdot))-$conversation-calibrated with respect to agent 1, then for any round $k$ such that $k = 1\mod n$:
    \begin{align*}
    \SQE(p^{1:T, k-n}_1, y^{\Tk{k}}) \leq \SQE(p^{1:T, k}_1, y^{\Tk{k}}) - ( (\frac{\eps}{2}) - g(T))^{2}\frac{|\Tk{k+1}|}{n} + 2g(T)T + 6\frac{f(g(T))}{g(T)}
    \end{align*}
\end{lemma}

\begin{proof}[Proof of Lemma \ref{lem:canonical-multi}]
    By the definition of Protocol~\ref{alg:general-agreement-many}, if the conversation continued from round $k-n$ to $k$ (for some $k \equiv 0 \mod n$), then every day, at least one other agent disagreed with agent 1's message in round $k-n$ by at least $\frac{\epsilon}{2}$. Then it must be the case that one of agents $2, \ldots, n$ disagreed with agent 1 in these rounds for at least $\frac{|\Tk{k+1}|}{n}$ days. Then, the claim follows as a corollary to Lemma \ref{thm:cases}.
\end{proof}

\begin{proof}[Proof of Theorem \ref{thm:canonical-multi}]
Consider any round $r$ such that $|\Tk{r}| \geq \delta T$.
We can create a telescoping sum by instantiating Lemma \ref{lem:canonical-multi} from round $k = 1$ to $r$:
\begin{align*}
    \SQE(p^{\Tk{r}, 1}_1, y^{\Tk{r}}) - \SQE(p^{\Tk{r}, r}_1, y^{\Tk{r}}) & \geq \sum_{k=1}^r \bigg( (\frac{\eps}{2} - g(T))^{2}|\frac{|\Tk{r+1}|}{n}| - 2g(T)T - 6\frac{f(g(T))}{g(T)} \bigg) 
    \\& = r \bigg( \frac{1}{n} (\frac{\eps}{2} - g(T))^2 \delta T  - 2g(T)T - 6\frac{f(g(T))}{g(T)}  \bigg)
\end{align*}

Therefore: 

\begin{align*}
  \SQE(p^{\Tk{r}, 1}_1, y^{\Tk{r}}) - r \bigg( \frac{1}{n} (\frac{\eps}{2} - g(T))^2 \delta T  - 2g(T)T - 6\frac{f(g(T))}{g(T)}  \bigg) &\geq 0 \tag{as $\SQE(p^{\Tk{r}, r}_1, y^{\Tk{r}}) \geq 0$} \\
  &\downarrow \\
  T -  r \bigg( \frac{1}{n} (\frac{\eps}{2} - g(T))^2 \delta T  - 2g(T)T -  6\frac{f(g(T))}{g(T)} \bigg) &\geq 0 \tag{as $\SQE(p^{\Tk{r}, 1}_1, y^{\Tk{r}}) \leq T$}\\
  & \downarrow \\
  \frac{T}{ \frac{1}{n} (\frac{\eps}{2} - g(T))^2 \delta T  - 2g(T)T - 6\frac{f(g(T))}{g(T)}  } & \geq r
\end{align*}
Hence: 
\begin{align*}
    r &\leq \frac{1}{ \frac{\delta }{n} ( (\frac{\eps}{2})^2 - 2 \frac{\eps}{2} g(T) + g(T)^2) - 2g(T) -  6\frac{f(g(T))}{T g(T)} } \\
    &\leq \frac{1}{\frac{\delta}{n} (\frac{\eps}{2})^2 - 3 \cdot g(T) - 6 \frac{f(g(T))}{T g(T)}}
   \\ &= \frac{n}{\frac{\delta\eps^{2}}{4} - n \cdot (3 \cdot g(T) + 6 \frac{f(g(T))}{T g(T)})}.
\end{align*}
\end{proof}

We conclude with the algorithmic reduction.
We will again use Theorem~\ref{thm:georgy} and the framework introduced in Section~\ref{sec:d-dimensions}. We will now have a different instantiation of the UnbiasedPrediction algorithm at each round $k$, and for the $k$th instantiation, the contexts each day $t$ will be the conversation $C^{t,1:k-1}$ on that day so far.

We are interested in being unbiased in each round conditional on events which are defined marginally by each of the other agent's most recent bucketed prediction, and our own bucketed prediction. Thus, we define our event set accordingly. To do this, we will define a new bucketing set, $\hat{B}$, which has a different number of buckets $\frac{1}{g_{1}(T)}$. This will be the bucketing that we measure agent 1's bucketed ECE on, which we will then convert into a distance to calibration bound. 

\begin{definition}[Multi-Conversation Events ]
For an agent $s$, a round $k$, and a pair of bucket indices $i_1,i_2$, let:
$$E_{s,i_{1},i_{2},k}(x^{t},\hat{y}^{t, k},C^{t,1:k-1}) = \mathbbm{1}\left[p_{s}^{t,k-((r-s) \mod n)} \in B_{i_{1}}\left(\frac{1}{g(T)}\right)\right]\mathbbm{1}\left[p_{1}^{t,k} \in \hat{B}_{i_{2}}\left(\frac{1}{g_{1}(T)}\right)\right]$$
Let $\mathcal{E}_{k} : = \{ E_{s,i_{1},i_{2},k} \forall i_{1}, i_{2}, s \}$. Note that $\left|\mathcal{E}_k\right| = \frac{1}{g(T)} \cdot \frac{1}{g_{1}(T)} \cdot (n-1) \leq \frac{n}{g_{1}(T) \cdot g(T)}$.
\end{definition}

We are now ready to define our reduction.

\begin{algorithm}[ht]
\begin{algorithmic}
    \STATE{\bf Input} {Baseline model algorithm $M_{0}$, Discretization $g_{m}(T)$}
    \FOR{$t = 1, \ldots, T$}
        \STATE Receive $x^t_m$ 
        \STATE Send prediction $\pmk{t}{1} = M_{0}(x^{t}_{m})$ to all other parties.
        \FOR{$k = 1, 1 + n, 1 + 2n, \ldots $}
            \STATE{$L \gets k$}
            \IF{$D_{k+1}$ uninitialized}
            \STATE Initialize $D_{k} = \textsc{UnbiasedPrediction}(\mathcal{E}_{k},\alpha)$
            \ENDIF
             \STATE Observe $n-1$ predictions $p^{t, k-((1-2) \mod n)}_2, \ldots, p^{t, k-((1-n) \mod n)}_{n}$
             \IF{$\forall q$, $|p^{t, k-((1-q) \mod n)}_q - p_{1}^{t,k} | \leq \frac{\epsilon}{2} $} 
            \STATE Predict $p_{1}^{t,k}$ and break out of loop
                \ENDIF
            \STATE Send prediction $p = D_{k}(C^{t,1:k-1},x^{t}_1)$
        \ENDFOR
    \STATE Observe $y^{t}$
    \FOR{$k = 1, 1 + n, 1 + 2n, \ldots L $}
    \STATE Update $D_{k}$ with $y^{t}$
    \ENDFOR
    \ENDFOR
\end{algorithmic}
\caption{\textsc{Converse-Many($M_{0},\alpha$)}} 
\label{alg:converse-reduction-many}
\end{algorithm}

\begin{theorem} \label{thm:reduction-many}
\textsc{Converse-Many($M_{0},\alpha$)} is $(O(\ln(\frac{ndT}{g(T)g_{1}(T)}) + \sqrt{T\ln(\frac{nd}{\alpha g(T)g_{1}(T)})} + \frac{T}{g_{1}(T)}),g_{1}(T))$-conversation-calibrated with respect to every other agent with probability $1 - \alpha$, and for any sequence of labels $y^{1:T}$, its first-round prediction is the same as the prediction of the base model $M_0$ for all $t$: $\textsc{Converse-Many($M_{0},\alpha$)}_{1}(x^t_m) = M_{0}(x^t_m)$, for all $t$. 
\end{theorem}
\begin{proof}
By construction, in each round $1, n+1,\ldots, 2n+1$, \textsc{Converse-Many($M_{0},\alpha$)} runs \\  $\textsc{UnbiasedPrediction}(\mathcal{E}_{k},\alpha)$ with subsequences defined by $\mathcal{E}_{k}$ in order to obtain predictions. By Theorem~\ref{thm:georgy}, in each round and for each other agent $s$, the bias on subsequences defined by the other agent's bucketing and agent 1's bucketing is $O(\ln(\frac{ndT}{g(T)g_{1}(T)}) + \sqrt{T\ln(\frac{nd}{\alpha g(T)g_{1}(T)})})$. Note that by Lemma~\ref{lem:distance-to-bucketed}, agent 1's distance to calibration on this subsequence is therefore at most $O(\ln(\frac{ndT}{g(T)g_{1}(T)}) + \sqrt{T\ln(\frac{nd}{\alpha g(T)g_{1}(T)})} + Tg_{1}(T))$

Thus the algorithm is $(O(\ln(\frac{ndT}{g(T)g_{1}(T)}) + \sqrt{T\ln(\frac{nd}{\alpha g(T)g_{1}(T)})} + Tg_{1}(T)),g_{1}(T))$-conversation-calibrated with respect to every other agent.

The second result follows directly from the definition of \textsc{Converse-Many($M_{0},\alpha$)}.
\end{proof}

To be concrete about rates, we will assume for the purposes of the remaining Theorems that Agent $1$ is employing Algorithm~\ref{alg:converse-reduction-many}, and the remaining Agents are employing Algorithm~\ref{alg:converse-reduction}. Thus, note by Corollary~\ref{cor:canonical} that Agents $2,\ldots,n$ are $\sqrt{T},T^{\frac{-1}{3}}$-conversation-calibrated, and therefore $g(T) = T^{-\frac{1}{3}}$.

\begin{theorem}
\label{thm:nparties}
    If $T \geq O\left( \frac{n^{6}}{\epsilon^{12}\delta^{6}}\cdot \ln^{6}(d)\right)$ and $g_{1}(T) =g^{2}(T)$, then with probability $\geq 1 - \alpha$, the number of rounds until agreement is at most 
    \[ K \leq \frac{2n}{\epsilon^{2} \delta} \]
\end{theorem}

\begin{proof}
By Theorem~\ref{thm:canonical-multi}, when $\eta(T) \leq \eps^{2} \delta$, on a $1-\delta$ fraction of days, the number of rounds until agreement is at most

       $ K \leq \frac{n}{\frac{\eps^{2} \delta}{4} - \eta(T)}$, where 
       $\eta(T) = n \left(3 g(T) + 6 \frac{f(g(T))}{T g(T)}\right)$
       . Instantiating this bound with the high-probability result for Theorem~\ref{thm:reduction-many}, we have that, with probability $1 - \alpha$: 
   \begin{align*}
& \eta(T) = n\cdot O\left( g(T) +  \frac{\ln(\frac{ndT}{g(T)g_{1}(T)}) + \sqrt{T\ln(\frac{nd}{\alpha g(T)g_{1}(T)})} + T g_{1}(T)}{Tg(T)}\right)  \\
& = n\cdot O\left( T^{\frac{-1}{3}} +  \frac{\ln(\frac{ndT}{T^{-1}}) + \sqrt{T\ln(\frac{nd}{\alpha T^{-1}})} + T^{\frac{1}{3}}}{T^{\frac{2}{3}}}\right) \tag{By the fact that $g_{1}(T) = g^{2}(T)$ and $g(T) - T^{\frac{-1}{3}}$} \\
\\ & = n\cdot O\left( T^{\frac{-1}{3}} +  \frac{\ln(ndT^{2}) + \sqrt{T\ln(ndT)}}{T^{\frac{2}{3}}}\right) \\
\\ & = n\cdot O\left( T^{\frac{-1}{3}} +  \frac{\ln(nd) + \sqrt{T\ln(nd)}}{T^{\frac{2}{3}}}\right) \\
\\ & = n\cdot O\left( \frac{\ln(nd)}{T^{\frac{2}{3}}} + \sqrt{\ln(nd)}T^{\frac{-1}{6}}\right) \\
\\ & \leq n\cdot O\left(\ln(nd)T^{\frac{-1}{6}}\right) \\
   \end{align*}

   Therefore, to set $K \leq \frac{2n}{\epsilon^{2}\delta}$, we can set

   \begin{align*}
    & \frac{\epsilon^{2}\delta}{2} \geq   n\cdot O\left(\ln(nd)T^{\frac{-1}{6}}\right) \\  
     \\ & \implies 
     T^{\frac{1}{6}} \geq 
   O\left( \frac{2n}{\epsilon^{2}\delta}\cdot \ln(nd)\right)
   \\ & \implies 
     T \geq 
   O\left( \frac{n^{6}}{\epsilon^{12}\delta^{6}}\cdot \ln^{6}(nd)\right)
      \\ & \implies 
     T \geq 
   O\left( \frac{n^{6}}{\epsilon^{12}\delta^{6}}\cdot \ln^{6}(d)\right)
   \end{align*}
\end{proof}

\end{document}